%% file: main.tex
\documentclass{article}

 \usepackage[preprint]{neurips_2026}

\usepackage[utf8]{inputenc} 
\usepackage[T1]{fontenc}    
\usepackage{hyperref}       
\usepackage{url}            
\usepackage{booktabs}       
\usepackage{amsfonts}       
\usepackage{nicefrac}       
\usepackage{microtype}      
\usepackage{xcolor}         
\usepackage{wrapfig}

\usepackage[english]{babel}
\usepackage{amsmath,amssymb,amsfonts,amsthm}
\usepackage{mathtools}
\usepackage{bm}
\usepackage{bbm}
\usepackage{subcaption}
\usepackage{multirow}

\usepackage{float}

\newtheorem{theorem}{Theorem}
\newtheorem*{theorem*}{Theorem}
\newtheorem{corollary}{Corollary}

\newtheorem*{proposition*}{Proposition}

\title{Latent-MoE: Domain-Aware Mixture-of-Experts for PDEs with Multi-Regime Physics}

\author{%
  Hanwen Wang \\
  Graduate Group of Applied Mathematics and Computational Science\\
  University of Pennsylvania\\
  Philadelphia, PA 19104 \\
  \texttt{wangh19@sas.upenn.edu} \\
  \And
  Paris Perdikaris \\
  Department of Mechanical Engineering and Applied Mechanics\\
  University of Pennsylvania\\
  Philadelphia, PA 19104 \\
  \texttt{pgp@seas.upenn.edu} \\
}

\begin{document}

\maketitle

\begin{abstract}
Physics-informed neural networks (PINNs) struggle on PDEs whose governing physics varies across the domain. We trace this to a structural property of standard coordinate networks: their neural tangent kernel (NTK) is translation-variant and lets training points of large coordinate magnitude disproportionately influence predictions elsewhere, producing long-range coupling and gradient conflict during training. We show analytically and empirically that mixture-of-experts (MoE) architectures with centered, compact-support routers yield a uniformly banded NTK whose kernel-regression weights decay exponentially with distance, localizing the learning. Building on this, we propose \emph{Latent-MoE}, which interleaves domain-aware MoE blocks within a shared backbone. Unlike FB-PINNs or X-PINNs, which rigidly partition both the domain and the parameters so that the parameters on different subdomains are updated independently, Latent-MoE is designed to preserve the localization benefit of domain-aware routing while allowing capacity to flow across regions through the shared backbone. On standard homogeneous-physics benchmarks Latent-MoE is competitive with established baselines; on benchmarks with multi-stage time-variable physics, where global models and rigid domain decompositions both fall into spurious solutions, it improves over them by more than an order of magnitude, with markedly reduced gradient conflict during training.
\end{abstract}

\section{Introduction}
Physics-informed neural networks (PINNs)\cite{raissi2017physicsinformeddeeplearning} embed the governing equations of a partial differential equation (PDE) directly into the training objective, allowing a coordinate network to act as a surrogate solution given only partial observations at the boundary or initial time. Modern automatic differentiation frameworks\cite{jax2018github,Ansel_PyTorch_2_Faster_2024,martin_abadi_tensorflow_2015}  make it straightforward to evaluate PDE residuals by differentiating the network with respect to its input coordinates, yielding a unified loss that enforces both the physics and the available data. This flexibility has made PINNs an attractive alternative to classical solvers, particularly for inverse problems and settings where mesh generation is difficult.

In practice, however, PINNs inherit the optimization pathologies of
coordinate networks~\cite{wang2021understanding}, and the higher-order
differentiation in the residual loss amplifies them. A range of
architectural and frequency-aware remedies have been
proposed~\cite{rahaman2019spectralbiasneuralnetworks,
basri2019convergencerateneuralnetworksfrequencies,
xu2019trainingbehaviordeepneuralnetwork,
sitzmann2020implicitneuralrepresentationsperiodic,
tancik2020fourierfeaturesletnetworkshighfrequency}, including
domain-decomposition methods that partition the spatio-temporal
domain and assign a subnetwork to each
region~\cite{jagtap2020extended, moseley2023finite, dolean2024multilevel,
botvinickgreenhouse2025abpinnsadaptivebasisphysicsinformedneural}.
Domain decomposition is conceptually appealing and empirically effective,
yet the question of \emph{why} partitioning helps -- and what kind of
partitioning is best -- has received limited theoretical attention.

In this work we revisit domain decomposition through the lens of the neural tangent kernel (NTK). Our starting point is that the NTK of a standard coordinate network is translation-variant: training points of larger coordinate magnitude exert disproportionate influence on predictions across the domain, even at distant locations. This long-range coupling is benign for spatially homogeneous physics but becomes destructive when the governing equations or forcing change across the domain, since gradients from regions with different dynamics end up competing during optimization. Collapse to spurious solutions is a known PINN failure mode. Even state-of-the-art architectures such as PirateNet fall into it on a homogeneous advection equation\cite{wang2026pinnswrongpseudotimestepping}, and we observe the same collapse on PDEs with multi-stage time-variable physics. We show that domain-aware mixture-of-experts (MoE) architectures with centered, compact-support routers eliminate this coupling by yielding a banded NTK, and we use this observation to motivate an architectural choice that has not, to our knowledge, been explored in the PINNs literature: interleaving sparse domain-aware experts within a shared backbone, rather than partitioning the entire model.

\paragraph{Contributions.}
\begin{itemize}
  \item Using NTK theory, we identify long-range coupling as a structural obstacle to PINN training: standard ReLU coordinate networks have a translation-variant NTK that lets training points of larger magnitude disproportionately influence predictions elsewhere in the domain.
  \item We prove that domain-aware MoE architectures with centered, compact-support routers yield a banded NTK whose kernel-regression weights decay exponentially in distance, eliminating the coupling and localizing the learning.
  \item We introduce \emph{Latent-MoE}, which interleaves domain-aware
experts within a shared backbone. Unlike rigid decompositions such as
FB-PINNs, in which each subdomain holds its own parameters and the training updates them separately, the shared backbone enables genuine gradient alignment across subdomains, while strictly subdomain-bound experts are designed to retain the localization that Theorem~\ref{thm:moe_decay} motivates.
  \item On standard homogeneous-physics PINN benchmarks Latent-MoE matches established baselines, while on PDEs with multi-stage time-variable physics -- where global baselines collapse to spurious solutions and rigid domain decomposition fails to capture stage-specific dynamics -- it achieves more than an order of magnitude lower relative $L^2$ error, with substantially reduced gradient conflict during training.
\end{itemize}

\section{Preliminaries}
\label{sec:prelim}

\paragraph{Physics-informed neural networks.}
Given a PDE operator $\mathcal{N}$, a forcing function $f$, a boundary operator $\mathcal{B}$, and boundary or initial data $g$ on $\partial\Omega$, PINNs train a neural network surrogate $u_\theta$ to satisfy
\begin{equation}
\mathcal{N}(u_\theta) = f \text{ on } \Omega, \qquad
\mathcal{B}(u_\theta) = g \text{ on } \partial\Omega,
\label{eq:pinn-pde}
\end{equation}
by minimizing a weighted sum of residual losses,
\begin{equation}
\mathcal{L}(\theta) =
\mathbb{E}_{x \in \Omega}\!\left[(\mathcal{N}(u_\theta) - f)^2\right]
+ \sum_i \lambda_i \,
\mathbb{E}_{x \in \partial\Omega}\!\left[(\mathcal{B}(u_\theta)_i - g_i)^2\right],
\label{eq:pinn-loss}
\end{equation}
where the weights $\lambda_i$ may be fixed or adapted during training to stabilize the multi-objective optimization \cite{wang2023expert}. The higher-order automatic differentiation required to evaluate $\mathcal{N}(u_\theta)$ makes the residual loss landscape stiff \cite{wang2021understanding} and motivates the architectural and optimization machinery discussed below.

\paragraph{Neural tangent kernel.}
For a network $f(\cdot;\theta)$ with parameters $\theta$, the empirical
neural tangent kernel (NTK)~\cite{jacot2018neuraltangentkernelconvergence}
is the Gram matrix of parameter-Jacobians,
\begin{equation}
K_\theta(x, x') = \nabla_\theta f(x;\theta)^\top \nabla_\theta f(x';\theta).
\label{eq:ntk}
\end{equation}
Under appropriate scaling, $K_\theta$ converges to a deterministic kernel
$K$ that remains constant during gradient-descent
training~\cite{lee2019wide}, and the trained network at any test point
becomes a kernel-regression predictor whose weights $H(x, X)$ characterize
how training-data influence propagates across the input domain. We
analyze the structure of these weights in
Section~\ref{sec:method}.

\section{Related Work}
\label{sec:related}

\paragraph{Domain decomposition for PINNs.}
The closest line of work to ours partitions the spatio-temporal domain into
subdomains and assigns a separate subnetwork to each. X-PINNs
\cite{jagtap2020extended} enforce solution and residual agreement on the interfaces
between adjacent experts, while FB-PINNs \cite{moseley2023finite} and their
multilevel extension \cite{dolean2024multilevel} use smooth partition-of-unity
windows to blend independently parametrized experts. AB-PINNs
\cite{botvinickgreenhouse2025abpinnsadaptivebasisphysicsinformedneural}
adapt the partitioning to the residual landscape by setting the domain decomposition parameters learnable. All of
these approaches share a common structural choice: the model is split
\emph{end-to-end}, so each subdomain has its own dedicated parameters with
no shared computation. This rigid partitioning localizes capacity but
prevents capacity from being reallocated across regions whose physics
differs in difficulty -- a limitation we identify and address.

\paragraph{Mixture-of-experts.}
Mixture-of-experts (MoE) architectures \cite{jacobs1991adaptive,
jordan1994hierarchical} have recently been revived in large language models
as a route to scaling parameter count while keeping per-token compute fixed
\cite{fedus2022switch}. Modern MoE designs interleave sparsely activated
expert layers within a shared backbone and use \emph{learnable, input-agnostic}
routers trained jointly with the model. Domain-decomposition PINNs can be
viewed as a special case of MoE in which a single MoE layer sits at the
output and the router is a fixed function of the input coordinates. Our
Latent-MoE architecture sits between these two regimes: it borrows the
interleaved-with-shared-backbone structure of language-model MoE, but uses
\emph{fixed, domain-aware} routers that exploit the prior geometric
structure of the PDE problem.

\paragraph{NTK analysis of PINNs.}
The neural tangent
kernel~\cite{jacot2018neuraltangentkernelconvergence,lee2017deep,lee2019wide}
characterizes the learning dynamics of wide networks under gradient
descent and has become a standard tool for analyzing PINN training. Wang
et al.~\cite{wang2022whenandwhypinnsfailtrain, wang2020eigenvectorbiasfourierfeature}
use NTK analysis to expose gradient pathologies in the multi-loss PINN
objective and to explain spectral bias in Fourier-feature networks. More
recent work studies gradient alignment as a second-order optimization
signal~\cite{wang2026gradient} and the spurious-solution failure mode
under stiff physics~\cite{wang2026pinnswrongpseudotimestepping}.

\paragraph{Architectures and training strategies for PINNs.}
Complementary work improves PINN training without changing domain
structure, including PirateNet~\cite{wang2024piratenets}, causal
training~\cite{wang2022respecting}, sinusoidal
activations~\cite{sitzmann2020implicitneuralrepresentationsperiodic}, and
Fourier feature mappings~\cite{tancik2020fourierfeaturesletnetworkshighfrequency};
these are largely orthogonal to ours and Latent-MoE in fact uses Fourier
features and causal weighting in its loss.

\section{Method}
\label{sec:method}

We begin with an analytical diagnosis: standard ReLU coordinate networks
have a translation-variant NTK that produces long-range coupling between
training points, and this coupling is exactly what makes PINNs struggle
when the governing physics varies across the domain
(\S\ref{subsec:method_ntk}). Centered, compact-support routing yields a uniformly
banded NTK whose kernel-regression weights decay exponentially with
distance, eliminating the coupling. We then introduce the
\emph{Latent-MoE} architecture, which interleaves domain-aware MoE blocks
within a shared backbone (\S\ref{subsec:architecture}), and specify the
router (\S\ref{subsec:routers}) and encoder (\S\ref{subsec:encoder}).

\subsection{NTK analysis: long-range coupling and its localization}
\label{subsec:method_ntk}

\paragraph{Translation-variance and long-range coupling.}
Consider the simplest non-trivial case: a single-hidden-layer ReLU network
with bias scale $\sigma_b$. A standard arc-cosine
calculation~\cite{cho2009kernel} (full derivation in
Appendix~\ref{appendix:relu_ntk_analytic}) yields the closed-form NTK
\begin{equation}
  K(x, y) =
  \left(xy + \sigma_b^2\right) \frac{\pi - \arccos(\rho)}{\pi}
  + \frac{\sigma_b \, |x - y|}{2\pi},
  \quad
  \rho =
  \frac{xy + \sigma_b^2}{\sqrt{(x^2 + \sigma_b^2)(y^2 + \sigma_b^2)}}.
  \label{eq:ntk_1d_relu}
\end{equation}
This kernel is translation-variant: for nearby points $y = x + \epsilon$, a
Taylor expansion (see Eq.~\eqref{eq:arccos_approx}, Appendix~\ref{appendix:relu_ntk_analytic})
 gives
$K(x, x + \epsilon) \approx x^2 + x\epsilon + \sigma_b^2 - \sigma_b
|\epsilon|/(2\pi)$, so the local kernel value scales with $|x|^2$. Two
training points at large $|x|$ are strongly correlated even when far apart,
while two points near the origin are only weakly correlated. When the
input coordinate is time, this means later-time data exerts
disproportionate influence on predictions everywhere in the domain -- a
violation of causality that is benign when the physics is homogeneous but
becomes destructive when different temporal regimes have different
dynamics, since gradients from competing regimes are forced into a single
shared kernel.

\paragraph{Banded NTK from domain-aware routing.}
Now consider a mixture of $M$ experts with disjoint or weakly overlapping
supports, blended by a fixed router $\phi$, so that $f(x) = \sum_i
\phi_i(x)\, e_i(x)$ where $e_i$ is the expert gated by router $\phi_i$.
For a hard router with non-trainable parameters, the NTK decomposes as
\begin{equation}
  K_{\text{MoE}}(x, y) = \sum_{i=1}^{M} \phi_i(x)\,\phi_i(y)\, K_i(x, y),
  \label{eq:moe_ntk}
\end{equation}
where $K_i$ is the NTK of expert $e_i$. If the router supports are compact intervals of width $w$ tiled across the
domain -- and, crucially, if each expert is \emph{centered} on its
subdomain so that $K_i$ has the small-magnitude form near the
origin -- then the empirical kernel matrix $K_{\text{MoE}}(X, X)$ becomes
uniformly banded: entries vanish for training point pairs separated by more than
$w$. Banded positive-definite matrices have a classical inverse-decay
property, which we leverage in the following result.

\begin{theorem}[Effective-neighborhood bound for centered MoE]
\label{thm:moe_decay}
Let $X = \{x_1, \dots, x_N\}$ be evenly spaced points with grid spacing
$\Delta x$ on $[A, B]$, and consider an MoE architecture with $M$ experts
whose routers $\phi_i$ have compact, uniformly-spaced supports of width
$w$ that partition $[A,B]$ with bounded overlap (concretely, the bump router of
Eq.~\eqref{eq:bump_router}; see Appendix~\ref{appendix:da_router}). Assume the component kernels $K_i$
are continuous on their router's supports, and constructed as translation of the same kernel, and the minimum eigenvalue of
$K_{\textnormal{MoE}}(X, X)$ is bounded below by $\lambda_0 > 0$. Then
there exist constants $C, \alpha > 0$ such that for every evaluation point
$x \in [A, B]$ and training point $x_j$, the kernel-regression weight
$H_j(x, X) := K_{\textnormal{MoE}}(x, X)\,(K_{\textnormal{MoE}}(X, X) + \lambda I)^{-1}$, where $\lambda \geq 0$, satisfies
\begin{equation}
  |H_j(x, X)| \le C e^{-\alpha |x - x_j|}.
  \label{eq:exp_decay}
\end{equation}
The decay rate $\alpha$ depends on $\lambda_0$, $\Delta x$, the bound on
$K_i$, and the support width $w$.
\end{theorem}
The proof, given in Appendix~\ref{pf:1d_bound}, applies the
Demko--Moss--Smith inverse-decay theorem~\cite{demko1984decay} to the
banded matrix $K_{\text{MoE}}(X, X)$ and uses the compact support of $\phi$
to extend the decay from matrix inverses to kernel-regression weights.

\paragraph{Effective neighborhood width.}\label{def:1d_interval_width}
Theorem~\ref{thm:moe_decay} states that the influence of training points
on the prediction at $x$ decays exponentially in distance. To
quantitatively measure this localization, we define the
\emph{$\alpha$-effective neighborhood width}
\begin{equation}
  B_\alpha(x, X) =
  \min_{i, j}
  \left\{
    x_j - x_i :
    \frac{\sum_{k=i}^{j} H_k(x, X)^2}{\|H(x, X)\|_2^2} \ge 1 - \alpha,
    \; j > i
  \right\},
  \label{eq:eff_neighborhood}
\end{equation}
the size of the smallest interval containing $1 - \alpha$ of the
kernel-regression mass. Smaller $B_\alpha$ means more localized learning.
For the spatio-temporal setting where only the temporal coordinate is
decomposed, an analogous \emph{temporal} width is defined by integrating
out the spatial coordinate (Appendix~\ref{appendix:2d_interval_def}).

\paragraph{Empirical verification.}
Figures~\ref{fig:ntk_diagnosis}(a,b) confirm that MoE
alone does not fix the translation-variant NTK structure, whereas
\emph{centering} the experts produces a uniformly bounded effective
neighborhood across the domain; the eigenspectrum
(Fig.~\ref{fig:ntk_diagnosis}(c)) remains well-spread and
well-conditioned. These diagnostics validate the two design choices that
drive the rest of the paper: \emph{compact-support routing} bands the
kernel, and \emph{centering} keeps the exponential decay from the diagonal uniform. Additional kernel and regression-weight visualizations are provided in
Appendix~\ref{appendix:ntk_diagnosis}.

\begin{figure*}[t]
  \centering
    \centering
\noindent\makebox[\textwidth][c]{%
        \includegraphics[width=1.05\textwidth]{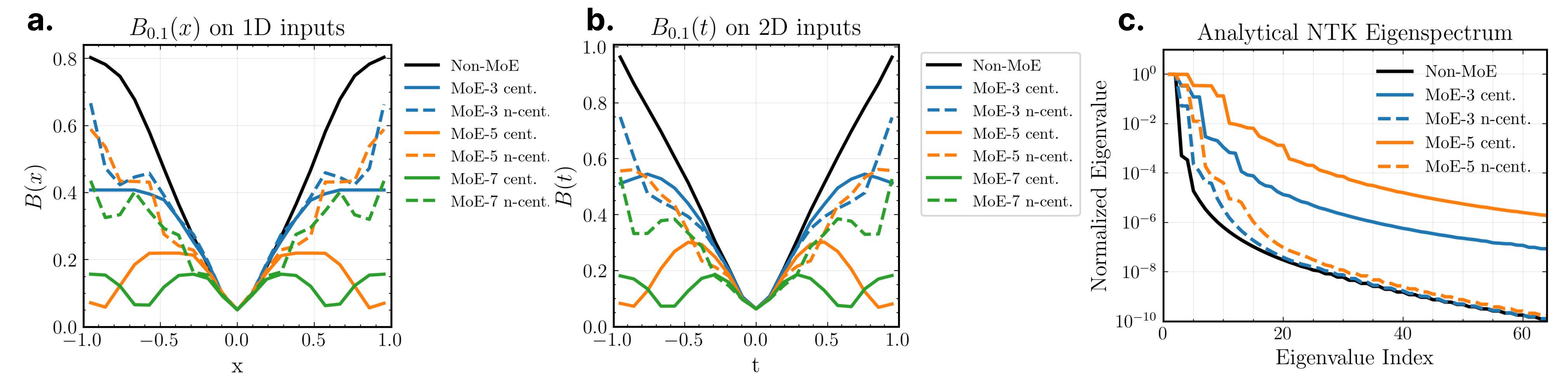}%
    }
  \caption{(a) 1D $\alpha$-effective neighborhood width, $B_{\alpha}(x)$. (b) 2D temporal $\alpha$-effective neighborhood width, $B_{\alpha}(t)$. (c) Analytical NTK eigenspectrum for 1D models. In all cases, weight scale is $1$, and bias scale is $0.05$. We set $\alpha=0.1$, corresponding to the $90\%$ kernel weight concentration.}
  \label{fig:ntk_diagnosis}
\end{figure*}

\subsection{From localization to architecture}
\label{subsec:architecture_motivation}

The NTK analysis says that \emph{centered, compact-support routing}
localizes learning. The most direct way to instantiate this -- used by
FB-PINNs~\cite{moseley2023finite} and X-PINNs~\cite{jagtap2020extended}
-- partitions the model end-to-end and blends experts at the output by a
fixed partition-of-unity. This is restrictive in two ways: per-subdomain
capacity is fixed, so the per-expert width must be set to the worst-case
region; and we hypothesize that errors propagate across subdomain boundaries when early-stage experts are capacity-limited. The variable-$t_1$ heatmaps for FB-PINNs in
Appendix~\ref{sec:variable_stage_transition}
(Fig.~\ref{fig:variable_transition_heatmaps_latest}(c,g)) are consistent
with this picture.

We propose \emph{Latent-MoE} as a remedy. Rather than partitioning the
entire model, Latent-MoE interleaves domain-aware MoE blocks within a
shared backbone, mirroring the interleaved-with-shared-backbone structure
of language-model MoE~\cite{fedus2022switch} but replacing learnable,
input-agnostic routers with fixed routers that exploit the geometric
structure of the PDE problem. The localization benefit, that
Theorem~\ref{thm:moe_decay} theoretically establishes for the idealized model, motivates the design, that is, experts remain strictly
subdomain-bound, while the shared backbone allows representational
capacity to flow flexibly across regions during training.

\subsection{The Latent-MoE architecture}
\label{subsec:architecture}

\begin{wrapfigure}[13]{r}{0.55\textwidth}
  \vspace{-1em}
  \centering
  \includegraphics[width=\linewidth]{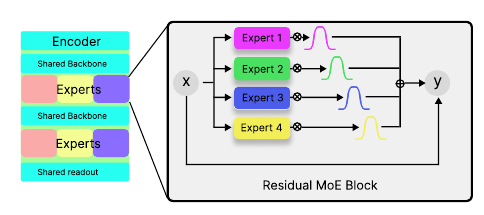}
  \caption{{\small The Latent-MoE architecture (shown with four experts). Domain-aware
  MoE blocks (with fixed, compact-support routers) are interleaved with
  shared residual blocks. Only experts whose routers cover the input
  point contribute to the forward pass.}}
  \label{fig:latent_moe_backbone}
\end{wrapfigure}

The model alternates two types of residual block. A \emph{shared block}
is a standard two-layer feedforward residual block,
\begin{equation}
  f_{\text{shared}}(z) = z + \sigma(W^{(0)} z + b^{(0)}) W^{(1)} + b^{(1)},
\end{equation}
applied uniformly across the input domain. A \emph{MoE block} replaces
the feedforward path with a sparse mixture of $E$ experts $\{e_i\}_{i=1}^E$
gated by a fixed domain-aware router $\phi: \mathbb{R}^d \to \mathbb{R}^E$
(specified in \S\ref{subsec:routers}):
\begin{align}
  &f_{\text{MoE}}(z) = z + \sum_{i=1}^{E} \phi(x)_i \, e_i(z),\\
  &e_i(z) = \sigma(W_i^{(0)} z + b_i^{(0)}) W_i^{(1)} + b_i^{(1)}.
\end{align}
Note that the router is a function of the input coordinate $x$ rather than
the latent $z$, so the routing pattern is fixed once the input is given
and is independent of depth. Because the routers have compact support,
only experts whose support contains $x$ contribute, yielding a sparse
activation pattern. The full network stacks $L$ alternating shared and MoE blocks on top of an
encoder, with a linear readout at the top (see Figure~\ref{fig:latent_moe_backbone}):
\begin{equation}
  z^{(0)} = \text{Encoder}(x),
  \qquad
  z^{(l+1)} = f_{\text{MoE}}^{(l)} \circ f_{\text{shared}}^{(l)}(z^{(l)}),
  \qquad
  u_\theta(x) = W^{(L)} z^{(L)} + b^{(L)}.
\end{equation}

\paragraph{Domain-aware routers.}\label{subsec:routers}
The router $\phi$ assigns each input coordinate to a small number of
experts via a partition-of-unity over the domain. We use compact-support
bump functions with a tunable plateau-to-bridge ratio: each expert $i$ is
assigned an interval $[a_i, b_i]$ on the decomposed coordinate, with a
flat plateau $[a_i + \delta_i, b_i - \delta_i]$ where $w_i \equiv 1$ and
two smooth bridge regions of width $\delta_i = \tfrac{1 - \rho}{2}(b_i - a_i)$
controlled by a polynomial $g_n: [0, 1] \to [0, 1]$ satisfying
$g_n(0) = 0$, $g_n(1) = 1$, and $g_n^{(j)}(0) = g_n^{(j)}(1) = 0$ for
$j = 1, \dots, n$:
\begin{equation}
  w_i(x) =
  \begin{cases}
    g_n\!\left((x - a_i)/\delta_i\right), & a_i \le x < a_i + \delta_i, \\
    1, & a_i + \delta_i \le x \le b_i - \delta_i, \\
    g_n\!\left((b_i - x)/\delta_i\right), & b_i - \delta_i < x \le b_i, \\
    0, & \text{otherwise}.
  \end{cases}
  \label{eq:bump_router}
\end{equation}
The plateau ratio $\rho \in (0, 1)$ controls the overlap between adjacent
experts: $\rho \to 1$ recovers a hard partition, while smaller $\rho$
yields wider blending regions. The plateau ratio values reported in our experiments use a different convention. They are computed as the plateau length divided by the bridge length, which equals $2\rho/(1-\rho)$.
In higher dimensions we can use a separable
router, $w_{\mathbf{i}}(\mathbf{x}) = \prod_{j=1}^d w_{i_j}^{(j)}(x_j)$,
which lets the decomposition act on a chosen subset of input coordinates. The router is normalized
to a partition of unity, $\phi_i(x) = w_i(x) / \sum_j w_j(x)$.
Visualizations of the bump construction and tiling are provided in
Appendix~\ref{appendix:architecture_details}.

\paragraph{Encoder.}\label{subsec:encoder}
The encoder maps raw input coordinates into a feature space suitable for
the MoE backbone. We use a periodic spatial embedding followed by a
trainable random-Fourier-feature
projection~\cite{tancik2020fourierfeaturesletnetworkshighfrequency}; the
exact form is given in Appendix~\ref{appendix:architecture_details}.

The NTK analysis suggests an additional refinement on the encoder side.
Centering each expert's input on its subdomain centroid produces a banded
NTK (\S\ref{subsec:method_ntk}). We carry the same idea into the encoder
by using a per-expert encoder that receives the input shifted by the
corresponding subdomain centroid $c_i$:
\begin{equation}
  z^{(0)} = \sum_{i=1}^{E} \phi_i(x^{\text{dd}}) \,
  g(x^{\text{dd}} - c_i, \, x^{\text{ndd}}),
\end{equation}
where $x^{\text{dd}}$ and $x^{\text{ndd}}$ denote the decomposed and
non-decomposed coordinates. We found that a centered and shared encoder outperforms both a centered but independent encoder and an uncentered shared encoder (Appendix~\ref{appendix:encoder_ablation}). We
attribute this to the encoder's role of producing a consistent
representation that the shared backbone can process; per-expert encoders
with independent random Fourier projections create a fragmented latent
space that disrupts this consistency. Centering and weight sharing are therefore both essential.

\section{Experiments}
\label{sec:experiments}

We evaluate Latent-MoE on three PDE benchmarks designed to stress
architectures with physics that varies vastly across the domain: an
advection--diffusion equation with time-dependent coefficients, a
damped wave equation with staged forcing, and a layered wave equation with a spatially layered wave speed. The first two stress multi-stage time-variable physics, while the layered wave probes a known spatial domain decomposition. All three have periodic spatial
boundary conditions to isolate the effect of the varying physics
from the complication from boundary-loss balancing. Reference solutions are obtained on a fine
grid; we report the relative $L^2$ error
$\|f_{\text{pred}} - f_{\text{true}}\|_2 / \|f_{\text{true}}\|_2$
evaluated on a regular grid.

\paragraph{Baselines.}
We compare against three baselines, all configured to approximately
equal total parameter count of the Latent-MoE.
\textbf{ResNet} is a widened residual network whose feedforward
bottlenecks are widened to reach the target parameter budget.
\textbf{FB-PINNs}~\cite{moseley2023finite} represents the rigid
domain-decomposition family and uses the same domain partitioning as
Latent-MoE (temporal on the time-variable benchmarks, spatial on the layered wave). \textbf{PirateNet}~\cite{wang2024piratenets} is a recent
state-of-the-art architecture for challenging PINN problems. We note that on the time-variable benchmarks PirateNet and FB-PINNs are
configured with slightly \emph{more} parameters than Latent-MoE
($\approx$550K vs $\approx$480K); the order-of-magnitude gap reported
below is therefore not attributable to capacity. Full
hyperparameters and training details are in
Appendix~\ref{appendix:experiment_details}.

\paragraph{Standard benchmarks.}\label{subsec:exp_standard}
To establish that Latent-MoE does not pay a cost on conventional PINN
problems, we first evaluate it on Burgers, Allen--Cahn and Korteweg--de Vries,
three canonical benchmarks with spatially and temporally homogeneous
physics from JAXPI\cite{wang2023expert}. Table~\ref{tab:benchmark_model_stats_main} reports the median relative $L^2$
errors across five seeds. All four architectures reach comparably low relative errors,
and the Latent-MoE model is competitive to the baseline model performance.
The localization machinery in Latent-MoE is benign in this regime: it
neither helps nor materially hurts when there is no cross-region gradient
conflict to resolve.

These results confirm that on problems where the physics is uniform
across the domain, dense global models are already adequate -- there is
no localization benefit to extract. The interesting question, and the
focus of the rest of this section, is what happens when this assumption
breaks down: when the governing operator or the forcing changes
qualitatively across regions of the domain, forcing a single network to
reconcile competing dynamics within a shared parameter set.
\subsection{Advection--diffusion with time-variable coefficients}
\label{subsec:exp_ad}

\begin{table}[t]
  \centering
  \renewcommand{\arraystretch}{1.15}
  \begin{tabular}{lcccc}
    \toprule
    Benchmark & ResNet & PirateNet & FB-PINNs & \textbf{Latent-MoE} \\
    \midrule
    Allen--Cahn            & $\mathbf{2.172 \times 10^{-5}}$ & $2.197 \times 10^{-5}$ & $2.735 \times 10^{-5}$ & $2.739 \times 10^{-5}$ \\
    KdV                    & $\mathbf{4.957 \times 10^{-5}}$ & $9.398 \times 10^{-5}$ & $8.754 \times 10^{-5}$ & $1.290 \times 10^{-4}$ \\
    Burgers                & $\mathbf{4.023 \times 10^{-5}}$ & $4.027 \times 10^{-5}$ & $4.126 \times 10^{-5}$ & $4.044 \times 10^{-5}$ \\
    \midrule
    Advection--Diffusion   & $0.553$ & $0.607$ & $0.373$ & $\mathbf{0.010}$ \\
    Damped Wave            & $0.734$ & $0.740$ & $0.741$ & $\mathbf{0.011}$ \\
    Layered Wave           & $0.147$ & $0.484$ & $0.066$ & $\mathbf{0.036}$ \\
    \bottomrule
  \end{tabular}
  \vspace{4pt}
  \caption{{\small Median relative $L^2$ error across benchmarks. All models are at comparable total parameter counts within the benchmark. Top block: homogeneous-physics benchmarks (5 seeds), where all architectures perform comparably. Bottom block: heterogeneous-physics benchmarks (5 seeds): multi-stage time-variable physics (advection--diffusion, damped wave), where Latent-MoE outperforms every baseline by more than an order of magnitude, and spatially-layered physics (layered wave), where the domain-decomposition models (FB-PINNs, Latent-MoE) lead the global baselines. Full mean$\pm$2\,std in Appendix~\ref{appendix:experiment_details}, Table~\ref{tab:benchmark_model_stats}; the high ResNet variance on damped wave reflects bimodal seed behavior, while PirateNet's tight $0.74$ on damped wave indicates a shared trivial-solution attractor.}}
  \label{tab:benchmark_model_stats_main}
\end{table}

\begin{figure}
  \centering
\noindent\makebox[\textwidth][c]{%
        \includegraphics[width=1.1\textwidth]{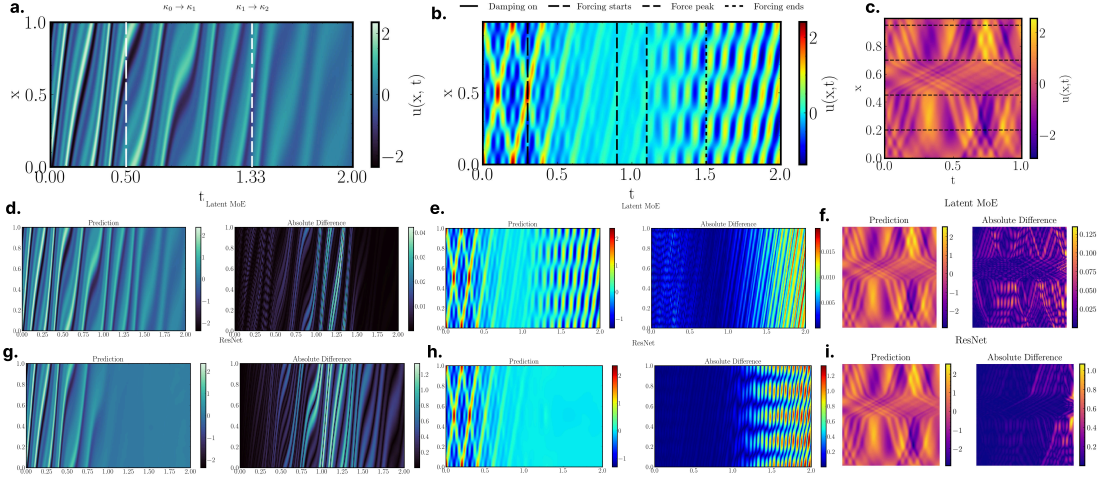}%
    }
  \caption{{\small Predictions and absolute errors across all three benchmarks. (a, b, c)
  Ground-truth solutions for advection--diffusion (AD), damped wave
  (DW), and layered wave (LW); dashed lines mark stage transitions or boundaries. (d, e, f) Latent-MoE predictions
  and pointwise errors. (g, h, i) Parameter-matched widened ResNet
  predictions and pointwise errors. ResNet collapses to a near-trivial
  solution after the first stage transition in both benchmarks with time-variable physics, and fails to learn the fine details brought by the layered medium in LW, while
  Latent-MoE preserves solution fidelity across stages and layers. Visualizations for FB-PINNs and PirateNet can be found in Appendix~\ref{appendix:experiment_details}.}}
  \label{fig:results_aggregate}
\end{figure}
Here we test how architectures handle a time-variable
\emph{differential operator} under zero forcing:
\begin{equation}
  u_t + c(t)\,u_x = \kappa(t)\,u_{xx},
  \qquad x \in [0, 1], \; t \in [0, 2],
\end{equation}
with periodic boundary conditions and the multi-mode initial condition
$u(x, 0) = \sin(2\pi x) + 0.7\sin(4\pi x) + \sin(6\pi x) + 0.25\sin(8\pi x)$.
The advection velocity $c(t)$ is a multi-frequency oscillation with a
linear drift, and the diffusion coefficient $\kappa(t)$ takes three
distinct values across the domain ($10^{-3}$, $5 \times 10^{-3}$,
$5 \times 10^{-2}$) joined by smooth $\tanh$ ramps at $t = 0.5$ and $t = 4/3$.
The full coefficient profiles are shown in
Appendix~\ref{appendix:experiment_details},
Figure~\ref{fig:adv_coefficients_appendix}.

Table~\ref{tab:benchmark_model_stats_main} reports relative $L^2$ errors
across five random seeds. Latent-MoE achieves a median error of
$1.0 \times 10^{-2}$, more than an order of magnitude below all
baselines. Figure~\ref{fig:results_aggregate}(g) shows that ResNet
preserves the early-time dynamics but collapses to a near-trivial solution
once the diffusion coefficient grows, whereas Latent-MoE tracks the
solution through all three stages. Per-seed distributions are in
Appendix~\ref{appendix:experiment_details}.

\subsection{Damped wave equation with staged forcing}
\label{subsec:exp_dw}

The second benchmark adds a time-variable \emph{forcing} on top of a
time-variable operator, testing both the homogeneous and inhomogeneous solution components:
\begin{equation}
  u_{tt} + 2\gamma(t)\,u_t - c^2 u_{xx} + a\,u_x = F(x, t),
  \qquad x \in [0, 1], \; t \in [0, 2],
\end{equation}
with $c = 5$, a small advection $a = 10$ to break axial symmetry of the
frequency signature, and periodic boundary conditions. The system has
three stages after the initial free propagation: a damping stage with $\gamma(t) = 5$ on $t \in [0.3, 0.9]$
and zero elsewhere; a forcing stage in which a resonant external drive
$F(x, t) =-250 \, r(t) \cos(4\pi x)\cos(20\pi t)$ is ramped on between
$t = 0.9$ and $t = 1.1$ and switched off at $t = 1.5$; and free evolution
afterward. The initial condition is
$u(x, 0) = \cos(2\pi x) + 0.5\cos(4\pi x) + 0.8\cos(6\pi x)$.

This benchmark elicits the spurious-solution failure mode characterized
by Wang et al.~\cite{wang2026pinnswrongpseudotimestepping}: ResNet,
PirateNet, and FB-PINNs all converge to near-trivial solutions
(Table~\ref{tab:benchmark_model_stats_main}, errors above 0.7 relative
$L^2$ for all three, with PirateNet's $\pm 0.007$ two-standard-deviation band
indicating that every seed reaches the same trivial-solution attractor).
Figure~\ref{fig:results_aggregate}(h) shows
that the parameter-matched ResNet fails to capture the resonant excitation
introduced by the forcing stage, while Latent-MoE recovers the correct
frequency signature throughout. We interpret this as evidence that the
shared-backbone-with-localized-experts structure provides enough
representational specialization to escape the trivial-solution basin that
single-network architectures fall into on stiff time-variable problems.
\subsection{Layered wave equation}
\label{subsec:exp_lw}

To test whether Latent-MoE's localization benefit extends to \emph{spatially} varying physics with known domain decomposition, we evaluate it on a pure wave equation with a spatially layered wave speed:
\begin{equation}
  u_{tt} = c(x)^2\, u_{xx}, \qquad x \in [0, 1),\; t \in [0, 1],
\end{equation}
where $c(x)$ is a periodic, piecewise-smooth function that takes four distinct values $\{2.0, 4.0, 0.5, 3.0\}$ across four spatial layers. These layers use the same bump functions defined in Eq.~\eqref{eq:bump_router} (\S \ref{subsec:architecture}), with a plateau ratio of $1.5$ (the plateau length divided by the bridge length) to render a smooth transition. The initial displacement $u(x,0)$ is a sinusoidal function whose local spatial frequency slightly adapts to the layer structure, with zero initial velocity $u_t(x,0)=0$; see Appendix~\ref{appendix:lw_details} for the full specification. Unlike the AD and DW benchmarks, where stage transitions are temporal and unknown to the model, the layered wave (LW) benchmark decomposes the \emph{spatial} axis. For FB-PINNs and Latent-MoE, we configure the domain-aware routers and the four experts to align exactly with these spatial layers. Additionally, we evaluate model performance when this architectural bias misaligns with the true physics, either by shifting the expert subdomains with various offsets or by incorrectly configuring the layer overlap ratio.

Table~\ref{tab:lw_ablation} reports median relative $L^2$ errors across three seeds. Thanks to its flexible capacity allocation, Latent-MoE does not require exact alignment with the physical subdomains; a coarse spatial ansatz suffices, significantly reducing implementation overhead. Interestingly, we find that for both MoE models, bump function routers with plateau ratio of $1$ outperforms the native bump function with plateau ratio of $1.5$ of the actual physics, implying that exact alignment in the domain decomposition perspective may not be optimal for the actual performance.

\begin{table}
  \centering
  \renewcommand{\arraystretch}{1.15}
    \begin{tabular}{lccccccc}
      \toprule
      & \multicolumn{3}{c}{Plateau ratio} & \multicolumn{4}{c}{Spatial shift} \\
      \cmidrule(lr){2-4}\cmidrule(lr){5-8}
      Value & 1.000 & \textbf{1.500} & 2.000 & \textbf{0.000} & 0.050 & 0.100 & 0.125 \\
      \midrule
      FBPINNs & 0.059 & 0.066 & 0.081 & 0.066 & 0.064 & 0.078 & 0.096 \\
      L-MoE & 0.025 & 0.036 & 0.065 & 0.036 & 0.062 & 0.042 & 0.038 \\
      \bottomrule
    \end{tabular}
  \vspace{4pt}
  \caption{{\small Median relative $L^2$ error on the layered-wave benchmark (3 seeds, except for the models with the exact domain decomposition, marked as bold values). Left block: plateau-ratio sensitivity at zero boundary offset, with ground truth ratio $1.5$. Right block: sensitivity to a rigid spatial offset $\delta$ of the bump-router boundaries from the physical layer transitions (plateau ratio fixed at $1.5$).}}
  \label{tab:lw_ablation}
\end{table}

\subsection{Robustness to stage-transition placement}
\label{subsec:exp_variable_t1}
\begin{table}
  \centering
  \renewcommand{\arraystretch}{1.15}
  \begin{tabular}{lccccccc}
    \toprule
    & \multicolumn{7}{c}{First stage transition $t_1$} \\ 
    \cmidrule(lr){2-8}
    & $0.1$ & $0.2$ & $0.3$ & $0.4$ & $0.5$ & $0.6$ & $0.7$ \\ 
    \midrule
    \multicolumn{8}{l}{\emph{Advection--Diffusion}} \\
    \textbf{Latent-MoE} & $\mathbf{0.002}$ & $\mathbf{0.002}$ & $\mathbf{0.004}$ & $\mathbf{0.006}$ & $\mathbf{0.023}$ & $\mathbf{0.013}$ & $\mathbf{0.005}$ \\
    ResNet & $0.421$ & $0.306$ & $0.509$ & $0.356$ & $0.531$ & $0.593$ & $0.477$ \\
    FB-PINNs & $0.192$ & $0.239$ & $0.288$ & $0.424$ & $0.409$ & $0.403$ & $0.393$ \\
    PirateNet & $0.413$ & $0.476$ & $0.503$ & $0.526$ & $0.551$ & $0.554$ & $0.608$ \\
    \midrule
    \multicolumn{8}{l}{\emph{Damped Wave}} \\
    \textbf{Latent-MoE} & $\mathbf{0.013}$ & $\mathbf{0.026}$ & $\mathbf{0.015}$ & $\mathbf{0.012}$ & $\mathbf{0.037}$ & $\mathbf{0.030}$ & $\mathbf{0.028}$ \\
    ResNet & $0.890$ & $0.818$ & $0.658$ & $0.430$ & $0.572$ & $0.478$ & $0.706$ \\
    FB-PINNs & $0.877$ & $0.803$ & $0.748$ & $0.655$ & $0.664$ & $0.664$ & $0.724$ \\
    PirateNet & $0.890$ & $0.819$ & $0.737$ & $0.664$ & $0.580$ & $0.548$ & $0.827$ \\
    \bottomrule
  \end{tabular}
  \vspace{4pt}
  \caption{{\small Median relative $L^2$ error across 3 seeds as the first
  stage-transition time $t_1$ is shifted, while keeping inter-stage gaps
  fixed. The Latent-MoE temporal subdomains are held fixed across all
  $t_1$ values, so most settings deliberately misalign the physical
  staging with the architectural decomposition. Latent-MoE remains in
  the $10^{-3}$--$10^{-2}$ regime throughout.}}
  \label{tab:variable_median_l2}
\end{table}

A natural concern with these benchmarks is that the stage transitions
might happen to align with the Latent-MoE subdomain boundaries, in which
case the method would benefit from a hand-chosen partition. To rule this
out, we vary the first transition time $t_1$ over the range
$[0.1, 0.7]$ -- which deliberately misaligns it with the fixed temporal
subdomain boundaries used by Latent-MoE.
Table~\ref{tab:variable_median_l2} reports median relative $L^2$ error
across three seeds for each $t_1$ value.\footnote{The $t_1 = 0.5$ column for advection--diffusion and the $t_1 = 0.3$ column for damped wave coincide with the default benchmarks in Table~\ref{tab:benchmark_model_stats_main}. The medians differ slightly because this sweep uses three seeds instead of five.} Latent-MoE remains in the
$10^{-3}$--$10^{-2}$ range across all settings, while ResNet errors stay
above $0.3$. Appendix~\ref{sec:variable_stage_transition} provides visualizations of the windowed relative $L^2$-error for the 4 models.

\subsection{Gradient conflict and the localization hypothesis}
\label{subsec:exp_gradient_conflict}

\begin{wrapfigure}[13]{r}{0.50\textwidth}
  \vspace{-1.2em}
  \centering
  \begin{subfigure}[t]{0.48\linewidth}
    \centering
    \includegraphics[width=\linewidth]{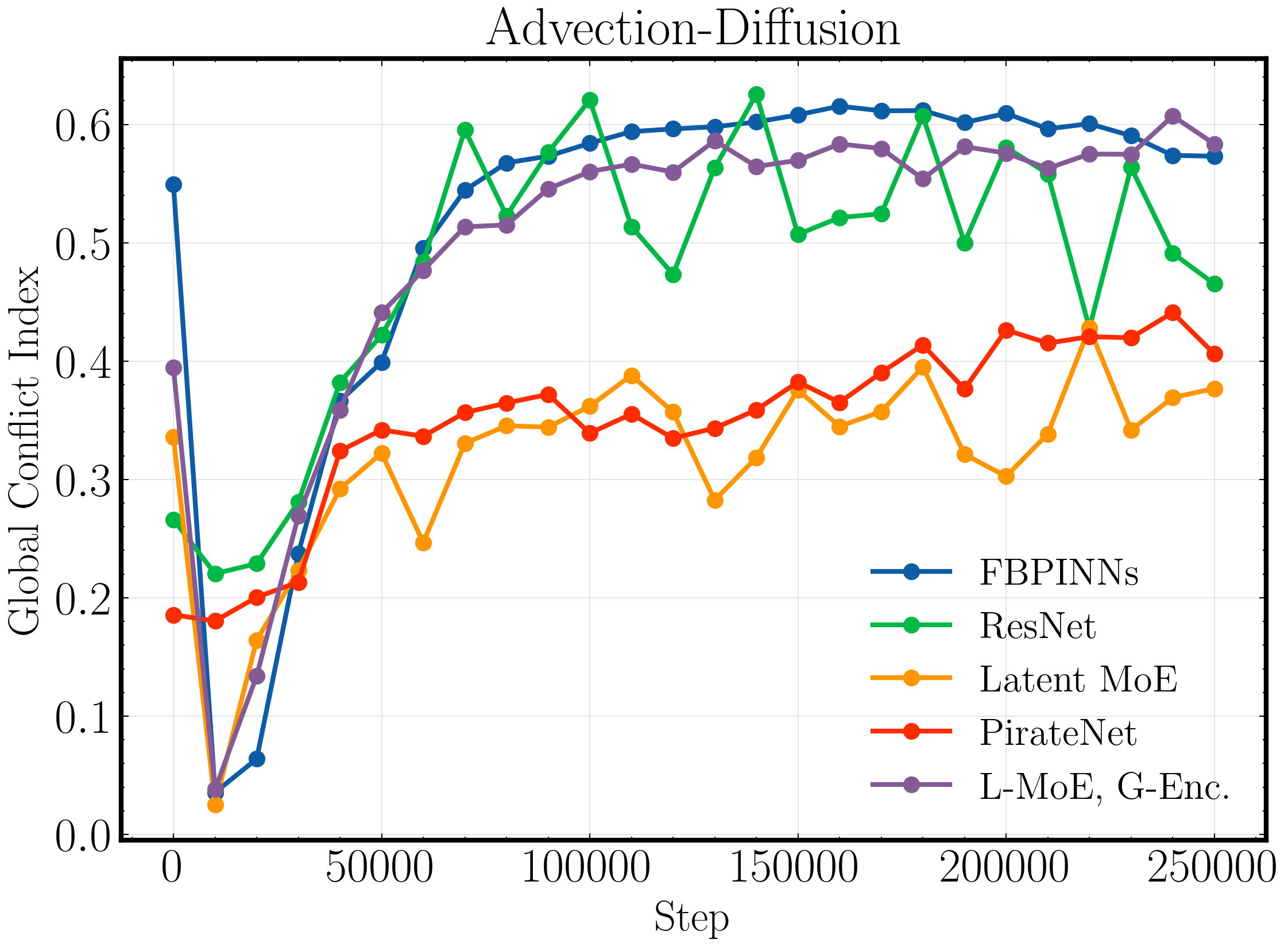}
    \caption{Advection--diffusion}
  \end{subfigure}
  \hfill
  \begin{subfigure}[t]{0.48\linewidth}
    \centering
    \includegraphics[width=\linewidth]{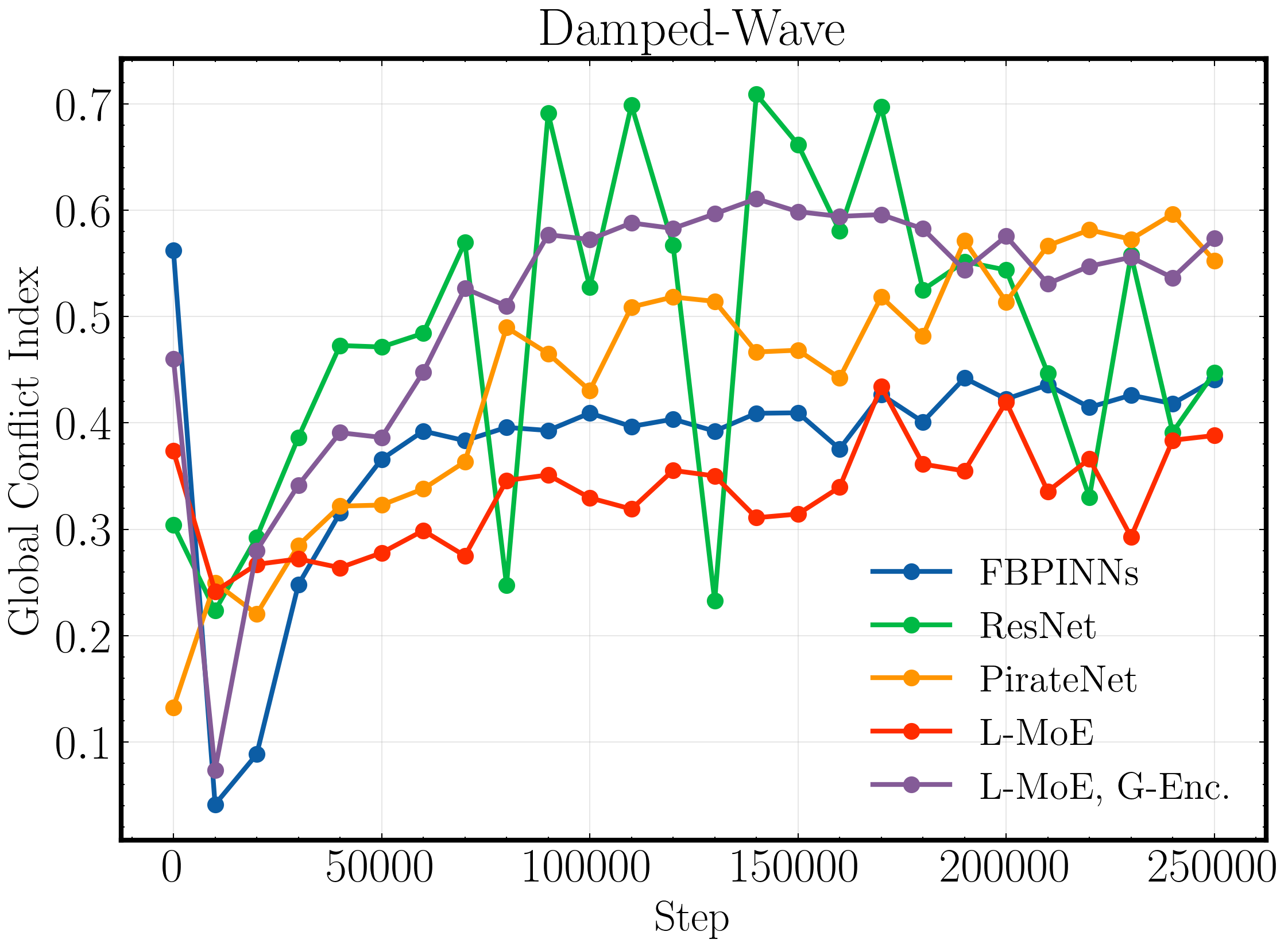}
    \caption{Damped wave}
  \end{subfigure}
  \caption{{\small Gradient-conflict index over training. Lower is better. L-MoE G-Enc. stands for Latent-MoE with global encoder that is not expert-centered. Latent-MoE consistently exhibits the least cross-chunk cancellation.}}
  \label{fig:grad_conflict_curves_main}
\end{wrapfigure}

Our central claim is that domain-aware routing eliminates the long-range
coupling between training
points that pits gradients from different
regions against each other. We test this directly by measuring gradient
conflict during training. Following the gradient-alignment perspective
of Wang et al.~\cite{wang2026gradient}, we partition the temporal domain
into $c$ non-overlapping chunks (we use $c=8$; see 
Appendix~\ref{app:gradient_conflict_index}), compute the residual-loss Jacobian
$J_i \in \mathbb{R}^p$ on each chunk, and define the
\emph{gradient conflict index}
\begin{equation}
  G(J) = 1 - \frac{\left\|\sum_i J_{i, \cdot}\right\|}
                  {\sum_i \left\|J_{i, \cdot}\right\|}.
  \label{eq:grad_conflict}
\end{equation}
$G = 0$ corresponds to perfectly aligned gradients across chunks; values
near $1$ indicate that gradients from different temporal regions cancel
each other during the optimization step
(Appendix~\ref{app:gradient_conflict_index} gives an asymptotic NTK
interpretation).

Figure~\ref{fig:grad_conflict_curves_main} shows $G(J)$ throughout
training for all four architectures and the global-encoder ablation on both benchmarks. Latent-MoE
maintains lower conflict than the dense baselines (ResNet, PirateNet) on both problems, with a clear margin on the damped wave and a smaller margin over PirateNet on advection--diffusion, consistent with its localized NTK structure.
FB-PINNs registers high $G$ on advection--diffusion as well, but for a structurally different
reason: its experts have disjoint parameters, so the per-chunk Jacobians
$J_i$ have nearly disjoint support and the metric mechanically inflates
toward $1 - 1/\sqrt{c}$ regardless of training dynamics. The diagnostic distinguishes Latent-MoE not just from globally coupled models
but also from rigid decompositions: the shared backbone is what enables
genuine gradient alignment across regions, rather than mere isolation.
Replacing the per-expert centered encoder with a single global encoder
(L-MoE G-Enc, Fig.~\ref{fig:grad_conflict_curves_main}) raises the
gradient-conflict index toward the dense-baseline range. This isolates
\emph{centering}, not MoE structure alone, as the mechanism that
suppresses cross-chunk gradient cancellation -- in direct agreement with
Theorem~\ref{thm:moe_decay}.

\section{Discussion}
\label{sec:discussion}

We have proposed Latent-MoE, a PINN architecture that interleaves
domain-aware mixture-of-experts blocks within a shared backbone. The
design is motivated by an NTK analysis: standard coordinate networks
have translation-variant kernels that produce long-range coupling
between training points, and centered compact-support routing localizes
learning by yielding a banded NTK whose kernel-regression weights decay
exponentially with distance. The shared backbone preserves the
flexibility to allocate capacity across regions, which we identify as
the central limitation of rigidly partitioned approaches such as
FB-PINNs.

The benefit is regime-specific. On standard PINN benchmarks with
homogeneous physics, Latent-MoE matches established baselines: when
there is no cross-region gradient conflict to resolve, the localization
machinery does not help. The advantage emerges when the governing
equations or forcing change qualitatively across the domain -- the
regime in which the global baselines we tested (ResNet, PirateNet) collapse to spurious solutions
while Latent-MoE preserves the dynamics, with reduced gradient conflict
throughout training. When the heterogeneity is spatial and the domain decomposition is known (the layered wave benchmark), rigid decomposition (FB-PINNs) is also effective; yet Latent-MoE remains competitive without requiring its experts to align exactly with the physical layers, indicating that a coarse spatial ansatz suffices. We view this scope -- problems whose physics is
well-defined but varies vastly across time or space -- as the natural
target of the architecture, with practical relevance to wave
propagation through layered media, reactive flows with
temperature-dependent kinetics, and fluid systems with time-variable
forcing.

\paragraph{Limitations.}
Our quantitative claims rest on three synthetic benchmarks designed to
isolate heterogeneous physics -- two with multi-stage time-variable physics and one with a spatially layered wave speed -- supplemented by Burgers,
Allen--Cahn and KdV as homogeneous-physics controls. The variable-$t_1$
and gradient-conflict experiments mitigate cherry-picking concerns, but
evaluation on real engineering or geophysical problems remains the
natural next step. The router supports are also fixed bump functions
tiled along a single (temporal or spatial) axis; adaptive decomposition driven by the
residual landscape~\cite{botvinickgreenhouse2025abpinnsadaptivebasisphysicsinformedneural,
wang2026adaptive}, joint spatio-temporal partitioning, and
non-uniform per-expert capacity are straightforward extensions we have
not pursued. Finally, the NTK analysis assumes infinite width and
squared-error loss; the gradient-conflict diagnostic provides a
training-centric proxy that tracks the analytical predictions on our
benchmarks, but a finite-width characterization under the residual loss
would close a real gap.

\paragraph{Broader Outlook.}
The broader observation underlying this work is that architectural
machinery developed for scaling language models -- sparse activation,
expert routing, interleaved specialization -- can be repurposed for
scientific computing when the routing is informed by physical structure
rather than learned from scratch, expanding the range of multi-regime
problems where PINNs are a viable modeling tool.



\begin{ack}
This work is supported by the U.S. Department of Energy (DOE), Office of Advanced Scientific Computing Research (ASCR) under award number DE-SC0024563.
\end{ack}

\newpage
\bibliographystyle{unsrt}
\bibliography{references}
\newpage
\include{appendix}





\end{document}

%% file: appendix.tex
\appendix

\section{Neural Tangent Kernel Theory}\label{appendix:relu_ntk_analytic}
\subsection{Preliminaries}
Without loss of generality, we only consider single-output neural networks without residual connections; however, the convergence results also apply to multi-output networks and architectures with residual connections\cite{barzilai2023kernelperspectiveskipconnectionsconvolutional, belfer2024spectralanalysisneuraltangentkernelresidualnetworks}. For an $L$-layer fully connected neural network $f(\bm{x}; \bm{\theta}): \mathbb{R}^{d} \rightarrow \mathbb{R}$ with parameters $\bm{\theta}$, let the layer input be $x^{(l)}$ for $l=0,\ldots,L$, where $x^{(0)}$ is the network input and $x^{(L)}$ is the post-activation output of the last hidden layer. Let the weight matrix and bias vector at layer $l$ be $W^l\in\mathbb R^{n_l\times n_{l+1}}$ and $b^l\in\mathbb R^{n_{l+1}}$, respectively, where $n_l$ is the width of layer $l$. The forward pass is defined as
\begin{align}\label{def:ffl}
\begin{split}
  &h^{l} = \frac{\sigma_w^l}{\sqrt{n_l}} x^l W^l + \sigma_b^l b^l,\\
  &x^{l+1} = \sigma(h^{l}), \quad l = 0,...,L-1,\\
  &f(x; \bm{\theta}) = h^{L} =  \frac{\sigma_w^L}{\sqrt{n_{L}}}x^{L}W^{L} + \sigma_b^L b^{L},
\end{split}
\end{align}
where $h^l$ is the pre-activation at layer $l$, $\sigma(\cdot)$ is the activation function, and $s\sigma_w^l, \sigma_b^l$ are scales of the weight and bias, which themselves are initialized as i.i.d Gaussian. The last layer is a linear readout layer and therefore has no activation. NTK analysis assumes i.i.d. weight and bias initialization. The collection of all parameters is denoted by $\bm{\theta} = \{W^l, b^l\}_{l=0}^{L}$. The neural tangent kernel $K$ is defined as the inner product of Jacobians at two input points over all parameters:
\begin{equation}
  K_{\bm\theta}(\bm{x}, \bm{x}') = \nabla_{\bm{\theta}} f(\bm{x}; \bm{\theta})^\top \nabla_{\bm{\theta}} f(\bm{x}'; \bm{\theta}).
\end{equation}
Given a training dataset $X\in\mathbb {R}^{N\times d}$ and corresponding labels $Y\in\mathbb{R}^{N}$, the empirical NTK matrix is defined as
\begin{equation}
  K_{\bm\theta}(X, X)_{i,j} = K_{\bm\theta}(x_i, x_j),
\end{equation}
that is, $K_{\bm\theta}(X, X) = \nabla_{\bm{\theta}} f_{\bm{\theta}}(X)^\top \nabla_{\bm{\theta}} f_{\bm{\theta}}(X)$.

We assume that the continuous time gradient descent training follows the dynamics
\begin{equation}
  \dot{\bm{\theta}} = - \eta\nabla_{\bm{\theta}} \mathcal{L}(\bm{\theta}),
\end{equation}
where $\eta$ is the learning rate (equivalently, the time-scaling factor in continuous time), and $\mathcal{L}(\bm{\theta})$ is the summed squared-error loss\footnote{The analysis can be extended to cross-entropy loss as well \cite{jacot2018neuraltangentkernelconvergence}.}.
\begin{equation}
  \mathcal{L}(\bm{\theta}) = \frac{1}{2}\sum_{i=1}^{N}(f(x_i; \bm{\theta}) - y_i)^2.
\end{equation}
Under such training dynamics, the evolution of the network parameters and its predictions on the training set can be described by the following ODE system, linearized around the initialization $\bm{\theta}(0)$:
\begin{align}
  &\dot{\bm{\theta}} = - \eta \nabla_{\bm{\theta}} f_{\bm{\theta}}(X)^\intercal (f(X; \bm{\theta}) - Y),\\
  &\dot{f}(X; \bm{\theta}) = \nabla_{\bm \theta} f(X) \dot{\bm{\theta}} = -\eta K_{\bm{\theta}}(X, X)(f(X; \bm{\theta}) - Y).\label{eq:ntk_ode}
\end{align}
The NTK theory states that under certain conditions, as the width of each layer $n_l \rightarrow \infty$, the NTK $K_{\bm{\theta}}$ converges to a deterministic kernel matrix $K$ and remains constant during the training with continuous gradient descent dynamics. We denote the limiting kernel as $K$. Therefore, the training dynamics of the network predictions can be obtained at the initialization as
\begin{equation}
  \dot{f}(X; \bm{\theta}) = -\eta K(X, X)(f(X; \bm{\theta}) - Y),
\end{equation}
and the solution to the ODE is given as
\begin{equation}
  f(X; \bm{\theta}(t))- Y = e^{-\eta K(X, X) t}(f(X; \bm{\theta}(0)) - Y),
\end{equation}
One prominent property of the NTK is translation variance, which implies uneven learning of the target function across the input domain. We use a simple one-hidden-layer fully connected ReLU network to illustrate this property. General recursion formulas for NNGP\cite{lee2017deep} and NTK\cite{jacot2018neuraltangentkernelconvergence,lee2019wide} can be found in the corresponding literature.
\begin{theorem}[Recursive NTK formula for Fully Connected Networks]\label{thm:ntk_recursive}
  Let the neural tangent kernel at layer $L$ be denoted as $K^L(x, x')$, and the covariance matrix of the pre-activation at layer $L$ be denoted as $\Sigma^L(x, x')$, and the covariance matrix of the pre-activation derivatives at layer $L$ be denoted as $\dot{\Sigma}^L(x, x')$, then the NTK at layer $L$ can be computed recursively as
  $$K^{L+1}(x, x') = K^{L}(x, x')\dot{\Sigma}^{L+1}(x, x') + \Sigma^{L+1}(x, x'),$$
  where $K^{1}(x,x') = (\sigma_b^0)^2 + \frac{(\sigma_w^0)^2}{d} xx'$ and $d$ is the input dimension.
\end{theorem}

\begin{proof}
  See section 4.1 of {\it Jacot et al.} \cite{jacot2018neuraltangentkernelconvergence}.
\end{proof}
In this corollary, we provide the exact closed-form expression for the NTK of a two-layer ReLU network with bias, which is used in the main text to analyze NTK locality and the effect of bias on kernel structure. The proof technique is based on the well-known arc-cosine kernel\cite{cho2009kernel}.
\begin{corollary}[Analytic NTK for Single-Layer ReLU Networks]\label{cor:relu_ntk_analytic}
Consider a one hidden layer neural network defined above with ReLU activation $\sigma(z) = \max(0, z)$, unit Gaussian initialization for weights $w_i^{(0)}$ and $w_i^{(1)}$, and Normal bias initialization with zero mean and standard deviation $\sigma_b$.\footnote{Due to the scale-invariance of the ReLU activation, the initialization variance offset for the inner and outer weights can be absorbed into the bias initialization.} That is, the network is defined as
\begin{equation}
  f(x) = \sigma(xw^{(0)} + \sigma_b b)w^{(1)},
\end{equation}
its limiting NTK $K(x,y)$ is given by
\begin{align}
  &K(x,y) = \frac{\sigma_b |x-y|}{2\pi} + \left(xy + \left(\sigma_b\right)^2\right)\frac{\pi - \arccos{\rho}}{\pi},\\
  &\rho = \frac{xy + \left(\sigma_b\right)^2}{\sqrt{\left(x^2 + \left(\sigma_b\right)^2\right)\left(y^2 + \left(\sigma_b\right)^2\right)}}.
\end{align}
\end{corollary}

\begin{proof}

Let $z_i = w^{(0)}_i x + b_i$ and $z'_i = w^{(0)}_i y + b_i$. Since the weights and biases are initialized i.i.d. from $\mathcal{N}(0,1)$, we have that $z_i$ and $z'_i$ are jointly Gaussian with zero mean and covariance matrix $\Sigma$, where $\Sigma_{11} = x^2 + \left(\sigma_b\right)^2$, $\Sigma_{22} = y^2 + \left(\sigma_b\right)^2$, and $\Sigma_{12} = \Sigma_{21} = xy + \left(\sigma_b\right)^2$, and the correlation coefficient
\begin{align}
  \rho = \frac{xy + \left(\sigma_b\right)^2}{\sqrt{\left(x^2 + \left(\sigma_b\right)^2\right)\left(y^2 + \left(\sigma_b\right)^2\right)}}.
\end{align}
We first derive NTK component for the first layer weights $w^{(0)}$ at infinite width. By the Strong Law of Large Numbers, as $n\to\infty$,
\begin{align}
  \langle \nabla_{w^{(0)}} f(x), \nabla_{w^{(0)}} f(y) \rangle &= \mathbb{E}_{z, z'}\left[\dot\sigma(z)\dot\sigma(z') w^{(1)}_i w^{(1)} x y\right],\\
  &= xy \mathbb{E}_{z, z'}\left[\dot\sigma(z)\dot\sigma(z')\right].
\end{align}
Similarly for the first layer bias $b$, as $n\to\infty$,
\begin{align}
  \langle \nabla_{b} f(x), \nabla_{b} f(y) \rangle &= \mathbb{E}_{z, z'}\left[\dot\sigma(z)\dot\sigma(z') w^{(1)}_i w^{(1)}\right],\\
  &= \left(\sigma_b\right)^2\mathbb{E}_{z, z'}\left[\dot\sigma(z)\dot\sigma(z')\right].
\end{align}
Given that the $\dot\sigma$ is a step function from $0$ to $1$ at $0$, $\dot\sigma(z)\dot\sigma(z')$ is an indicator function of the event that $z>0$ and $z'>0$. Therefore, we can geometrically compute that
\begin{align}
  \mathbb{E}_{z, z'}\left[\dot\sigma(z)\dot\sigma(z')\right] &= \mathbb{P}(z>0, z'>0)\\
  &= \frac{\pi - \arccos(\rho)}{2\pi}.
\end{align}
We then derive the NTK component for the second layer weights $w^{(1)}$ at infinite width.
\begin{align}
  \langle \nabla_{w^{(1)}} f(x), \nabla_{w^{(1)}} f(y) \rangle &= \mathbb{E}_{z, z'}\left[\sigma(z)\sigma(z')\right],\\
  &= \sqrt{\Sigma_{11}\Sigma_{22}} \mathbb{E}_{z, z'}\left[\sigma\left(\frac{z}{\sqrt{\Sigma_{11}}}\right)\sigma\left(\frac{z'}{\sqrt{\Sigma_{22}}}\right)\right].
\end{align}
Let $U, V$ be two standard Gaussian random variables with correlation coefficient $\rho$, then
\begin{align}
  \mathbb{E}_{z, z'}\left[\sigma(z)\sigma(z')\right] &= \sqrt{\Sigma_{11}\Sigma_{22}} \mathbb{E}_{U, V}\left[\sigma\left(U\right)\sigma\left(V\right)\right].
\end{align}
Using Price's theorem, we compute the derivative of $\mathbb{E}_{U, V}\left[\sigma\left(U\right)\sigma\left(V\right)\right]$ with respect to $\rho$ as
\begin{align}
  \frac{\partial}{\partial \rho} \mathbb{E}_{U, V}\left[\sigma\left(U\right)\sigma\left(V\right)\right] &= \mathbb{E}_{U, V}\left[\dot\sigma\left(U\right)\dot\sigma\left(V\right)\right],\\
  &= \frac{\pi - \arccos(\rho)}{2\pi}.
\end{align}
We can then integrate the above equation with respect to $\rho$ to obtain
\begin{align}
  \mathbb{E}_{U, V}\left[\sigma\left(U\right)\sigma\left(V\right)\right] &= \int_0^\rho \frac{\pi - \arccos(t)}{2\pi} dt + C,\\
  &= \frac{1}{2\pi}\left(\rho \pi + \sqrt{1-\rho^2} - \rho \arccos(\rho)\right) + C.
\end{align}
Using the boundary condition of $\mathbb{E}_{U, V}\left[\sigma\left(U\right)\sigma\left(V\right)\right] = \frac{1}{2}$ when $\rho=1$, we can solve for the constant $C$ and derive the final expression for $\mathbb{E}_{U, V}\left[\sigma\left(U\right)\sigma\left(V\right)\right]$ as
\begin{align}
  \mathbb{E}_{U, V}\left[\sigma\left(U\right)\sigma\left(V\right)\right] &= \frac{1}{2\pi}\left(\sqrt{1-\rho^2} + \rho(\pi - \arccos(\rho))\right).\label{eq:ntk_EUV}
\end{align}

Summing up the above components, we have that the NTK of the two-layer ReLU network at infinite width is
\begin{align}
  &K(x,y) = (\sigma_b^2 + xy)\frac{\pi - \arccos(\rho)}{2\pi} + \frac{\sqrt{\Sigma_{11}\Sigma_{22}}}{2\pi}\left(\sqrt{1-\rho^2} + \rho(\pi - \arccos(\rho))\right),\\
  &= \Sigma_{12}\frac{\pi - \arccos(\rho)}{2\pi} + \frac{\sqrt{\Sigma_{11}\Sigma_{22}}}{2\pi}\left(\sqrt{1-\rho^2} + \rho(\pi - \arccos(\rho))\right),\\
  &=\Sigma_{12}\frac{\pi - \arccos(\rho)}{\pi} + \frac{\sqrt{\Sigma_{11}\Sigma_{22}}}{2\pi}\sqrt{1-\rho^2},\\
  &=(xy+\sigma_b^2)\frac{\pi - \arccos(\rho)}{\pi} + \frac{\sigma_b|x-y|}{2\pi},
\end{align}
where the last line is obtained by plugging in the definition of $\rho,\Sigma_{\cdot}$ and simplifying.

\end{proof}
With the analytic NTK formula, we can quantitatively show that the NTK kernel is highly translation-variant with scale growing quadratically with the input magnitude. 
\begin{corollary}[Linearized NTK approximation for Single-Layer ReLU Networks]\label{cor:relu_ntk_locality}
With $y = x + \epsilon$, where $x > 0$ and $\epsilon$ is a perturbation term that is sufficiently small, then
\begin{align}
K(x, x+\epsilon) \approx (x^2 + x\epsilon + \sigma_b^2) - \frac{\sigma_b|\epsilon|}{2\pi}
\end{align}
\end{corollary}

\begin{proof}
Again we denote $\Sigma_{11} = x^2 + \sigma_b^2$. Then
\begin{align}
  \rho &= \frac{xy + \sigma_b^2}{\sqrt{\left(x^2 + \sigma_b^2\right)\left(y^2 + \sigma_b^2\right)}}\\
  &= \frac{x^2 + x\epsilon + \sigma_b^2}{\sqrt{\left(x^2 + \sigma_b^2\right)\left((x+\epsilon)^2 + \sigma_b^2\right)}}\\
  &= \frac{\Sigma_{11} + x\epsilon}{\sqrt{\Sigma_{11} \left(\Sigma_{11} + 2x\epsilon + \epsilon^2\right)}}\\
  &=\left(1 + \frac{x\epsilon}{\Sigma_{11}}\right)\left(1 + \frac{2x\epsilon + \epsilon^2}{\Sigma_{11}}\right)^{-\frac{1}{2}}\\
  &= \left(1 + \frac{x\epsilon}{\Sigma_{11}}\right)\left(1 - \frac{2x\epsilon + \epsilon^2}{2\Sigma_{11}} + \frac{3}{8}\frac{\epsilon^2(2x+\epsilon)^2}{\Sigma_{11}^2} + \mathcal O(\epsilon^3)\right)\ \text{by Taylor expansion,}\\
  &= \left(1 + \frac{x\epsilon}{\Sigma_{11}}\right)(1 - \frac{x\epsilon}{\Sigma_{11}} - \frac{\epsilon^2}{2\Sigma_{11}} + \frac{3\epsilon^2x^2}{2\Sigma_{11}^2}+ \mathcal O(\epsilon^3))\\
  &= 1 - \frac{\epsilon^2}{2\Sigma_{11}^2}(\Sigma_{11} - x^2) + \mathcal O(\epsilon^3)\\
  &= 1 - \frac{\epsilon^2}{2\Sigma_{11}^2}\sigma_b^2 + \mathcal O(\epsilon^3).
\end{align}
Using the smaller angle approximation $\arccos(1-\delta) \approx \sqrt{2\delta}$ for sufficiently small $\delta$, we have that
\begin{align}\label{eq:arccos_approx}
  \arccos(\rho) &= \arccos(1 - \frac{\epsilon^2}{2\Sigma_{11}^2}\sigma_b^2 + \mathcal O(\epsilon^3))\approx \frac{|\epsilon|}{\Sigma_{11}}\sigma_b.
\end{align}

Plugging the above approximation and $\Sigma_{12} = (x+\epsilon)x + \sigma_b^2$ into the NTK formula in Corollary \ref{cor:relu_ntk_analytic}, we have that

\begin{align}
  K(x, x+\epsilon) &\approx (x^2 + x\epsilon + \sigma_b^2)\frac{\pi - \frac{|\epsilon|}{\Sigma_{11}}\sigma_b}{\pi} + \frac{\sigma_b|\epsilon|}{2\pi},\\
  &= (x^2 + x\epsilon + \sigma_b^2) - \frac{\sigma_b|\epsilon|}{2\pi} + \mathcal O(\epsilon^2).
\end{align}
\end{proof}
Corollary~\ref{cor:relu_ntk_locality} suggests that the linearized NTK approximation grows with input magnitude $x$ when $\sigma_b$ is small and $x$ dominates. This trend is clearer when the inner bias scale is set to $0$, in which case the NTK reduces to $K(x,y)=xy$ if $xy>0$ and $0$ otherwise, and the linearized approximation is exactly $K(x,x+\epsilon)= x^2 + x\epsilon$. When the input is time, this implies that the NTK assigns higher correlation to temporally distant data points, which can induce spurious correlations and violate causality. Theorem~\ref{thm:moe_decay} shows that by centering each expert at its subdomain center and making its support compact, the NTK becomes banded and its kernel-regression weights exhibit translation-based exponential decay with distance between the evaluation point and training points. In short, domain decomposition mitigates the network's tendency to disproportionately emphasize later temporal data by localizing experts to specific subdomains.

\subsection{Proof of Theorem \ref{thm:moe_decay}}\label{pf:1d_bound}

\begin{proof}
Let $K_{\text{MoE}}(X, X) \in \mathbb R^{N\times N}$ be the composite NTK matrix of the MoE architecture evaluated on the training dataset $X$. By the assumption of a positive minimum eigenvalue $\lambda_0 > 0$, we have that $K_{\text{MoE}}(X, X)$ is positive definite and therefore invertible, since it is a Gram matrix of the Jacobian at the datapoints with respect to the parameters. Let $x_j, x_k$ be two training points in $X$. Then 
\begin{align}
  K_{\text{MoE}}(X, X)_{j,k} &= K_{\text{MoE}}(x_j, x_k),\\
  &= \sum_{m=1}^M \phi_m(x_j)\phi_m(x_k) K_m(x_j, x_k).
\end{align}
Given that $\phi_m$s are supported on evenly spaced compact intervals of uniform width $w$, the off-block entries of $K_{\text{MoE}}(X, X)$ are zeroed. That is, if $|x_j - x_k| > w$, then $K_{\text{MoE}}(X, X)_{j,k} = 0$. The NTK matrix $K_{\text{MoE}}$ is therefore $\frac{w}{\Delta x}$-banded.

By assumption, each $K_i$ is continuous on the compact support of
$\phi_i$, hence uniformly bounded by some $L_i > 0$. The translational
construction of the subdomains (each $K_i$ is the same kernel translated
to centroid $c_i$) gives a common bound $L = \max_i L_i$, independent of
the discretization. Since at most two routers are active at any point,
$|K_{\textnormal{MoE}}(x, x')| \le 2L$ uniformly on the domain.

Demko-Moss-Smith theorem~\cite{demko1984decay} states that for such a banded matrix that is positive definite, bounded, and boundedly invertible, the entries of its inverse also decay exponentially away from the diagonal. More formally, there exists constants $C_0> 0$ and $\alpha > 0$ such that for any $j, k \in \{1, \dots, N\}$,
\begin{equation}
  |[K_{\text{MoE}}(X, X) + \lambda I]^{-1}_{j,k}| \le C_0 e^{-\alpha|j-k|}.
\end{equation}
Similarly, given an evaluation point $x$, the partial kernel $K_{\text{MoE}}(x, X) \in \mathbb R^{1\times N}$ is also banded with the same bandwidth $w$, since the router $\phi_m$ is supported on compact intervals of width $w$. Therefore, if $|x-x_j| > w$, then $K_{\text{MoE}}(x, x_j) = 0$. Let $I(x) = \{x_j\in X| |x_j - x| \le w\}$. Then for the any evaluation point $x$, $|I(x)| \leq \frac{2w}{\Delta x} + 1$. The equivalent kernel regression weight for a training point $x_j$
\begin{align}
  |H_j(x, X)| &= \left|(K_{\text{MoE}}(x, X) [K_{\text{MoE}}(X, X) + \lambda I]^{-1})_j\right|,\\
  &\leq \sum_{k\in I(x)} |K_{\text{MoE}}(x, x_k)|\cdot |([K_{\text{MoE}}(X, X) + \lambda I]^{-1})_{k,j}|,\\
  &\leq \sum_{k\in I(x)} L C_0 e^{-\alpha|x_k-x_j|}, \quad \text{where $L$ is the uniform bound of $K_i$ on the subdomain},\\
  &\leq \sum_{k\in I(x)} L C_0 e^{-\alpha(|x-x_j| - |x-x_k|)},\\
  &\leq |I(x)|L C_0 e^{-\alpha(|x-x_j| - w)},\\
  &\leq C e^{-\alpha|x-x_j|},\quad \text{where $C = L C_0 (\frac{2w}{\Delta x} + 1) e^{\alpha w}$}.
\end{align}
\end{proof}
\subsection{2D effective neighborhood width}\label{appendix:2d_interval_def}
Similar to the 1D effective width (\S\ref{def:1d_interval_width}), we define the temporal width for a 2D space-time model in which only the temporal domain is partitioned. Let the input be $x=[t,s]$, where $t$ is time and $s$ is space, and let $H([s,t],X)$ denote the NTK-regression weight matrix with shape $N_t\times N_s$. We define the temporal $\alpha$-effective neighborhood width as
\begin{equation}
  B_\alpha([s,t],X) = \min_{i,j} \left\{t_j - t_i:\sum_{k=i}^{j} \frac{\|H([s,t],X)_{k,\cdot}\|_2^2}{\|H([s,t],X)\|_2^2} \ge 1-\alpha,\ j>i\right\}.
\end{equation}

\subsection{NTK diagnosis}\label{appendix:ntk_diagnosis}
\begin{figure}[t]
  \centering
  \includegraphics[width=0.9\textwidth]{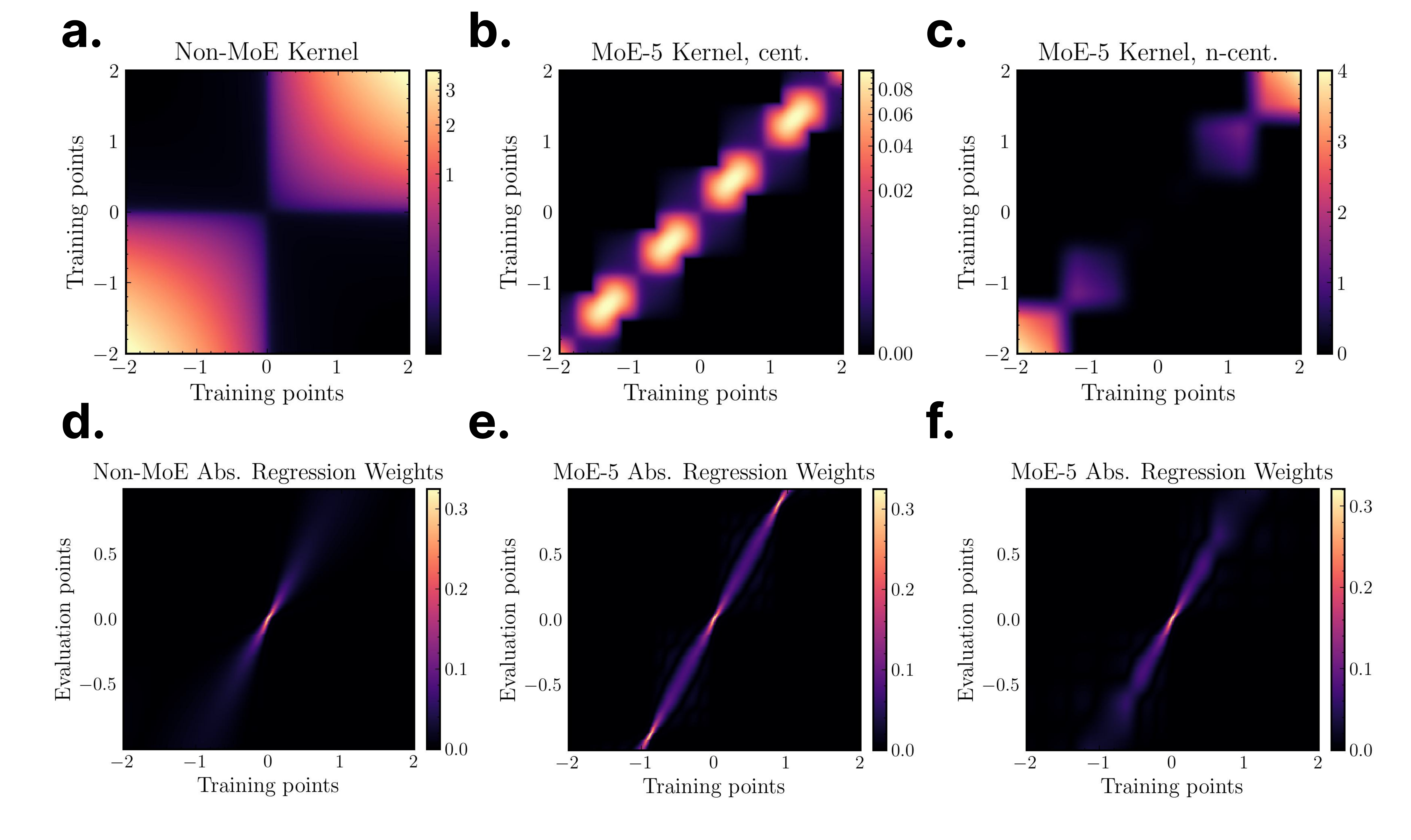}
  \caption{NTK analysis for domain decomposition MoE. (a-c) NTKs of a one-hidden-layer ReLU network, centered MoE, and non-centered MoE. (d-f) NTK regression weights of the same architectures. In all cases, weight scale is $1$, and bias scale is $0.05$.}
  \label{fig:ntk_diagnostics}
\end{figure}
To complement the analytical results in Corollary~\ref{cor:relu_ntk_analytic}, Corollary~\ref{cor:relu_ntk_locality}, and Theorem~\ref{thm:moe_decay}, we visualize the NTK and the corresponding kernel-regression weights for a non-MoE model as well as MoE models with varying numbers of experts, both centered and non-centered. Figure~\ref{fig:ntk_diagnostics} shows that while the non-centered MoE suppresses off-diagonal correlations in the Jacobian, the unboundedness of the diagonal and near-diagonal elements still produces inhomogeneous kernel-regression weights across the domain. In contrast, the subdomain-centric centering successfully introduces translational invariance into the kernel-regression weights. Figure~\ref{fig:appendix_ntk_matrix_comparison} further demonstrates that increasing the number of experts progressively uniformizes the learning behavior.
\begin{figure*}[t]
  \centering
  {\small MoE NTK, centered \par}
  \includegraphics[width=0.19\linewidth]{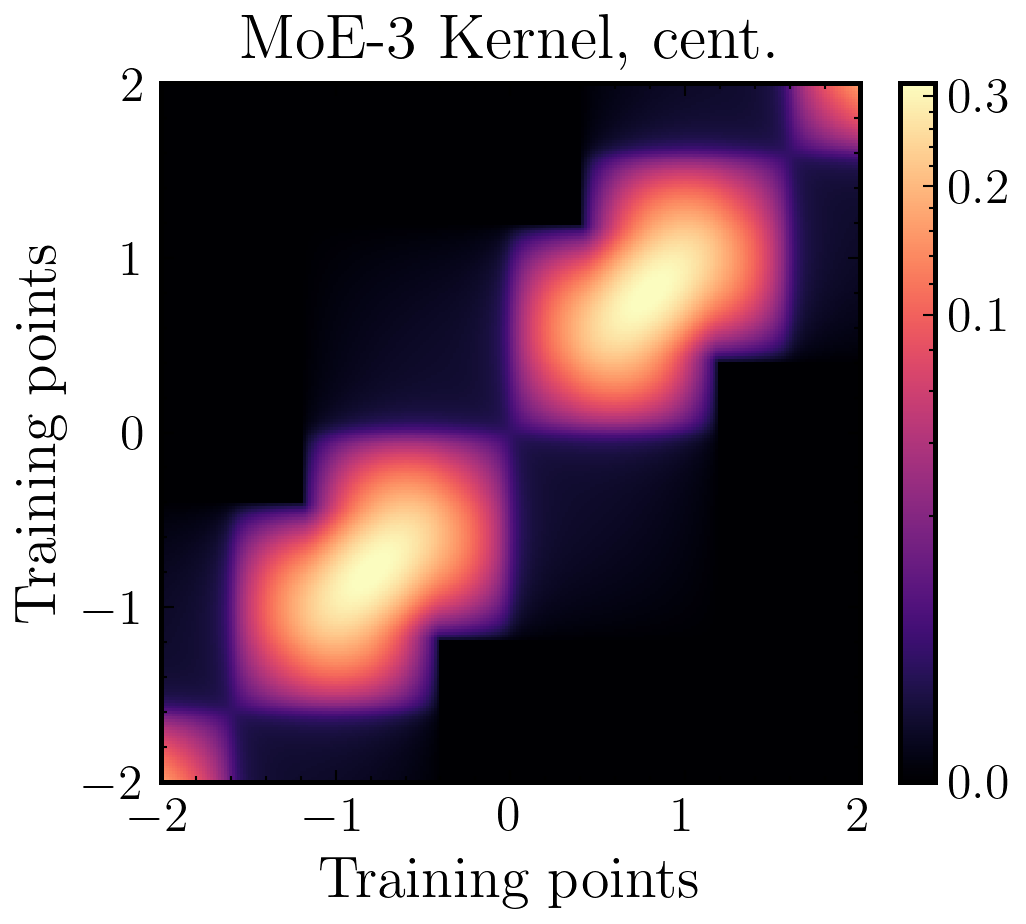}
  \includegraphics[width=0.19\linewidth]{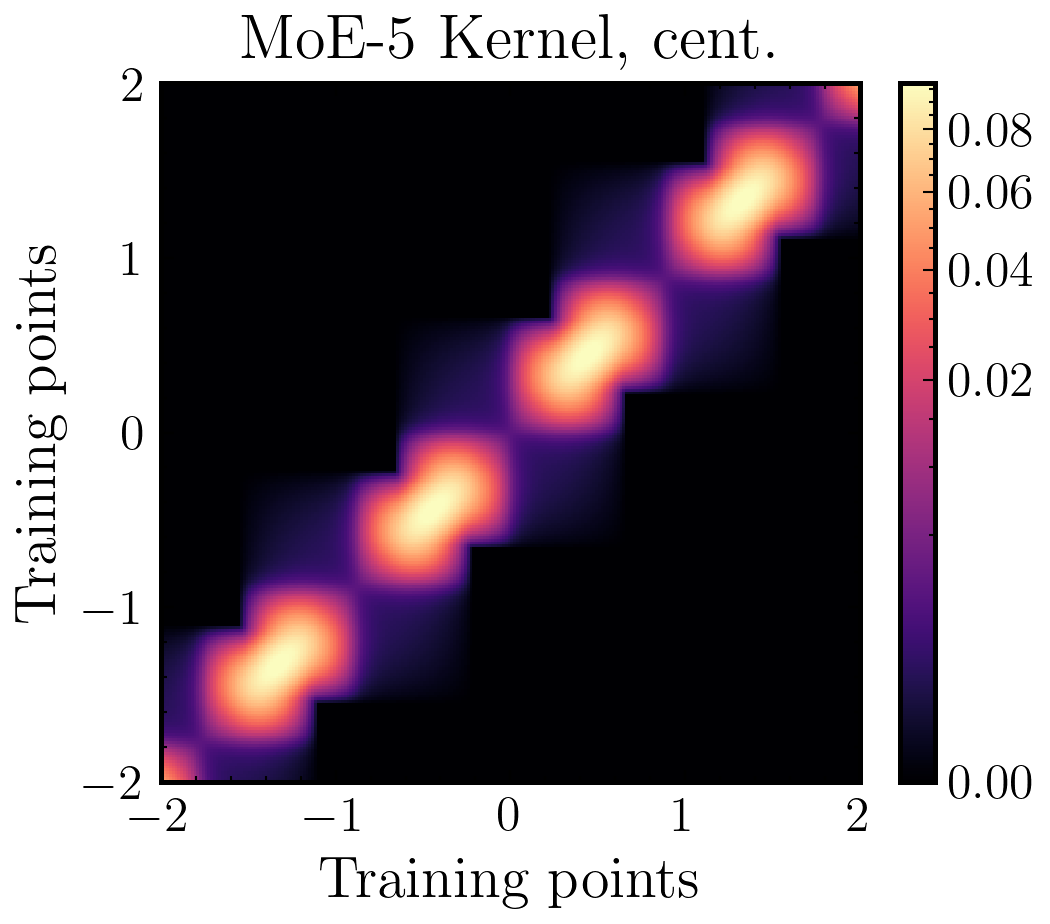}
  \includegraphics[width=0.19\linewidth]{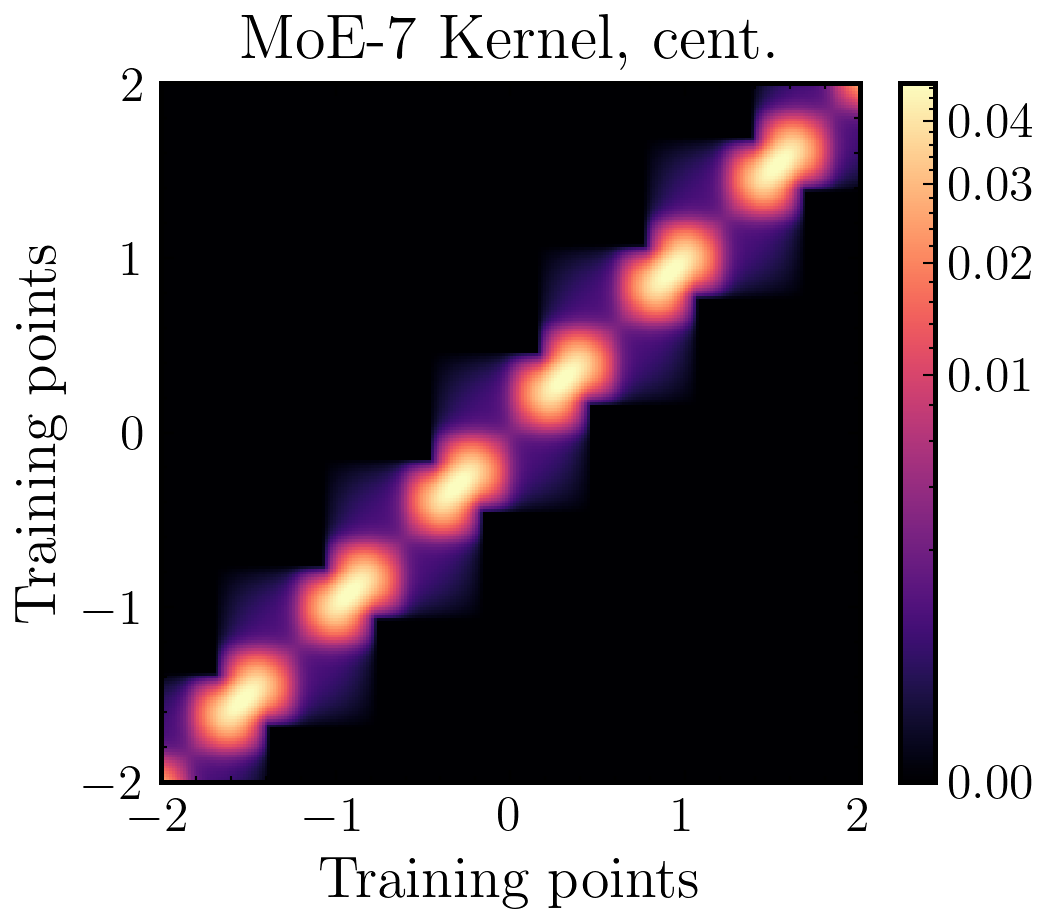}
  \includegraphics[width=0.19\linewidth]{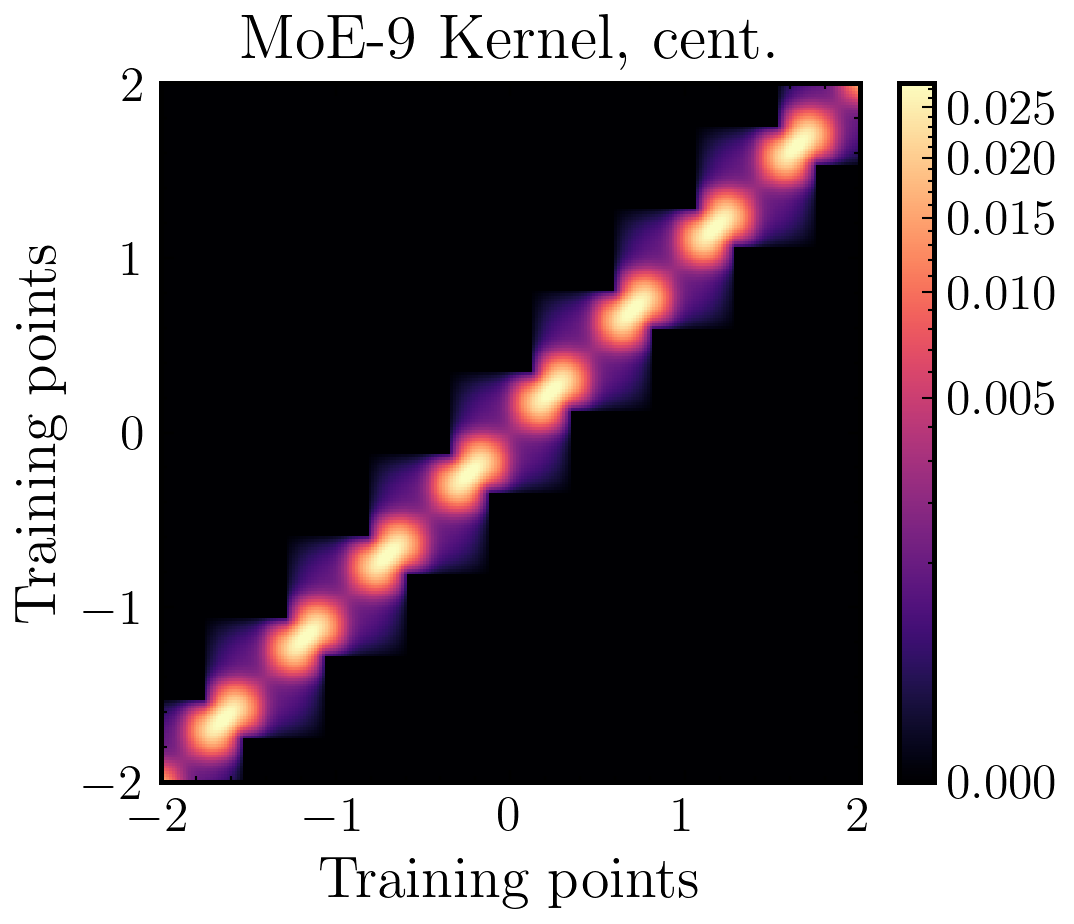}
  \includegraphics[width=0.19\linewidth]{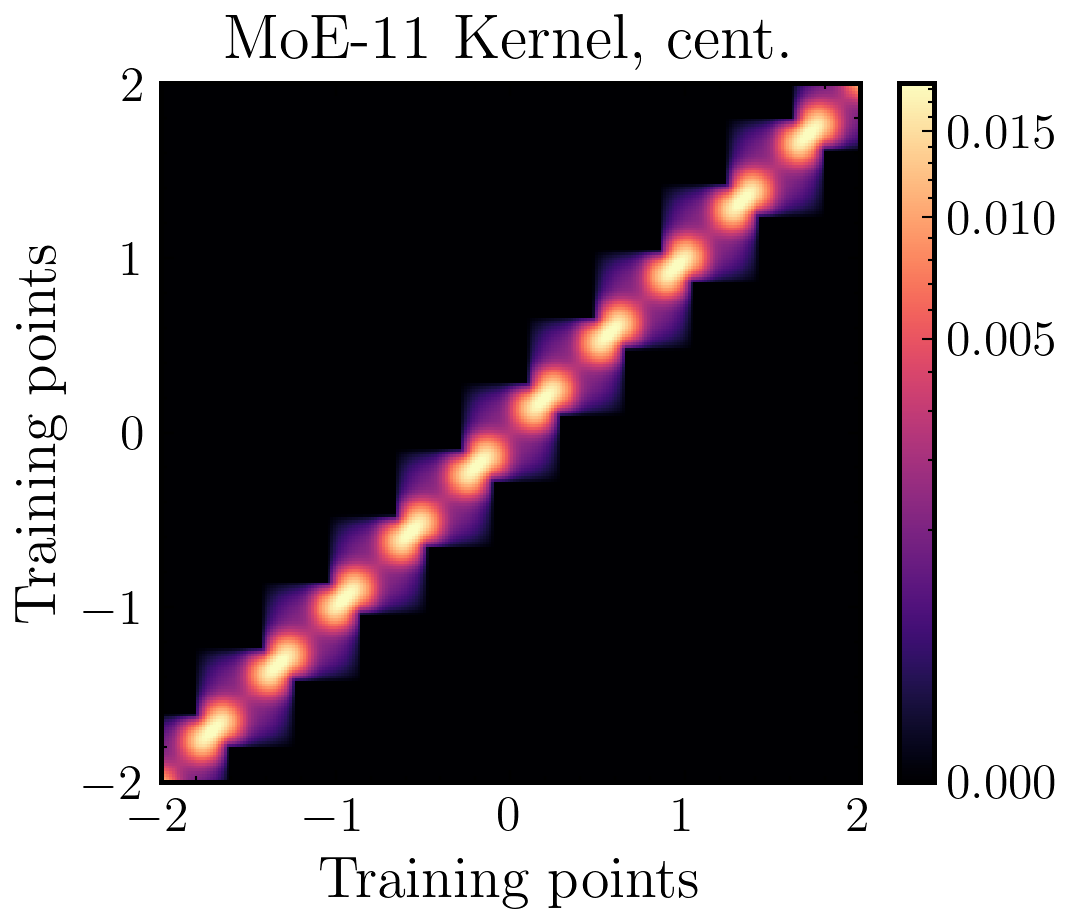}

  {\small MoE NTK kernel-regression absolute weights, centered \par}
  \includegraphics[width=0.19\linewidth]{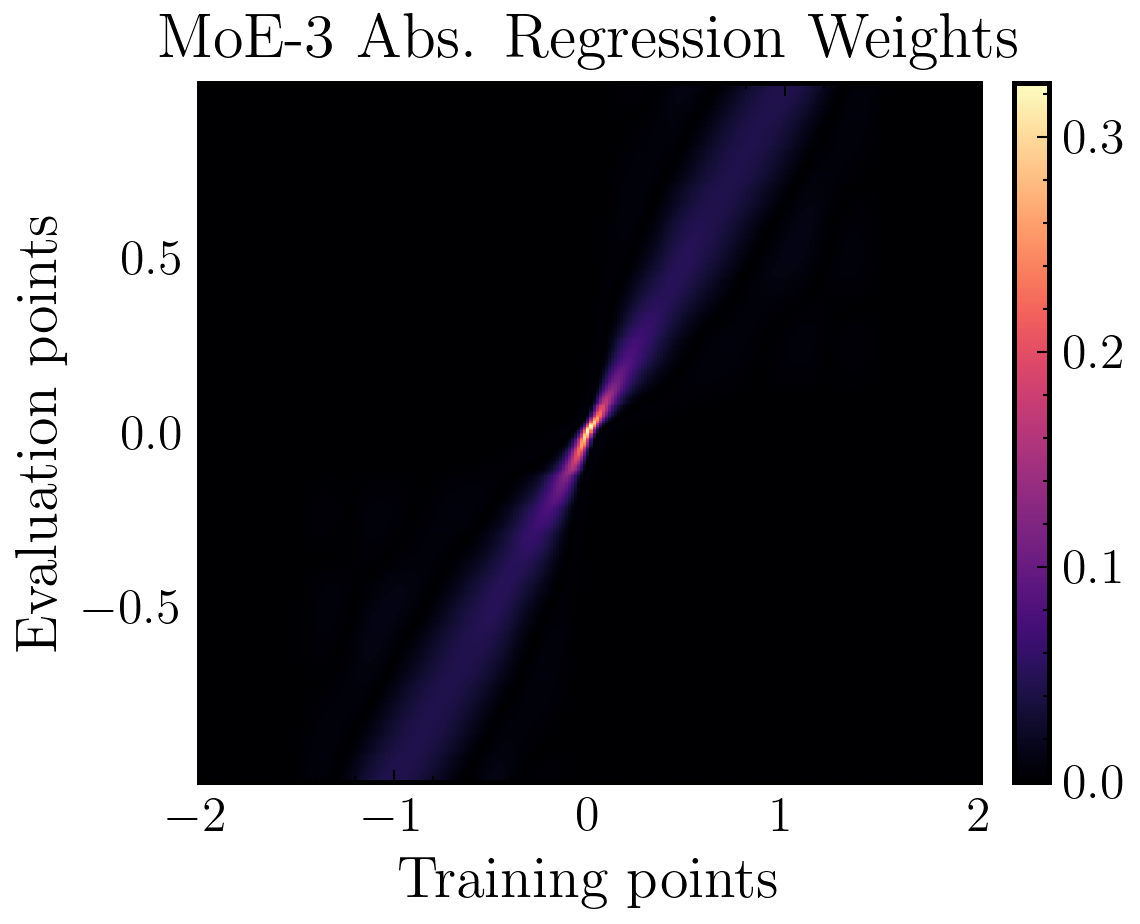}
  \includegraphics[width=0.19\linewidth]{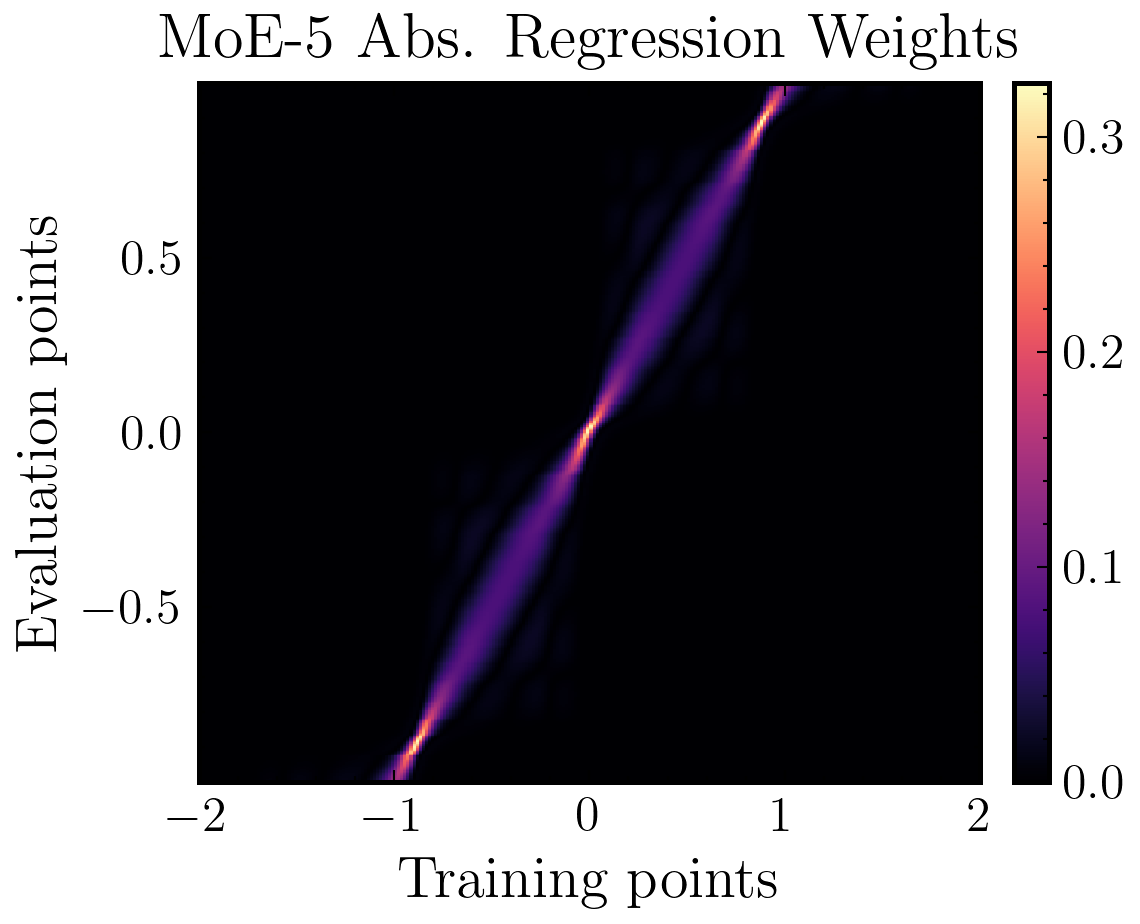}
  \includegraphics[width=0.19\linewidth]{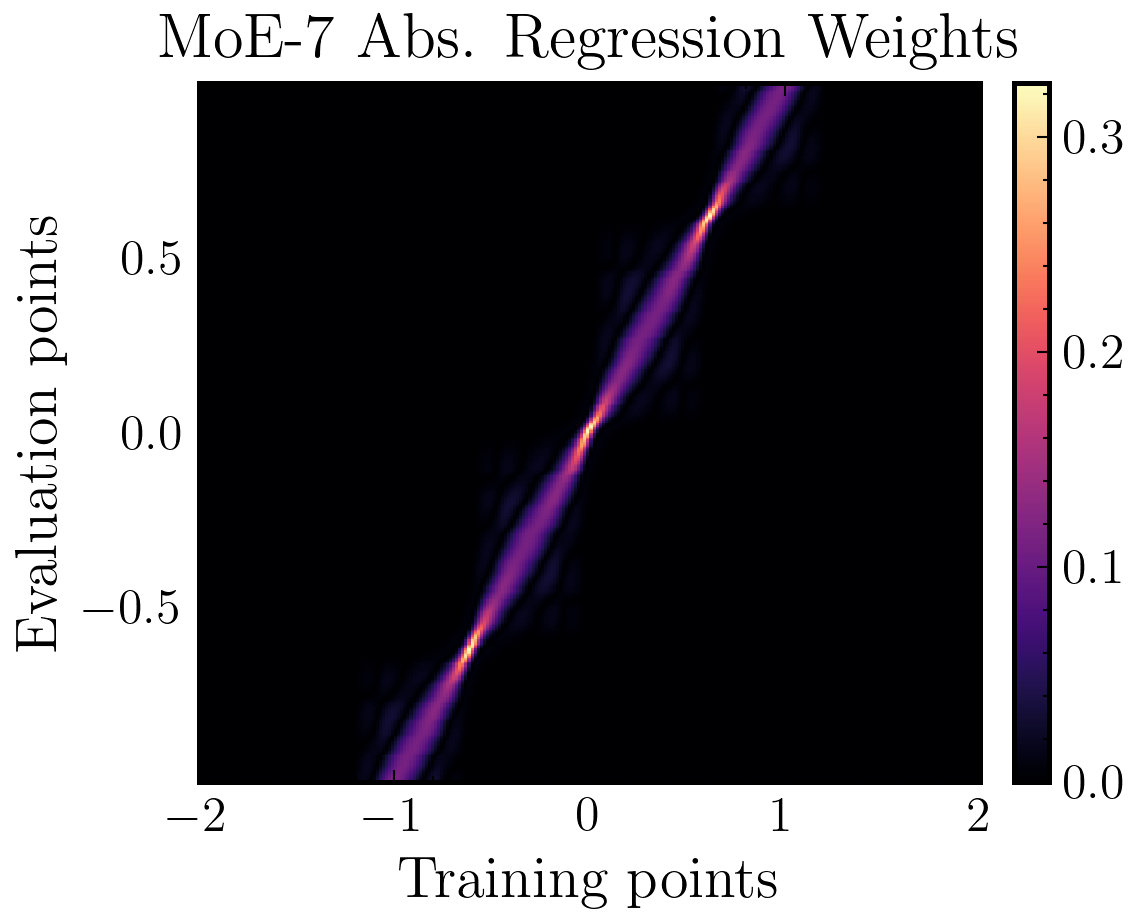}
  \includegraphics[width=0.19\linewidth]{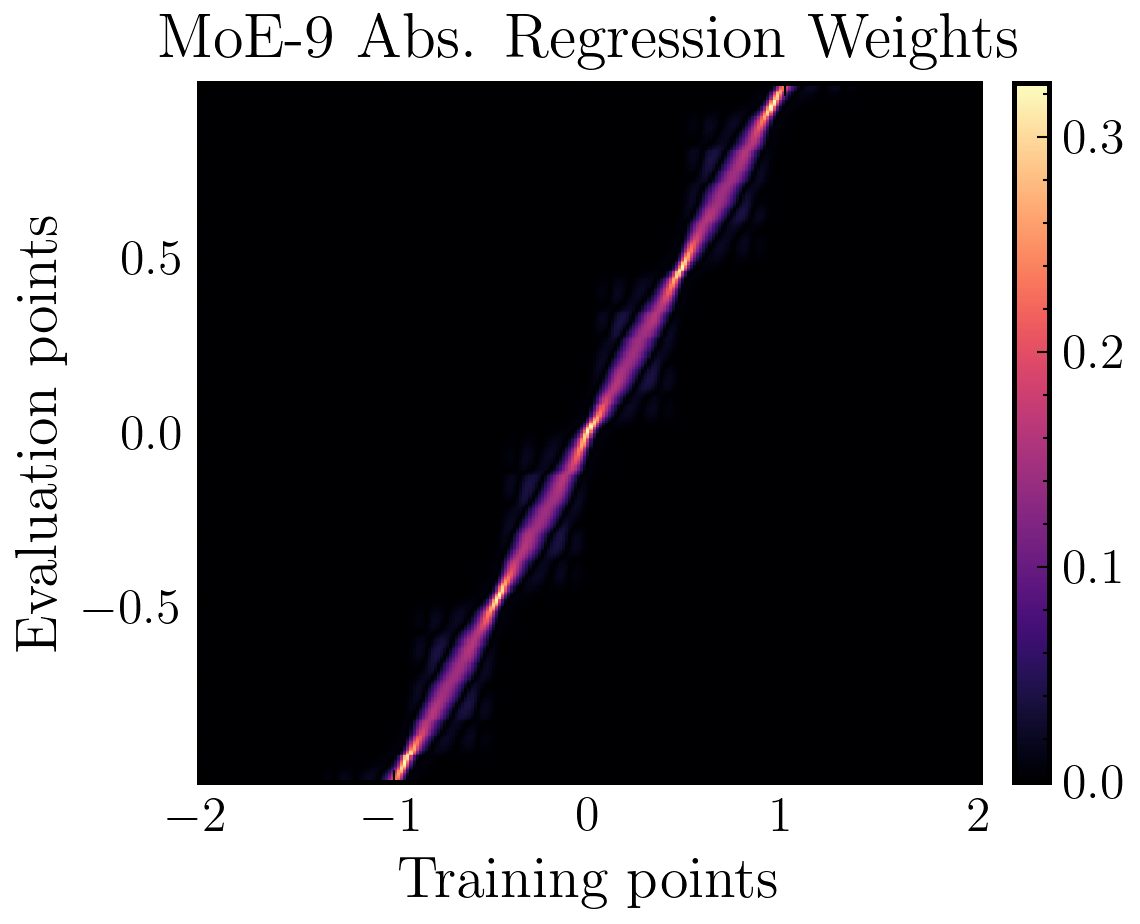}
  \includegraphics[width=0.19\linewidth]{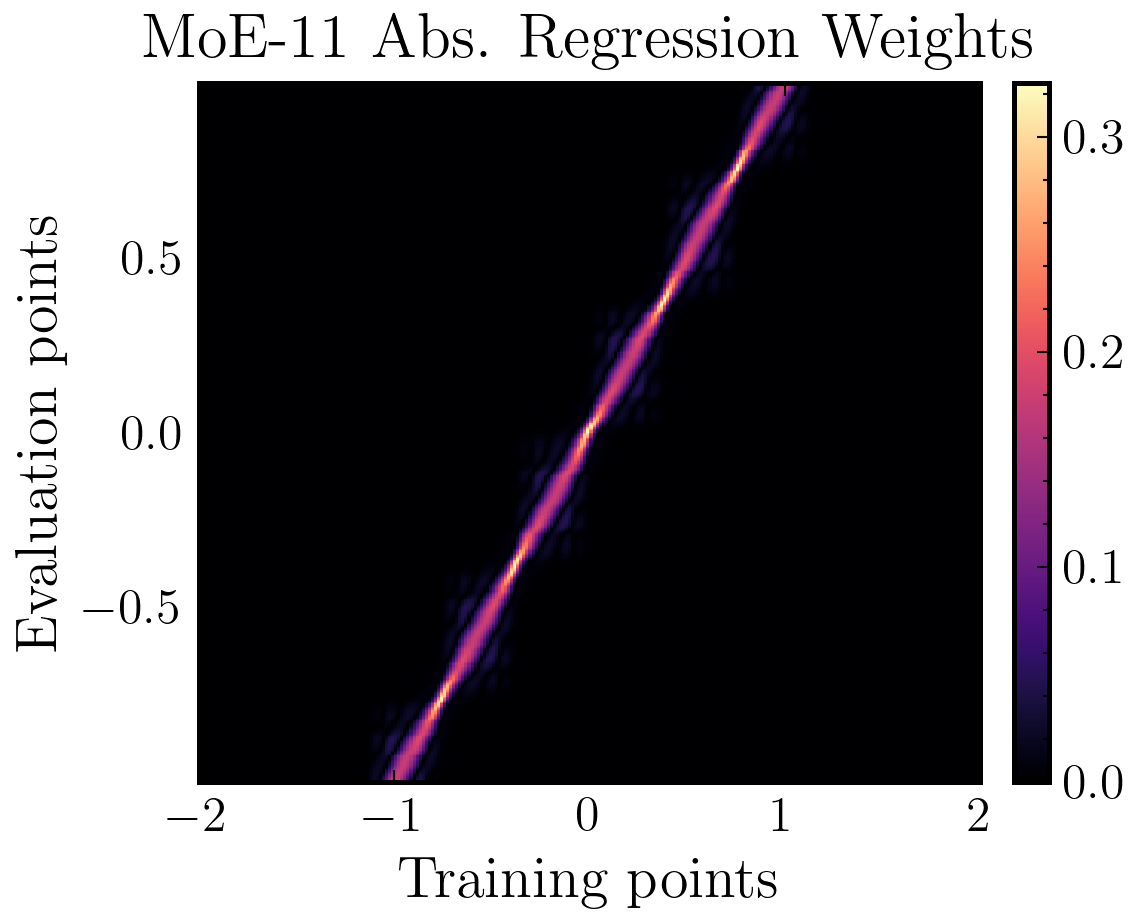}

  {\small MoE NTK, non-centered \par}
  \includegraphics[width=0.19\linewidth]{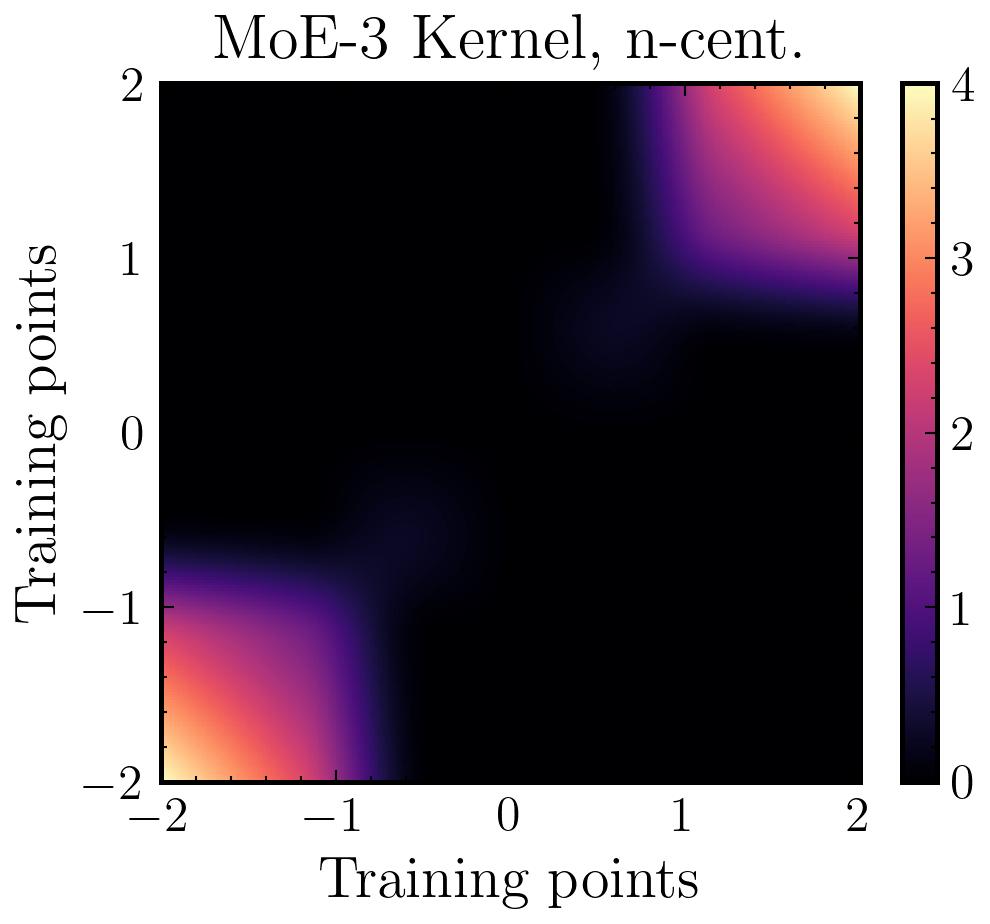}
  \includegraphics[width=0.19\linewidth]{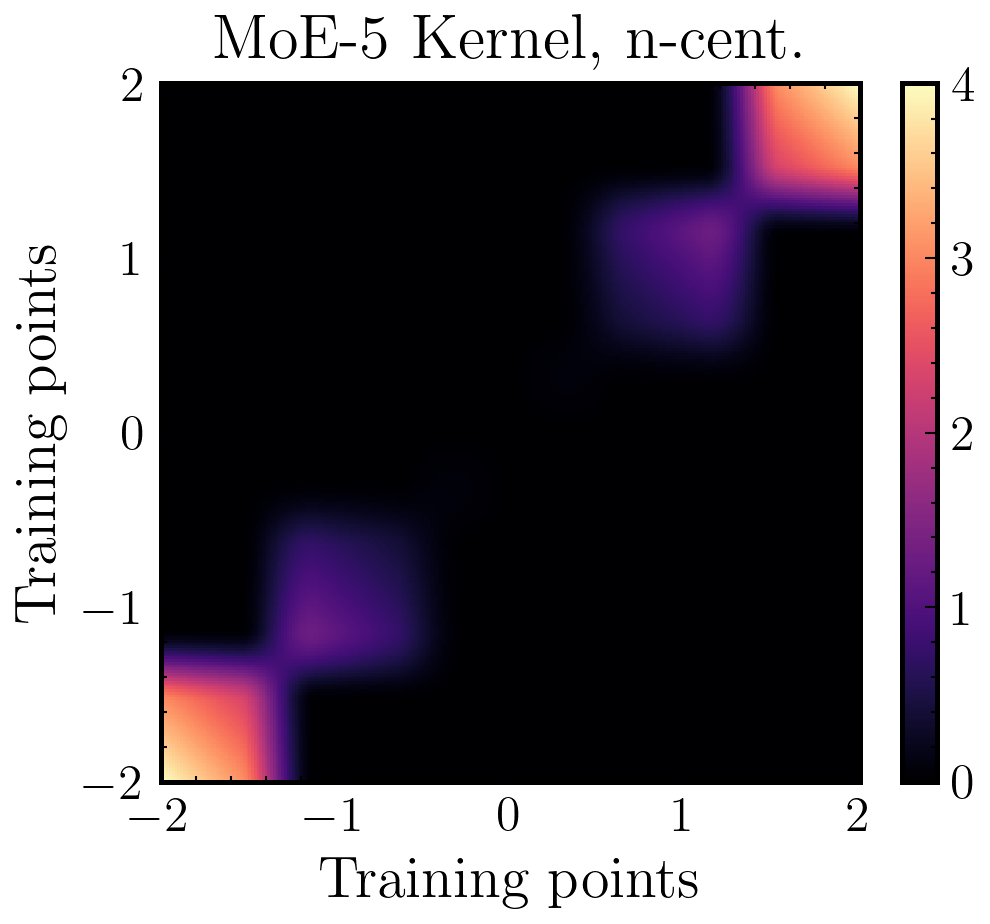}
  \includegraphics[width=0.19\linewidth]{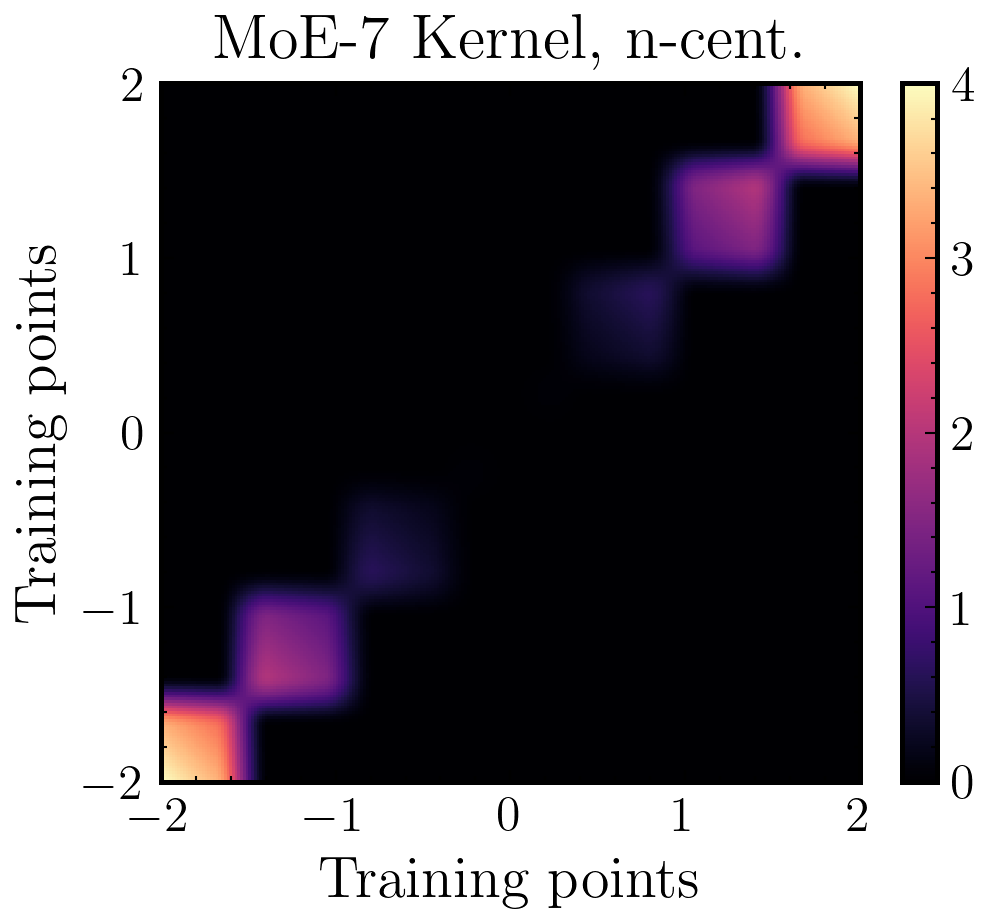}
  \includegraphics[width=0.19\linewidth]{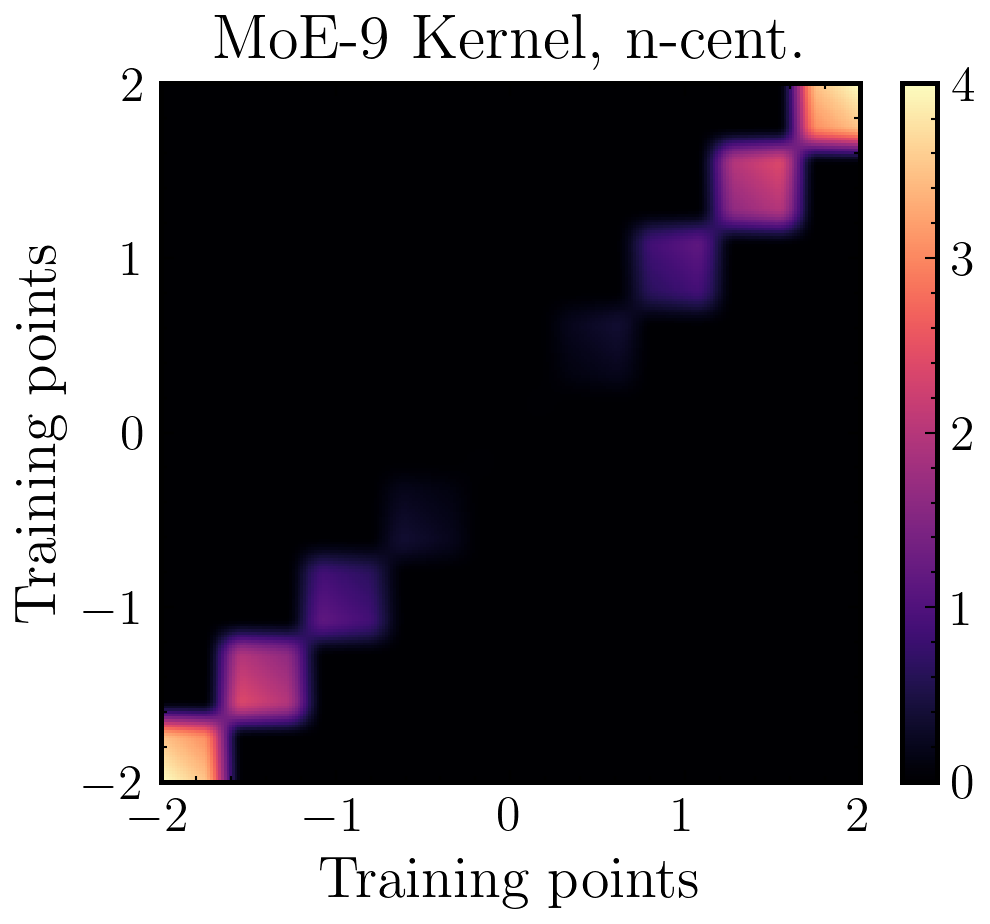}
  \includegraphics[width=0.19\linewidth]{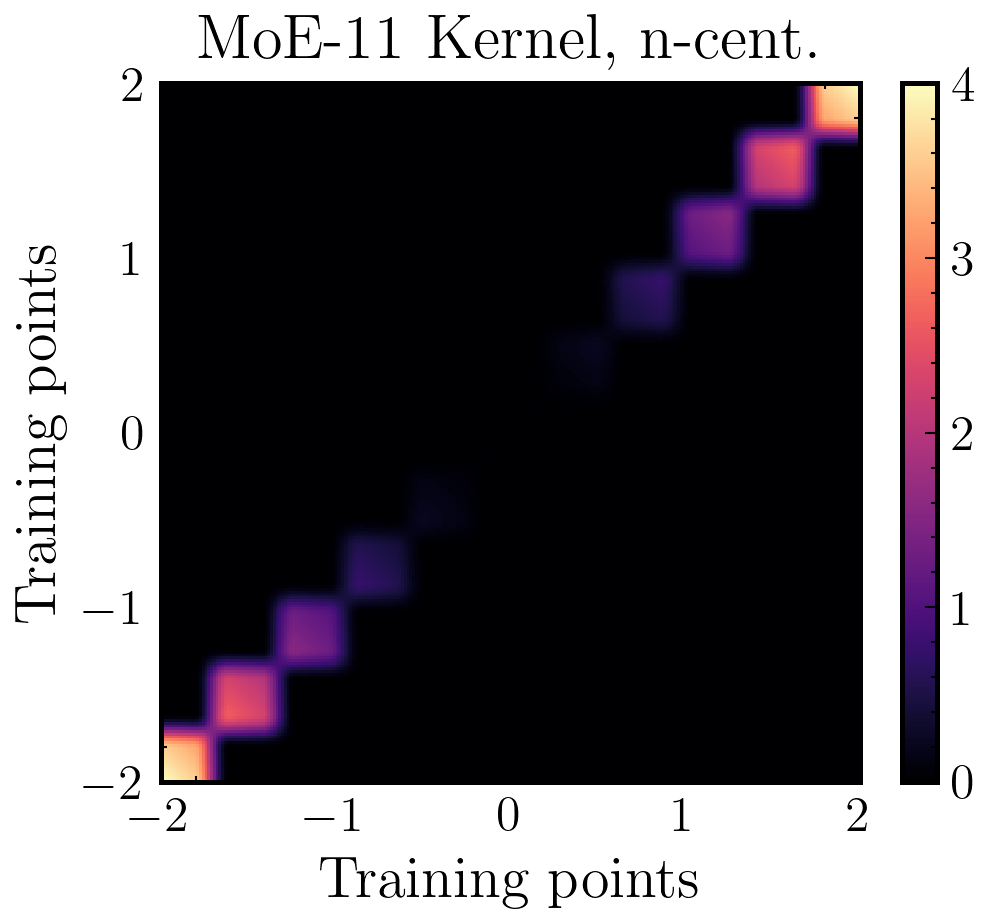}

  {\small MoE NTK kernel-regression absolute weights, non-centered \par}
  \includegraphics[width=0.19\linewidth]{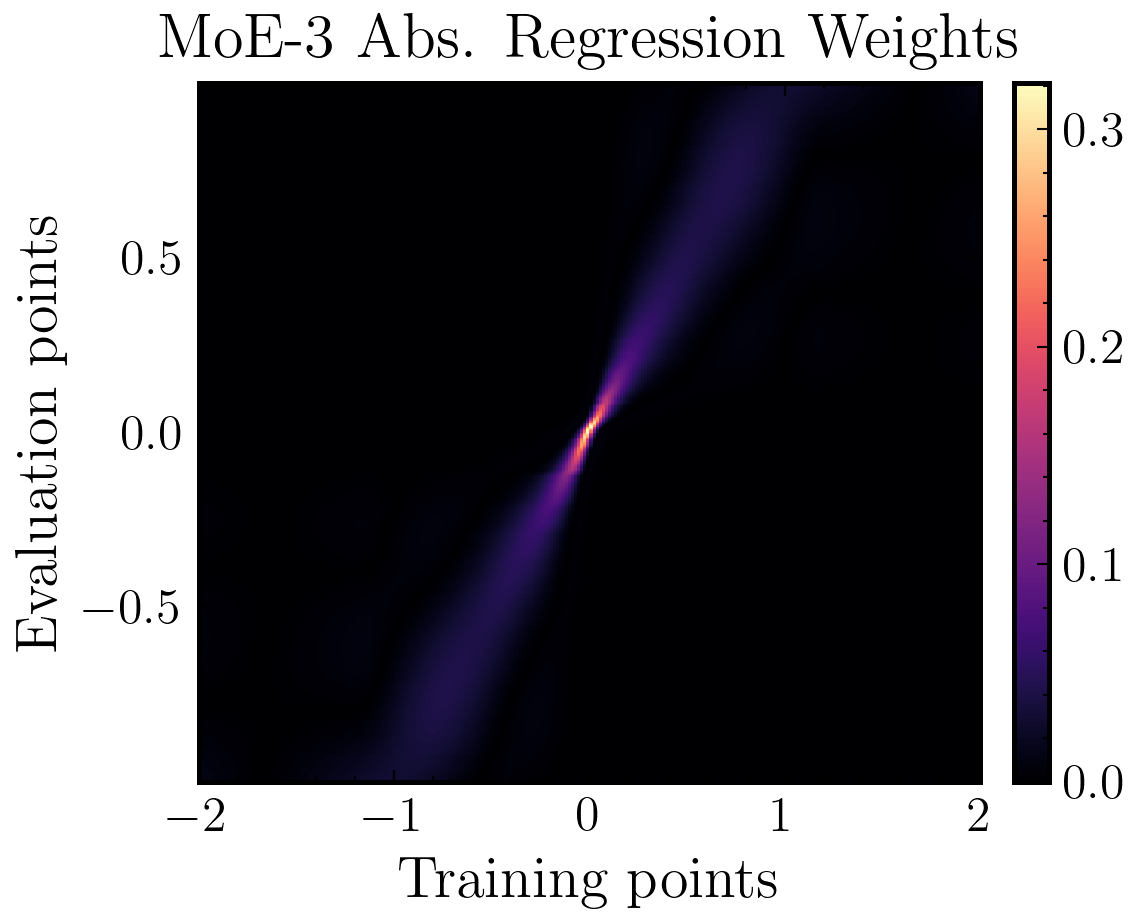}
  \includegraphics[width=0.19\linewidth]{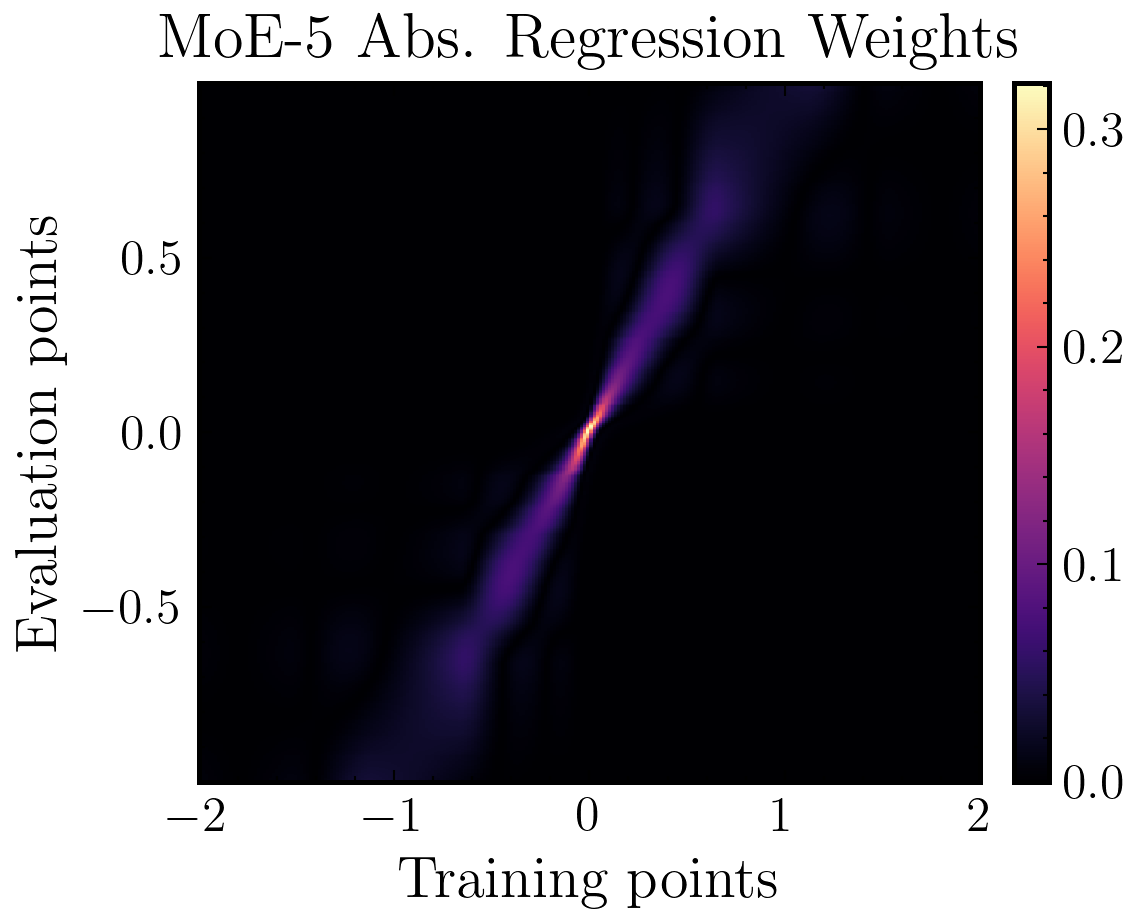}
  \includegraphics[width=0.19\linewidth]{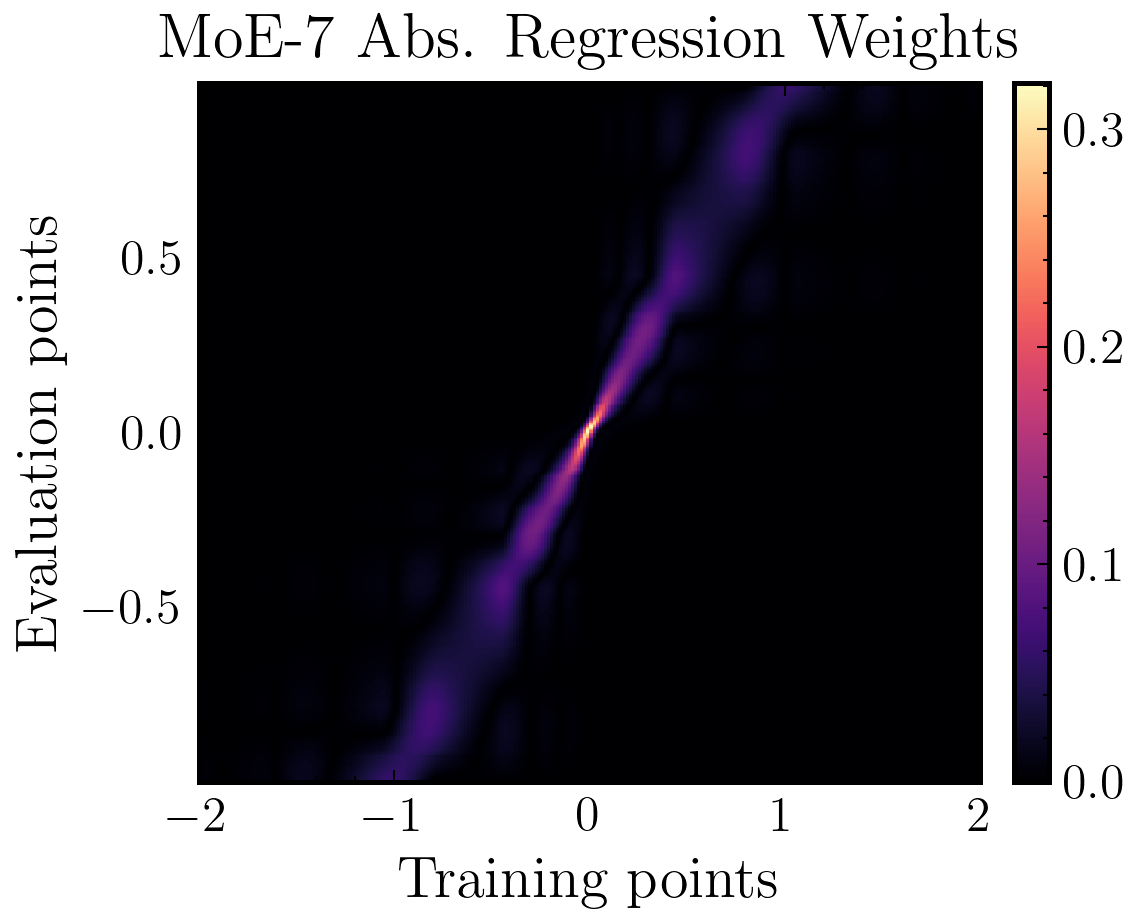}
  \includegraphics[width=0.19\linewidth]{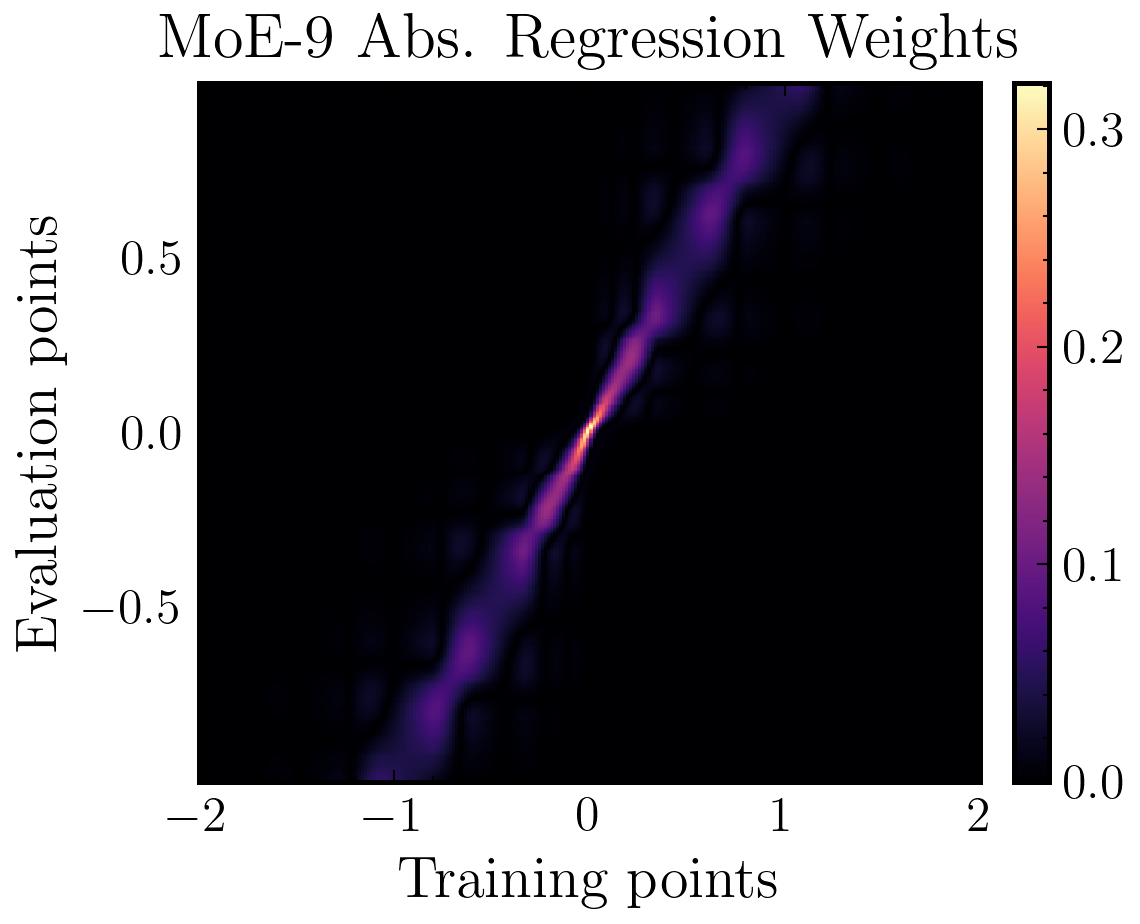}
  \includegraphics[width=0.19\linewidth]{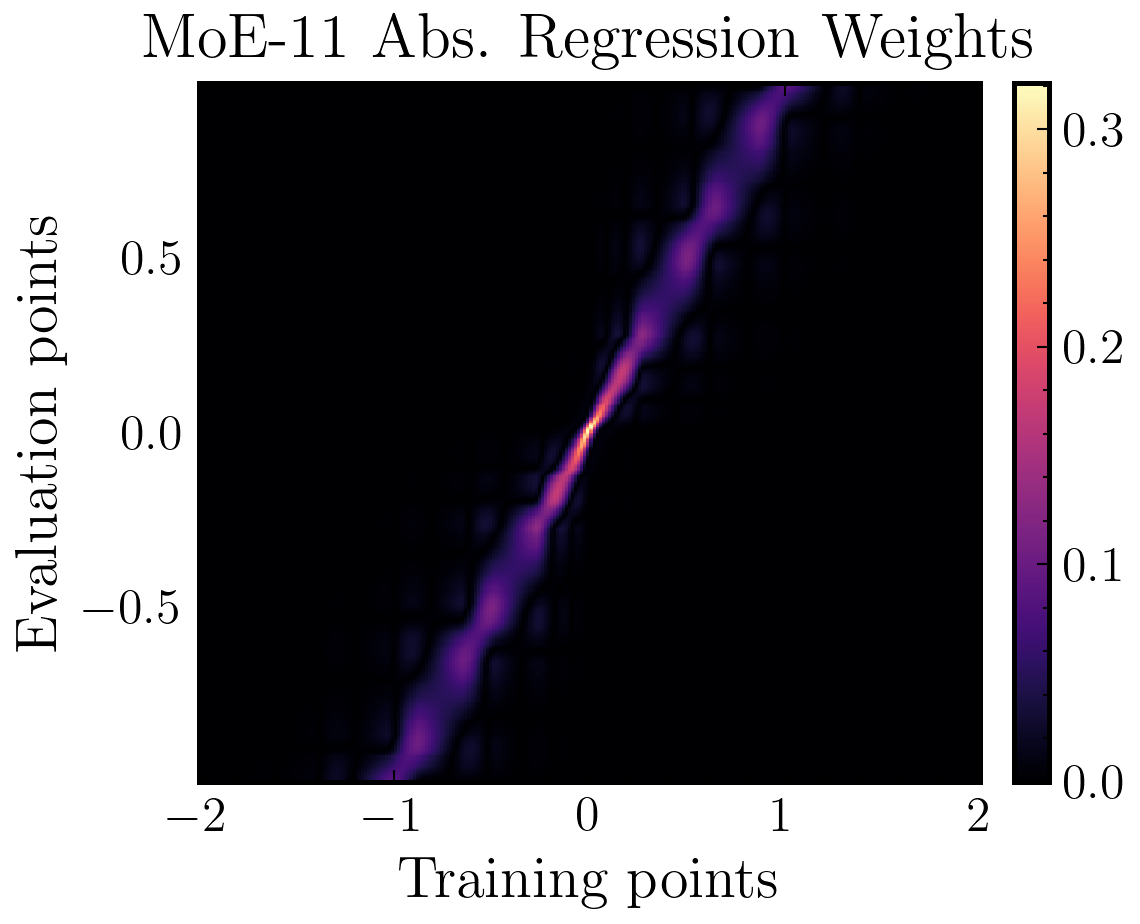}

  \caption{NTK and absolute kernel regression weights for MoE models with various numbers of experts. Rows are (1) centered NTK, (2) centered NTK kernel-regression absolute weights $|H(x,X)|$, (3) non-centered NTK, and (4) non-centered NTK kernel-regression absolute weights $|H(x,X)|$. Centering produces more localized and more regular structures across expert counts.}
  \label{fig:appendix_ntk_matrix_comparison}
\end{figure*}

\clearpage
\section{Architecture Details}\label{appendix:architecture_details}
\subsection{Domain-aware Routers}\label{appendix:da_router}
\begin{figure}[H]
  \centering
  \begin{subfigure}[t]{0.44\linewidth}
    \centering
    \includegraphics[width=\linewidth]{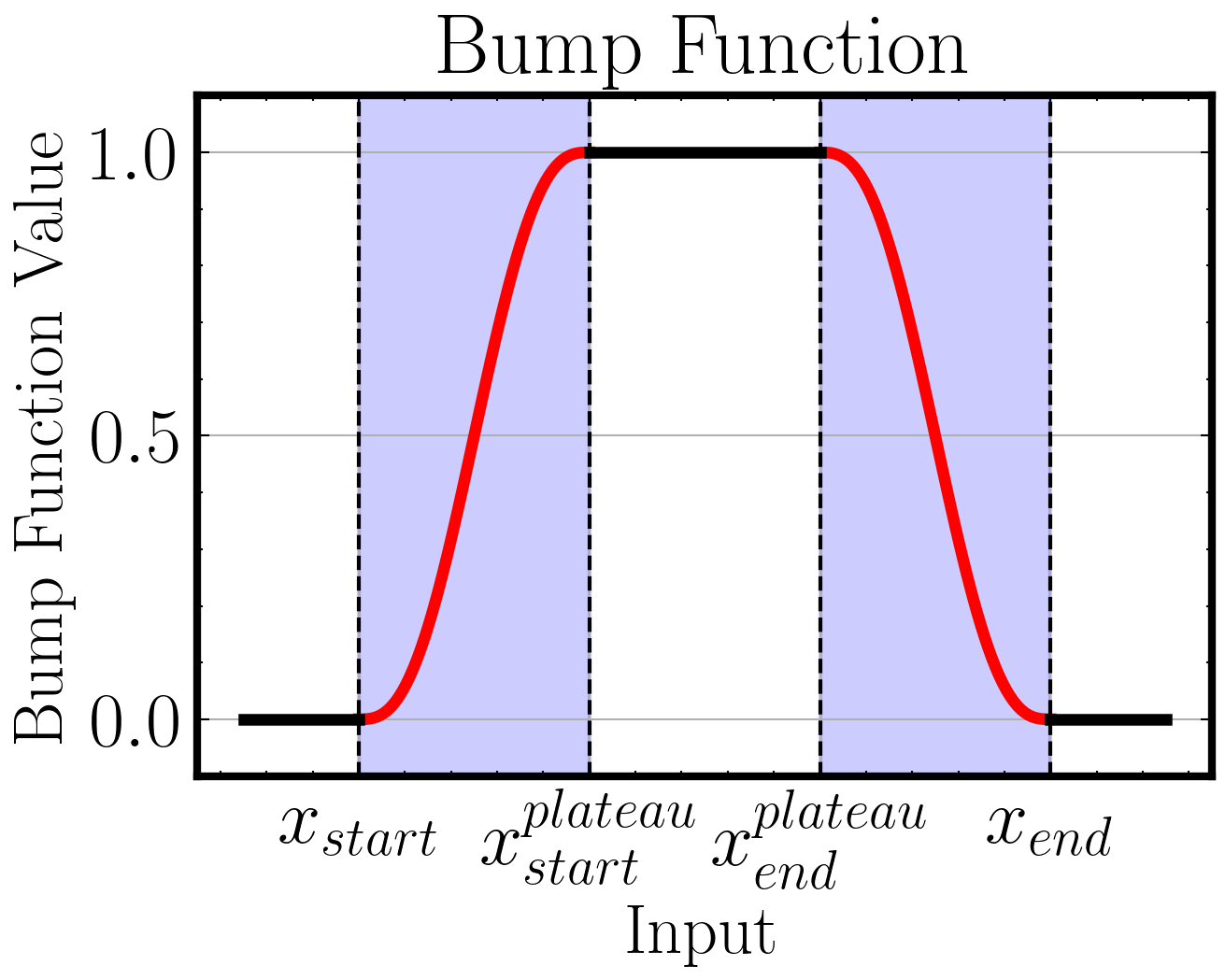}
    \caption{Single bump with two bridges and one plateau.}
  \end{subfigure}
  \begin{subfigure}[t]{0.46\linewidth}
    \centering
    \includegraphics[width=\linewidth]{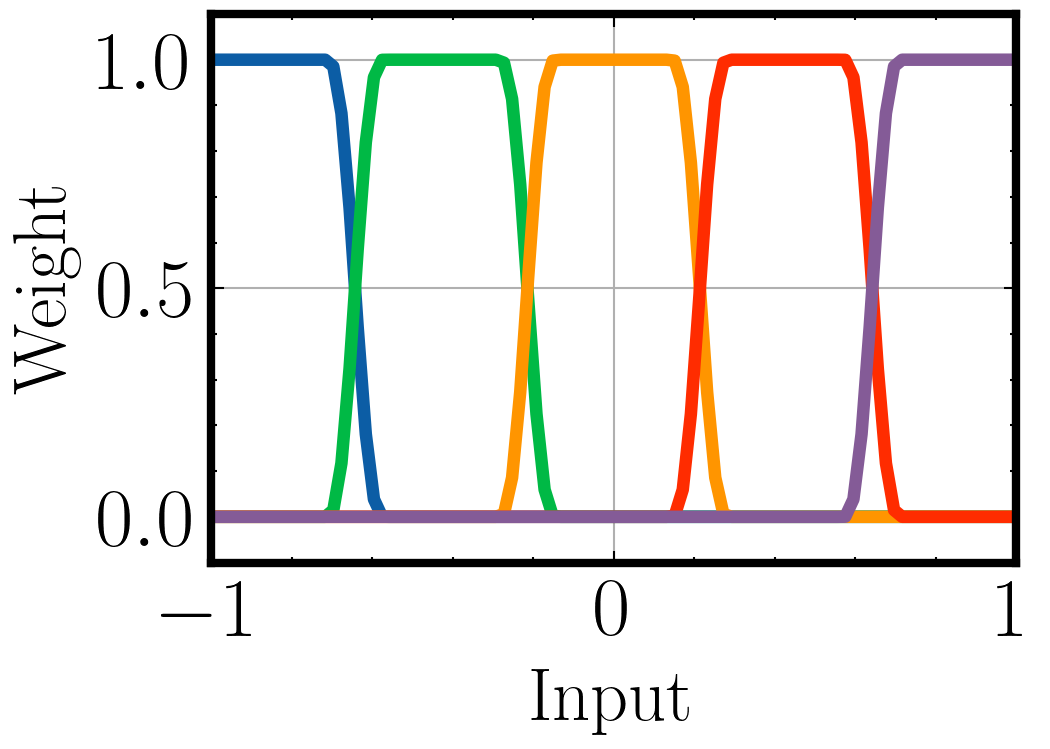}
    \caption{Overlapping bumps across the domain.}
  \end{subfigure}
  \caption{Bump-router construction used for soft domain decomposition. Overlap between neighboring bumps enables smooth blending of experts through partition-of-unity bump weights.}
  \label{fig:appendix_bump_router_softdd}
\end{figure}
Figure~\ref{fig:appendix_bump_router_softdd} shows how the bump function is designed and tiled to softly partition the domain into overlapping subdomains. In this work, we use two types of bump-bridge functions. The first is a $C^1$ bridge, whose first derivative vanishes at both ends.
\begin{equation}
  b(x) = 3 x^2 - 2 x^3, x\in[0,1],
\end{equation} 
The second is a $C^2$ bridge, whose first and second derivatives vanish at both ends. It is used for benchmarks whose differential operator is second order in time.
\begin{equation}
  b(x) = 6 x^5 - 15 x^4 + 10 x^3, x\in[0,1].
\end{equation}

\subsection{Encoder}
We use periodic spatial encoding for both models. That is, given spatial coordinate $x$ and temporal coordinate $t$, the periodic encoding is defined as 
\begin{align}\label{eq:periodic_encoding}
  B(x, t) = [\sigma_x\sin(2\pi x), \sigma_x\cos(2\pi x),\cdots,\sigma_x\sin(2\pi K x), \sigma_x\cos(2\pi K x), \sigma_t t].
\end{align}
And we choose $K=2$ for Advection-Diffusion, Damped-Wave, and Layered-Wave benchmarks, and $K=1$ for standard JAXPI benchmarks. For the standard JAXPI benchmarks the spatial domain is $[-1, 1]$, so the embedding uses $\sin(\pi K x)$ and $\cos(\pi K x)$ instead.
The spatially periodic embedding is then followed by a Gaussian random projection and Fourier embedding. That is,
\begin{align}
  \gamma(x, t) = [\sin(B(x, t)W), \cos(B(x, t)W)].
\end{align}
$\sigma_x$ and $\sigma_t$ are the spatial and temporal random Fourier-feature scales. We set both to $4$, except for the KdV benchmark where both are $8$. The matrix $W$ is learnable during training. A trainable linear layer then maps $\gamma(x, t)$ to the backbone width.

For Burgers equation with Dirichlet boundary condition instead of periodic boundary condition, we remove the periodic embedding in Equation~\eqref{eq:periodic_encoding} and directly apply the random Fourier features as
\begin{align}
  \gamma(x, t) = [\sin(\sigma_x xW_x + \sigma_t tW_t), \cos(\sigma_x xW_x + \sigma_t tW_t)].
\end{align}
\clearpage
\section{Experimental Details}\label{appendix:experiment_details}

We use JAX ecosystem\cite{jax2018github}, Flax\cite{flax2020github} in Python3.11 to implement the experiments in this paper.

All experiments use an unofficial JAX implementation~\cite{Jones2026haydn} of the SOAP optimizer~\cite{vyas2025soapimprovingstabilizingshampoo}. Table~\ref{tab:optimizer_training_schedule} summarizes the training hyperparameters shared across all experiments.

\begin{table}[H]
\centering
\caption{Training hyperparameters shared across all experiments.}
\label{tab:optimizer_training_schedule}
\renewcommand{\arraystretch}{1.15}
\begin{tabular}{ll}
\toprule
Component & Value \\
\midrule
Init LR & $10^{-6}$ \\
Peak LR & $10^{-3}$ \\
Min LR & $10^{-5}$ \\
Warmup steps & $1000$ \\
Decay factor / interval & $0.90$ every $2000$ steps \\
Weight decay & $10^{-5}$ \\
SOAP preconditioner update & Every $2$ steps \\
Residual collocation points & $4096$ \\
Initial-condition points & $1024$ \\
Initial-velocity points (wave benchmarks) & $1024$ \\
Boundary points (Burgers) & $1024$ \\
Activation function & swish \\
Polyak averaging rate~\cite{polyak1992acceleration} & $0.01$ \\
\bottomrule
\end{tabular}
\end{table}

We set the width of the networks to be $128$, and each backbone has $4$ residual blocks, where, for the Latent-MoE, two of them are MoE blocks with chosen number of experts. The width of the corresponding blocks in the non-MoE architecture is widened to keep the total number of parameters comparable, up to some small difference due to the additional bias terms in the MoE blocks. 

In addition to the widened ResNet, we also compare against FB-PINNs and PirateNet. As discussed in the main manuscript, FB-PINNs directly follow the NTK intuition behind MoE localization, whereas PirateNet is considered a strong state-of-the-art architecture across multiple problems.

For all of the benchmarks without analytical solutions, the ground truth is generated by a RK4 or ETDRK4 solver\cite{kassam2005fourth}. We verified the solver against the chebfun package\cite{Driscoll2014}, which generates the ground truth for the standard JAXPI benchmarks, and the two solutions agree to at least single precision.

\paragraph{Damped-Wave and Advection-Diffusion} We configure FB-PINNs with width $100$ across all layers. For the damped-wave and advection-diffusion benchmarks (including their time-variable variants), PirateNet uses four Pirate blocks, each with width $192$. FB-PINNs and Latent-MoE use $6$ temporal experts. The total parameter count for both ResNet and Latent-MoE is around $480{,}000$, while the baselines FB-PINNs and PirateNet are around $550{,}000$; this gap further strengthens the observed performance advantage. Each experiment is trained for $250000$ steps to balance training cost and convergence.

\paragraph{Layered-Wave} To align the actual domain decomposition, we use the 4 experts variant of the FB-PINNs, Latent-MoE, and widened ResNet, as well as the narrower PirateNet with width $160$ to match the parameter counts. Latent-MoE and ResNet have approximately $347{,}000$ parameters, FB-PINNs has approximately $365{,}000$, and PirateNet has approximately $387{,}000$. We follow the Damped-Wave and Advection-Diffusion for train related hyperparameter choice. Additionally, we use wrapping bump function routers to enforce the periodicity on the FB-PINNs and the Latent-MoE models. The global baseline models still use the periodic encoding. See Figure~\ref{fig:lw_router_setup} for the wrapping bump-router weights and expert centering.

\paragraph{Standard JAXPI benchmarks} For standard JAXPI benchmarks, we reduce the number of experts for FB-PINNs and Latent-MoE to $4$, and narrow PirateNet to width $160$ because these benchmarks are simpler. We also reduce training to $150,000$ steps based on loss and error convergence trend. Latent-MoE and ResNet have approximately $347{,}000$ parameters, FB-PINNs has approximately $365{,}000$, and PirateNet has approximately $387{,}000$.

\subsection{Huber loss function}
To prevent diminishing or exploding gradient, we employ the Huber loss that is linear on both small and large error magnitude.
\begin{align}
L(x;\epsilon,E) =
\begin{cases}
2\epsilon |x|, & |x| < \epsilon, \\
|x|^2 + \epsilon^2, & \epsilon \le |x| < E, \\
2 E |x| - E^2 + \epsilon^2, & |x| \ge E.
\end{cases}
\end{align}
We set $\epsilon = 0.01$ and $E = 1000$ in all experiments.

\subsection{Causality Enforcement}
We modify the original causality-enforcement method proposed by {\it Wang et al.}\ \cite{wang2022respecting}, where the causality coefficient $\epsilon$ is fixed or manually scheduled throughout training. The original causality weights for losses $\mathcal L_t$ are formulated as
\begin{equation}\label{eq:unnorm_causal}
  w_t = \exp\left[-\epsilon\sum_{s=0}^{t-1} \mathcal L_s\right],
\end{equation}
where $\epsilon$ has to be carefully tuned to balance the mitigation of error propagation and the vanishing gradient issue. In our experiments, we use relative causality weights, where the weights are calculated with respect to the first time step loss $\mathcal L_0$. That is,
\begin{equation}
  w'_t = \exp\left[-\frac{\epsilon}{T}\sum_{s=0}^{t-1} \mathcal L'_s\right], \quad \mathcal L'_s = \frac{\mathcal L_s}{\mathcal L_0},
\end{equation}
where $T$ is the number of temporal chunks. Furthermore, to avoid vanishing causality weights when the early time losses are not properly optimized, when the last time step weights $w'_{T-1}$ is smaller than a threshold $\tau$, we use the following alternative formulation for the causality weights:
\begin{align}\label{eq:dyn_causal}
  \epsilon' = -\frac{T\log(\tau)}{\sum_{s=0}^{T-2} \mathcal L'_s},
\end{align}
that is, all temporal-chunk losses still receive subtle but meaningful gradients even when early-time losses are not yet optimized, and the causality weights are automatically adjusted to avoid vanishing gradients. In our experiments, we set $\tau=0.001$. The ablation of the relative causal coefficient $\epsilon$ is reported in Appendix~\ref{appendix:ablation}. We use $8$ causal chunks for all experiments, and collocation points are sampled uniformly at random from the spatio-temporal domain for each chunk to compute $\mathcal L_t$.

For standard JAXPI benchmarks (\S\ref{subsec:exp_standard}), we keep the unnormalized causal weighting in Equation~\eqref{eq:unnorm_causal}. In our implementation the cumulative sum in Equation~\eqref{eq:unnorm_causal} is also divided by $T$. For the other benchmarks, we use the dynamic causal weighting in Equation~\eqref{eq:dyn_causal}.
\subsection{Loss balancing}
We adopt the loss balancing method used in \cite{wang2023expert} to balance the PDE residual loss with the initial and boundary condition losses. Let $G\in\mathbb R^{l\times p}$ be the gradient tensor of the losses, where $l$ is the number of loss terms and $p$ is the number of parameters. The weight for each loss term is calculated as 
\begin{align}
  w_i = \frac{\mu_g}{\|G_i\|_2 + \epsilon \mu_g}, \mu_g = \frac{1}{l}\sum_{i=1}^{l} \|G_i\|_2,
\end{align}
where the small added $\epsilon$ term is used to avoid the dominance of the convergent loss term with minimal gradient magnitude. In our experiments, we set $\epsilon=0.0001$. The weights are then normalized to sum to one and smoothed by an exponential moving average with rate $0.01$, updated every $1000$ steps.

\subsection{Gradient Conflict Index}\label{app:gradient_conflict_index}
In the main text, we defined the gradient conflict index on the gradients of the PDE residual loss evaluated on the temporal residual chunks as 
\begin{align}
  G(J) = 1 - \frac{\left|\left|\sum_i J_{i,\cdot} \right|\right|}{\sum_i \left|\left|J_{i,\cdot}\right|\right|}.
\end{align}
$G$ serves as a proxy for Jacobian correlation, which is central to NTK theory. When the number of points sampled for the $i$-th chunk $c_i$ approaches infinity, the Jacobian $J_{i,\cdot}$ converges to the expectation
\begin{align}
  J_{i,\cdot}\to\mathbb{E}_{x\in c_i}[r(x)\nabla_\theta f(x)],
\end{align}
where $r(x)$ is the PDE residual at point $x$. Then
\begin{align}
  \langle J_{i,\cdot}, J_{j,\cdot} \rangle &\to \mathbb{E}_{x\in c_i, y\in c_j}[\langle r(x)\nabla_\theta f(x), r(y)\nabla_\theta f(y) \rangle],\\
  &=\mathbb{E}_{x\in c_i, y\in c_j}[r(x)K(x,y)r(y)],\quad \text{by the linearity of the expectation}.
\end{align}
Then the numerator of the gradient-conflict index can be asymptotically viewed as
\begin{align}
  \left\|\sum_i J_{i,\cdot}\right\| &\to \sqrt{\sum_{i,j} \mathbb{E}_{x\in c_i, y\in c_j}[r(x)K(x,y)r(y)]},
\end{align}
and the denominator can be asymptotically viewed as
\begin{align}
  \sum_i \|J_{i,\cdot}\| &\to \sum_i \sqrt{\mathbb{E}_{x\in c_i, y\in c_i}[r(x)K(x,y)r(y)]}.
\end{align}
Therefore, the gradient-conflict index can be viewed as a function of the NTK $K$
\begin{align}
  G(K) \to 1 - \frac{\sqrt{\sum_{i,j} \mathbb{E}_{x\in c_i, y\in c_j}[r(x)K(x,y)r(y)]}}{\sum_i \sqrt{\mathbb{E}_{x\in c_i, y\in c_i}[r(x)K(x,y)r(y)]}},
\end{align}
when the network width and the number of data points for each chunk both approach infinity.\\
Although these asymptotic conditions are difficult to attain exactly, the gradient-conflict index is still a practical proxy to quantify the negative correlation and tracks real training dynamics and loss evolution.

\subsection{Advection-Diffusion}

Figure \ref{fig:adv_coefficients_appendix} shows the time-variable coefficient profiles used in the advection-diffusion benchmark. The advection speed is $c(t) = 6 + 1.5\sin(2\pi t) + 3\sin(6\pi t + \pi/5) + 0.4\cos(10\pi t) + 0.2t$, while the diffusion coefficient $\kappa(t)$ is a piecewise function that has different constant values in three stages, with smooth transitions between stages. The transitions are $\tanh$ ramps of the form $\tfrac{1}{2}(1 + \tanh((t - t_0)/w))$ with widths $w = 0.14$ and $w = 0.06$, for the 2 transitions.

\begin{figure}[t]
  \centering
  \begin{subfigure}[t]{0.4\linewidth}
    \centering
    \includegraphics[width=\linewidth]{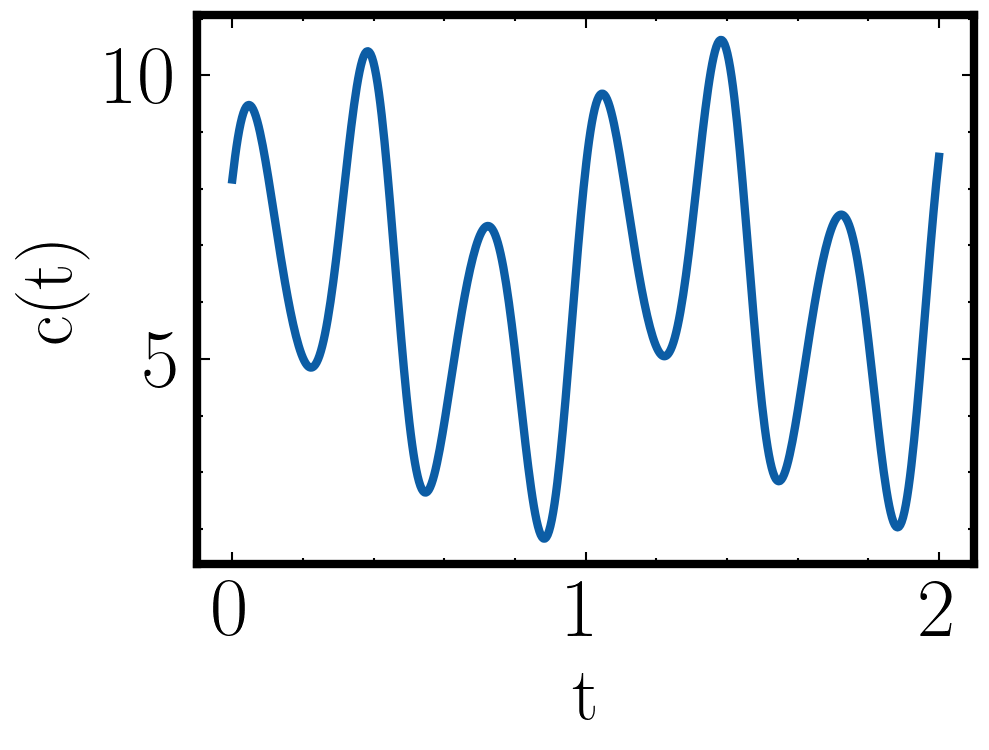}
    \caption{Time-variable advection speed $c(t)$.}
  \end{subfigure}
  \begin{subfigure}[t]{0.4\linewidth}
    \centering
    \includegraphics[width=\linewidth]{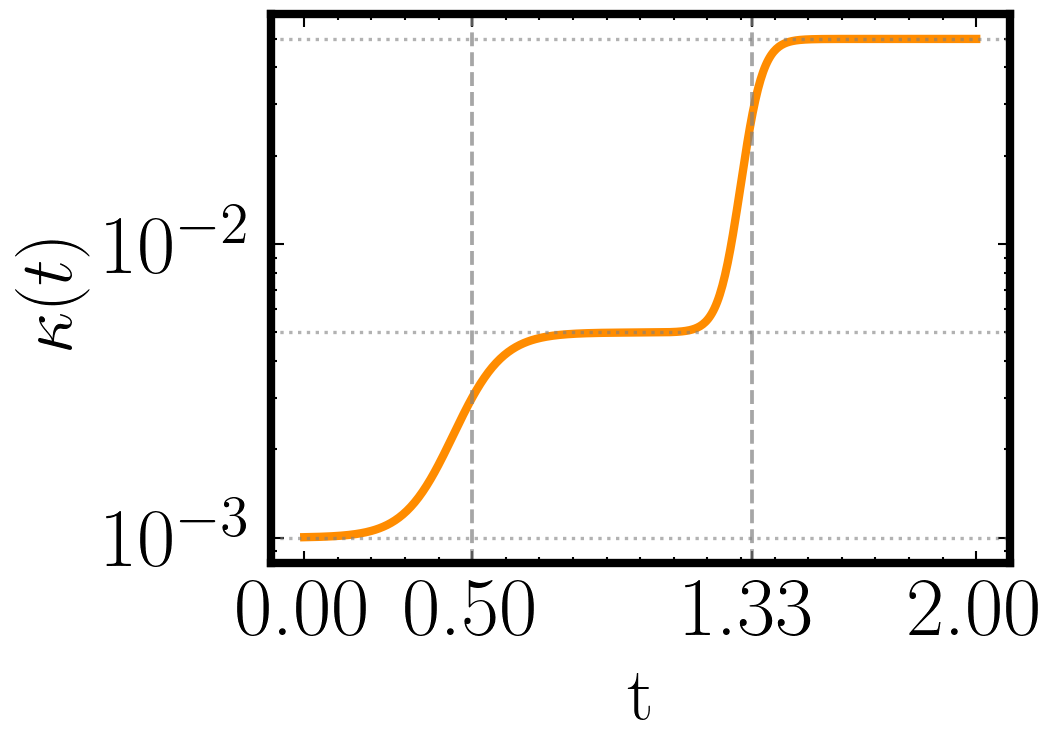}
    \caption{Time-variable diffusion coefficient $\kappa(t)$.}
  \end{subfigure}
  \caption{Coefficient profiles used in the advection-diffusion benchmark. The diffusion plot uses log scale, and dashed lines indicate stage boundaries at $t=0.5$ and $t=4/3$.}
  \label{fig:adv_coefficients_appendix}
\end{figure}

\subsection{Damped Wave Equation with forcing}

 Figure~\ref{fig:damped_wave_force_injection_appendix} shows the force-injection profile used in the damped-wave benchmark, where forcing is injected in the middle stage of the time domain, following the damping stage, and the system evolves freely in the first stage before the damping and the last stage after the force injection. The forcing term ramps up linearly over the first third of the force-injection stage and then stays constant until that stage ends.

\begin{figure}[t]
  \centering
  \includegraphics[width=0.50\linewidth]{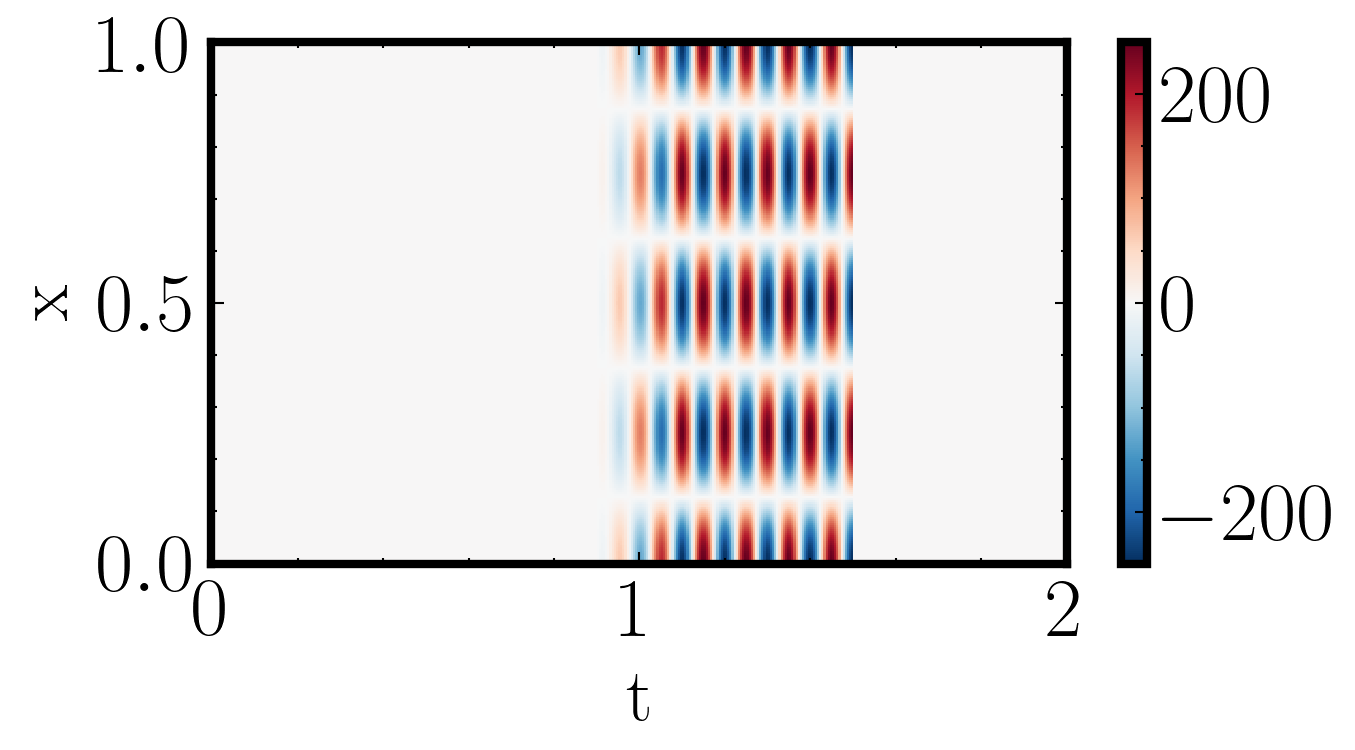}
  \caption{Force-injection profile used in the damped-wave benchmark. Dashed lines mark the boundaries between stages.}
  \label{fig:damped_wave_force_injection_appendix}
\end{figure}

\subsection{Layered Wave Equation}\label{appendix:lw_details}

The layered-wave benchmark solves the pure second-order wave equation
\begin{equation}
  u_{tt} = c(x)^2\, u_{xx}, \qquad x \in [0, 1),\; t \in [0, 1],
\end{equation}
with periodic spatial boundary conditions. The wave speed is a periodic, piecewise-smooth function defined by four spatial layers with speeds $c_1 = 2.0$, $c_2 = 4.0$, $c_3 = 0.5$, $c_4 = 3.0$. The initial displacement $u(x,0)$ is a sine wave whose local spatial frequency varies with the layer structure: in each layer, the frequency is set proportionally to the spatial Fourier modes $\{6, 4, 5, 4\}$, producing a $C^2$-smooth initial condition by construction. The initial velocity is zero, $u_t(x,0) = 0$. Figure~\ref{fig:lw_benchmark_setup} shows the wave speed profile and initial condition.
\paragraph{Initial condition} We use a cosine ramp to connect different local frequencies of adjacent layers. Given an interval $[a,b]$ where the frequency ramps up from $A_l$ to $A_r$, we construct the effective frequency
\begin{align}
    A_{\text{eff}}(x) = \frac{A_r + A_l}{2} - \frac{A_r - A_l}{2}\cos\left(\frac{\pi(x-a)}{b-a}\right)
\end{align}
Let $\varphi(x) = \int_0^x A_{\text{eff}}(s) ds$, $\varphi_L = \varphi(1), N=\operatorname{round}(\varphi_L)$, the initial condition is
\begin{align}
    u(x,0) = \sin\left(\frac{2\pi N}{\varphi_L}\varphi(x)\right),
\end{align}
where the scaling $N$ ensures the periodicity of the initial condition. For this benchmark we decompose the \emph{spatial} axis into four expert subdomains. The bump-router boundaries are tiled uniformly on $[0,1)$ with the same plateau ratio (the plateau length divided by the bridge length) used to define the physical layers; the spatial shift $\delta$ in Table~\ref{tab:lw_ablation} rigidly offsets all router boundaries by $\delta$ to probe sensitivity to misalignment. Figure~\ref{fig:lw_router_setup} shows the corresponding bump-router weights and expert centering.

\begin{figure}[H]
  \centering
  \begin{subfigure}[b]{0.48\textwidth}
    \centering
    \includegraphics[width=\linewidth]{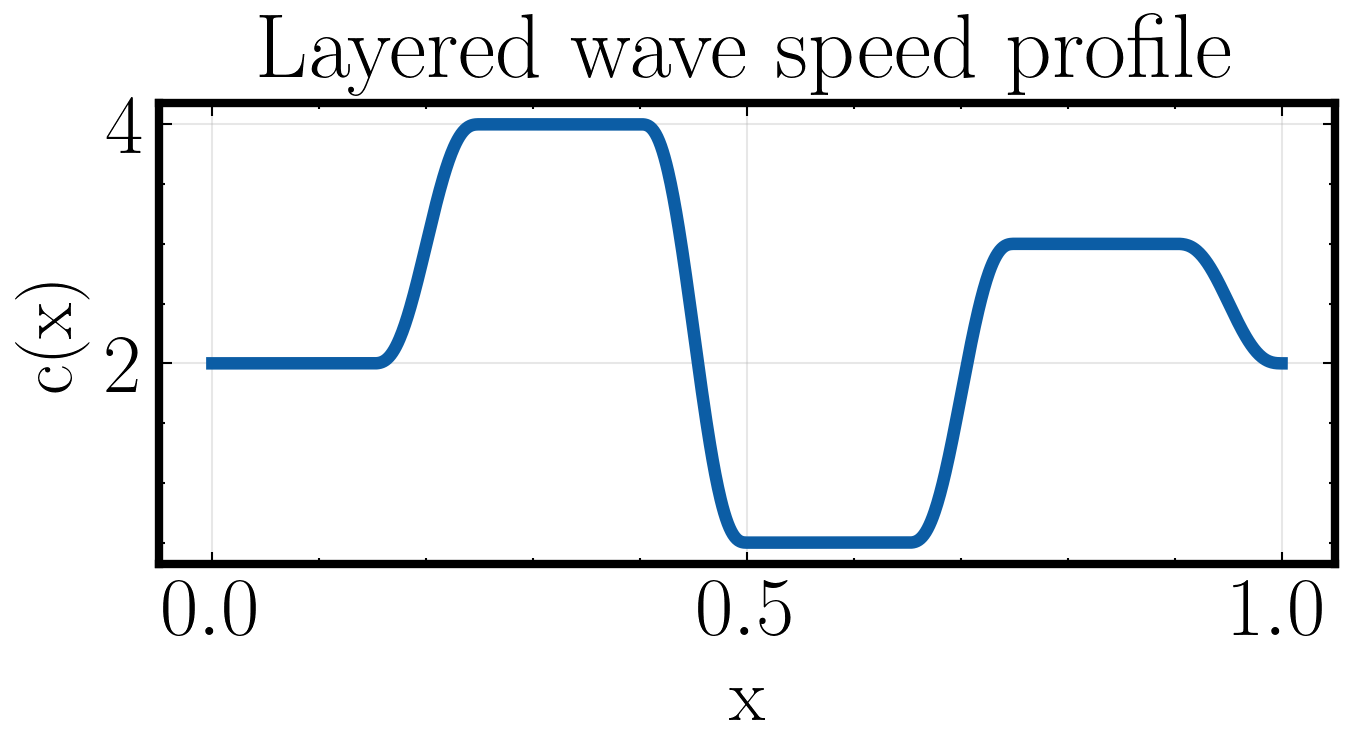}
    \caption{}
  \end{subfigure}
  \begin{subfigure}[b]{0.40\textwidth}
    \centering
    \includegraphics[width=\linewidth]{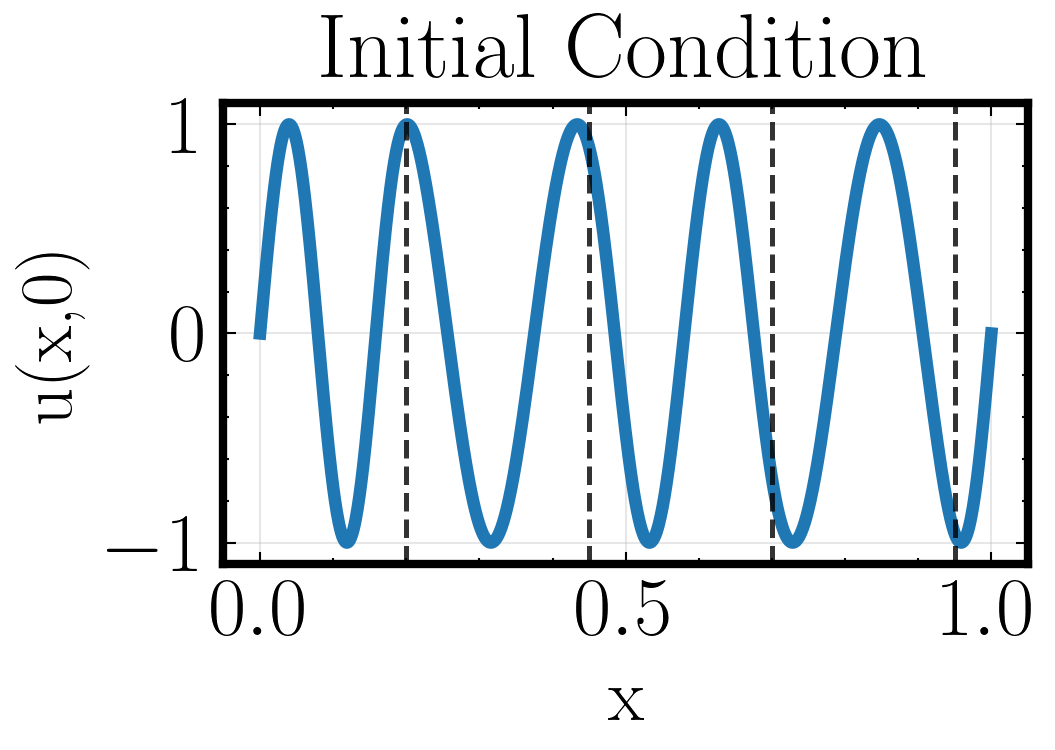}
    \caption{}
  \end{subfigure}
  \caption{Layered-wave benchmark: (a) wave speed and (b) initial condition.}
  \label{fig:lw_benchmark_setup}
\end{figure}

\begin{figure}[H]
  \centering
  \begin{subfigure}[b]{0.55\textwidth}
    \centering
    \includegraphics[width=\linewidth]{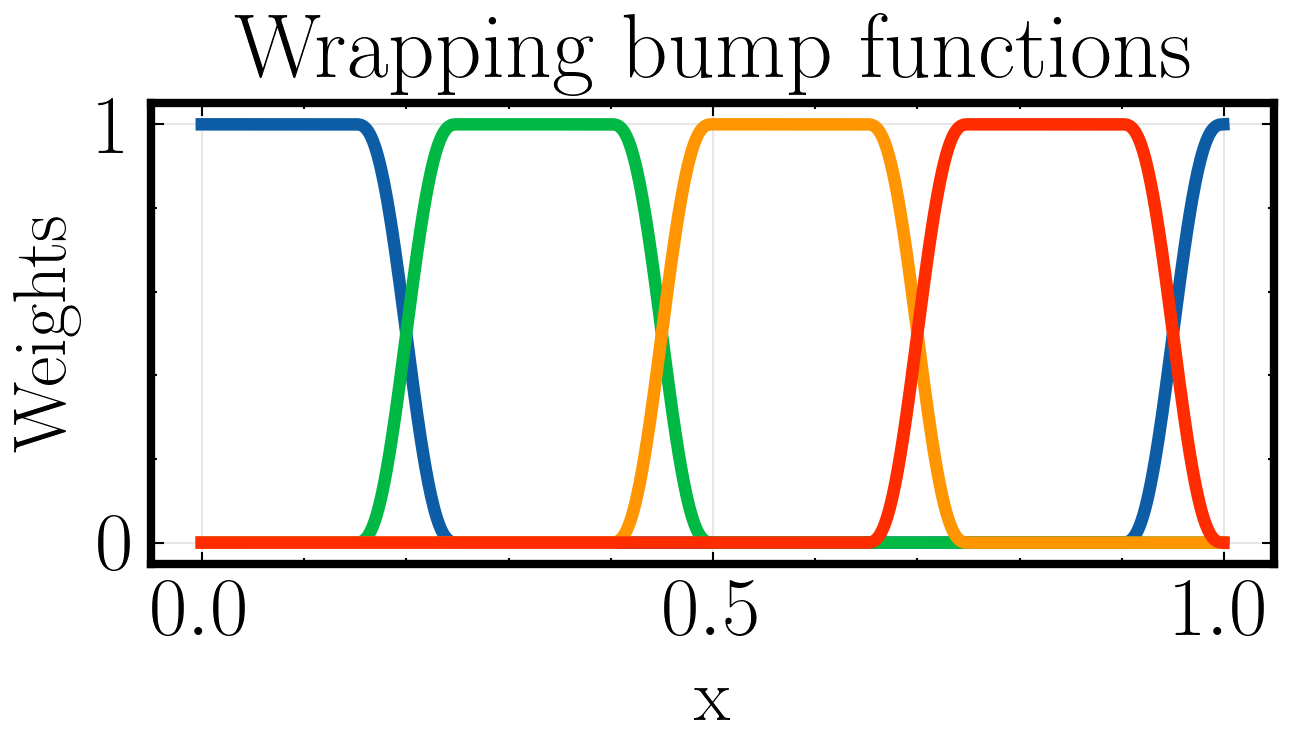}
    \caption{}
  \end{subfigure}
  \hfill
  \begin{subfigure}[b]{0.35\textwidth}
    \centering
    \includegraphics[width=\linewidth]{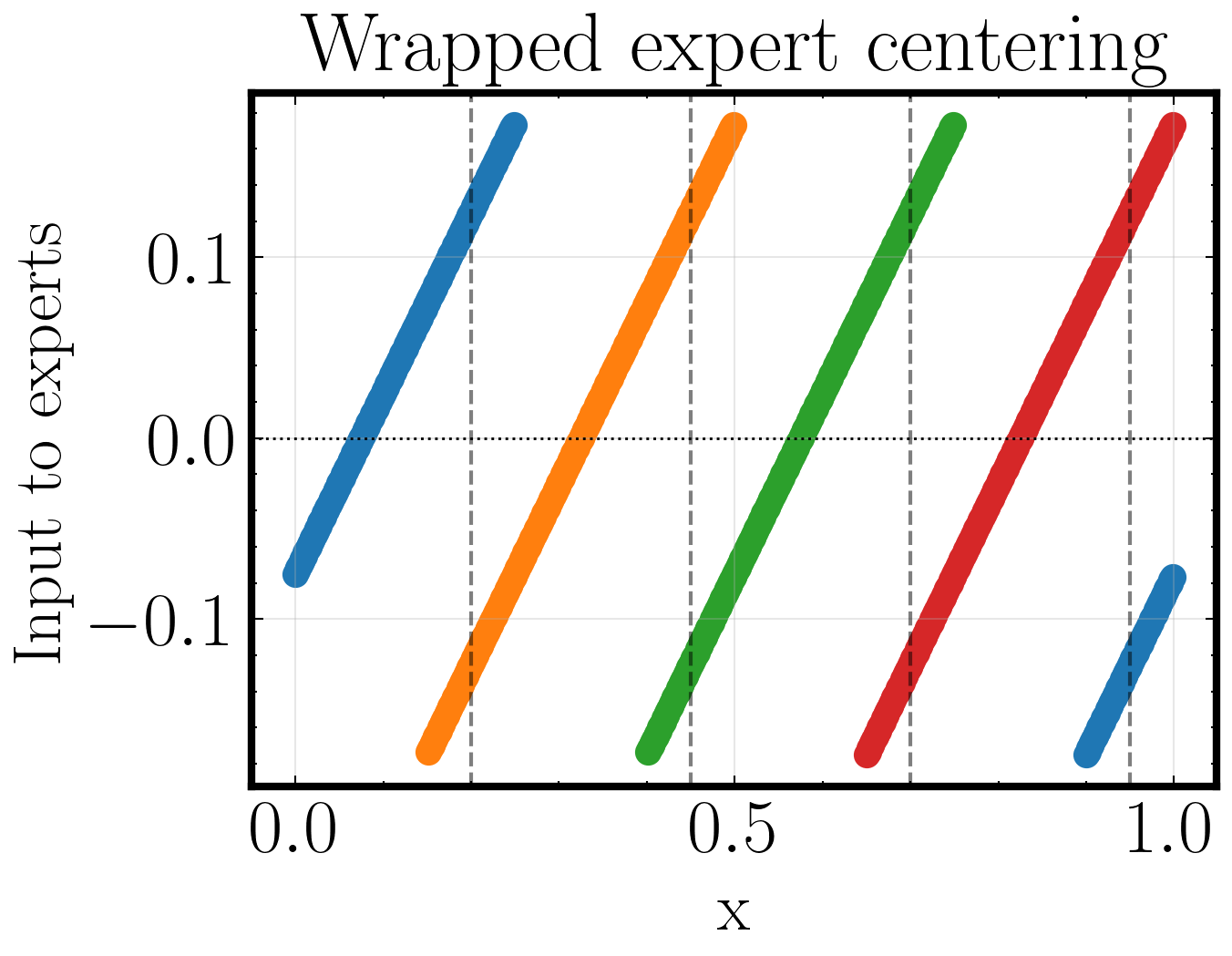}
    \vspace{1mm} 
    \caption{}
  \end{subfigure}
  \caption{{\small Wrapped domain decomposition to preserve periodicity for the Layered-Wave benchmark. (a) Wrapping bump functions: To enforce periodic boundary conditions, the first bump function (blue) wraps around the domain limits, smoothly connecting the left and right boundaries. (b) Wrapped expert centering: The local coordinate mapping for the experts. The spatial input for the first expert is similarly wrapped, ensuring continuous and consistent centering across the periodic boundary.}}
  \label{fig:lw_router_setup}
\end{figure}

\subsection{Baseline model visualizations}
We provide additional error visualizations for the FB-PINNs and PirateNet baselines in Figure \ref{fig:appendix_baseline_error_visualizations}, where the median prediction error across random seeds is visualized for each model and benchmark. The ground truth visualizations for all three benchmarks are shown in Figure \ref{fig:appendix_ground_truth_visualizations}.
\begin{figure}[H]
  \centering
  \begin{subfigure}[t]{0.34\textwidth}
    \centering
    \includegraphics[width=\linewidth]{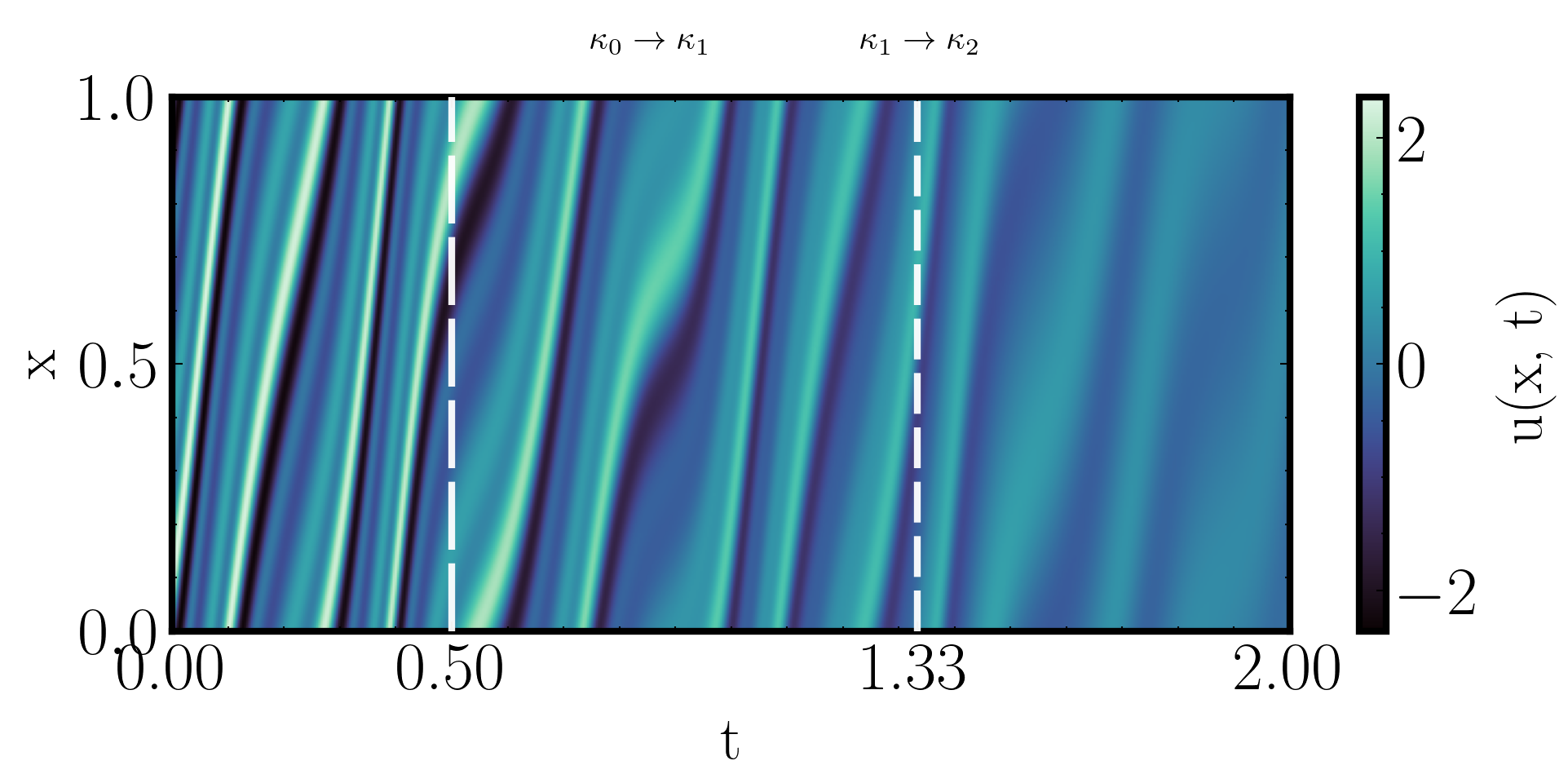}
    \caption{}
  \end{subfigure}
  \begin{subfigure}[t]{0.34\textwidth}
    \centering
    \includegraphics[width=\linewidth]{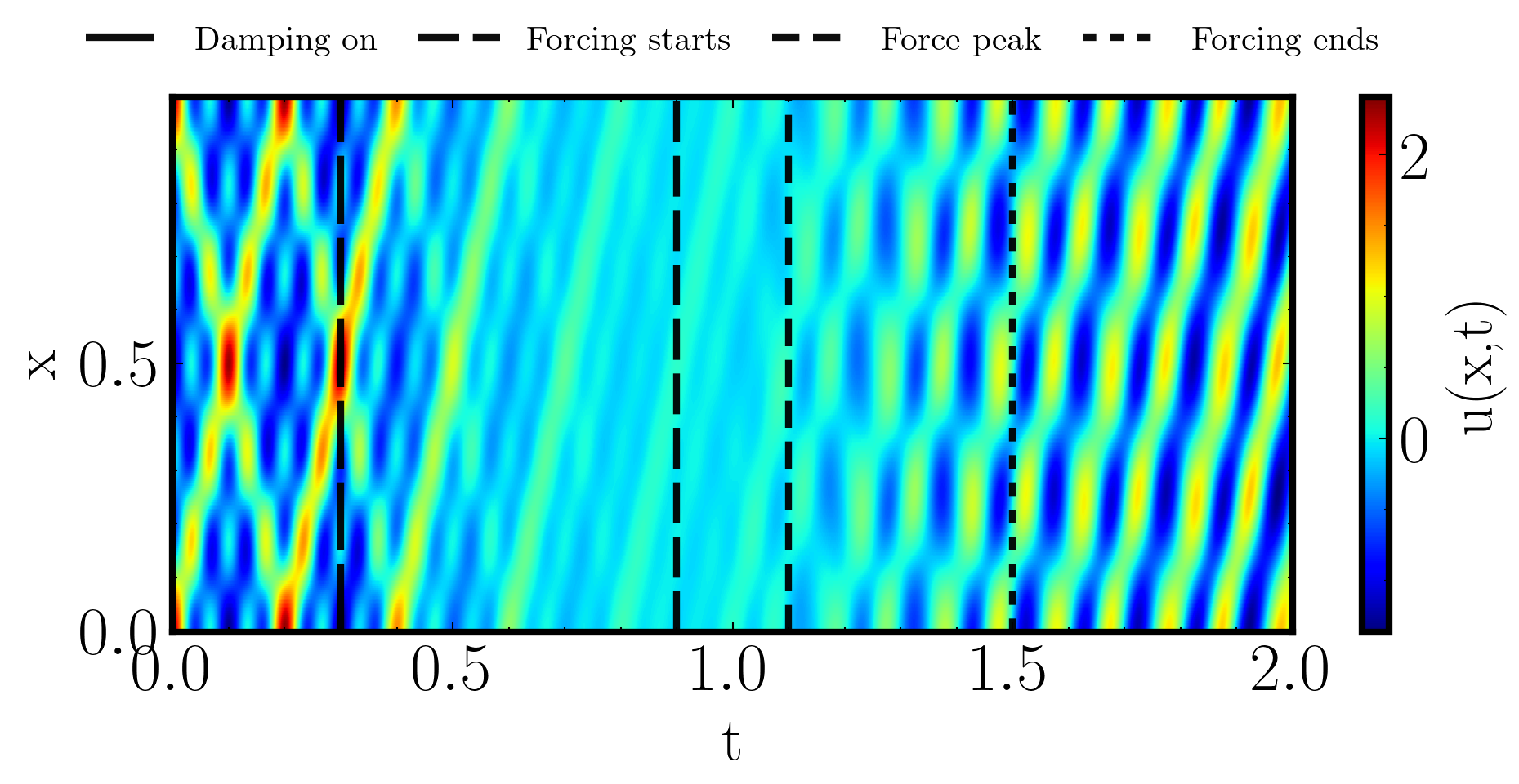}
    \caption{}
  \end{subfigure}
  \begin{subfigure}[t]{0.20\textwidth}
    \centering
    \includegraphics[width=\linewidth]{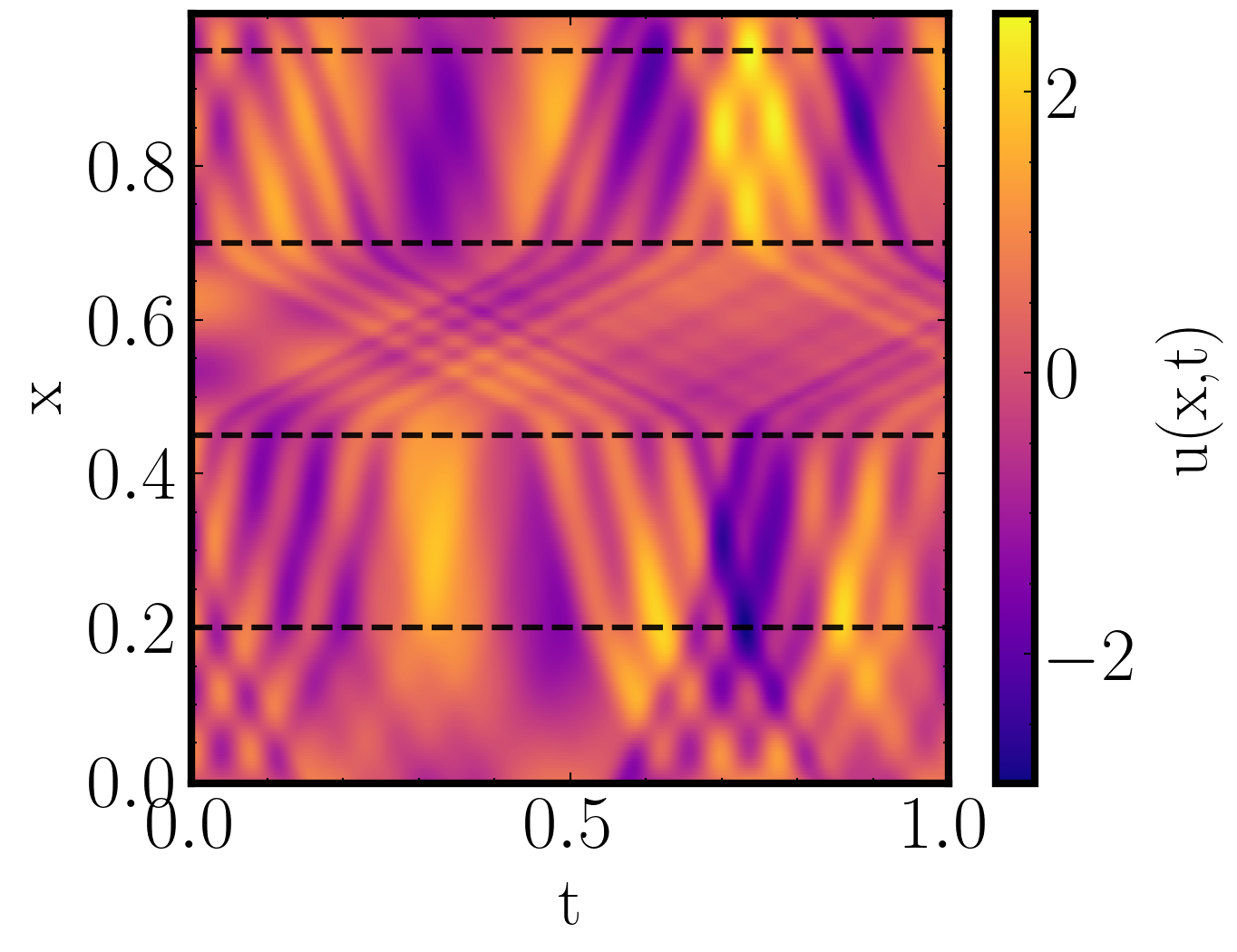}
    \caption{}
  \end{subfigure}
  
  \caption{Ground-truth visualizations for the (a) advection-diffusion (AD), (b) damped-wave (DW), and (c) layered-wave (LW) benchmarks.}
  \label{fig:appendix_ground_truth_visualizations}
\end{figure}
\begin{figure}[h]
  \centering
  \begin{subfigure}[t]{0.48\textwidth}
    \centering
    \includegraphics[width=\linewidth]{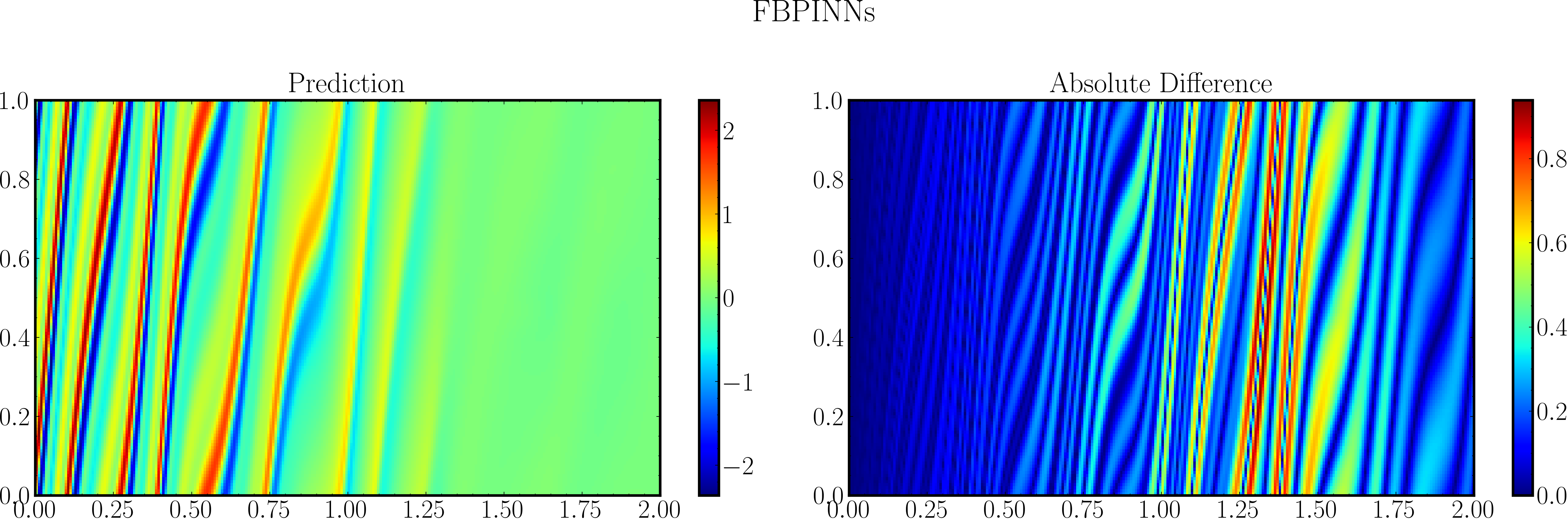}
    \caption{AD: FB-PINNs.}
  \end{subfigure}\hfill
  \begin{subfigure}[t]{0.48\textwidth}
    \centering
    \includegraphics[width=\linewidth]{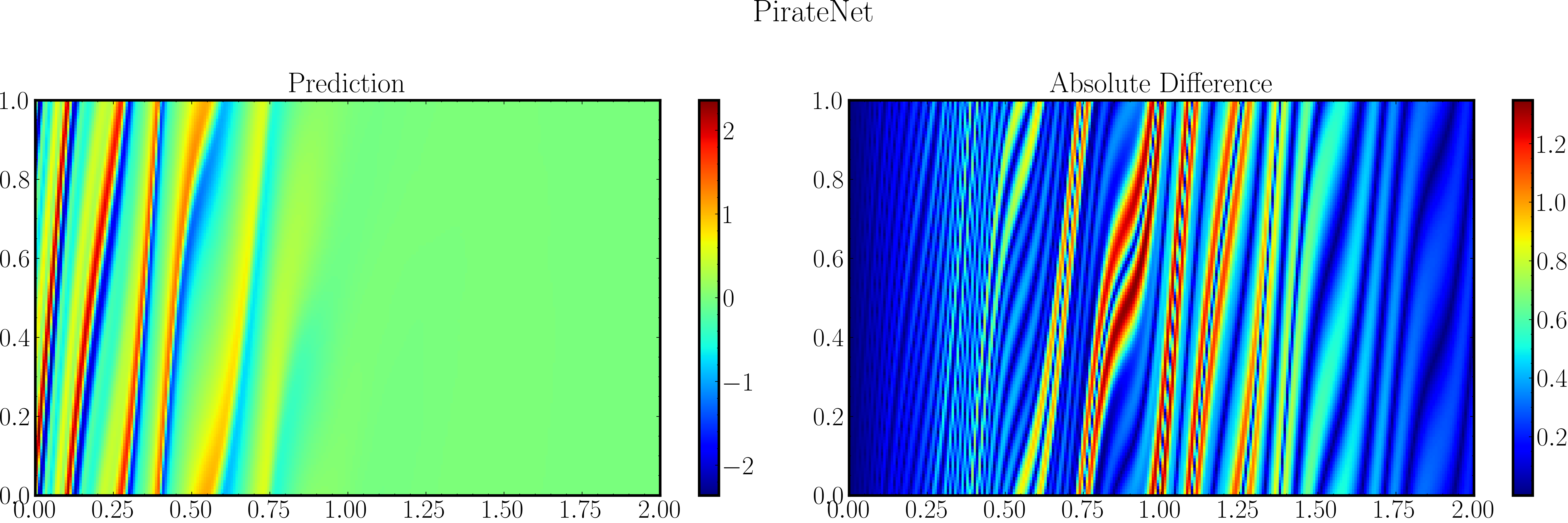}
    \caption{AD: PirateNet.}
  \end{subfigure}

  \vspace{0.6em}

  \begin{subfigure}[t]{0.48\textwidth}
    \centering
    \includegraphics[width=\linewidth]{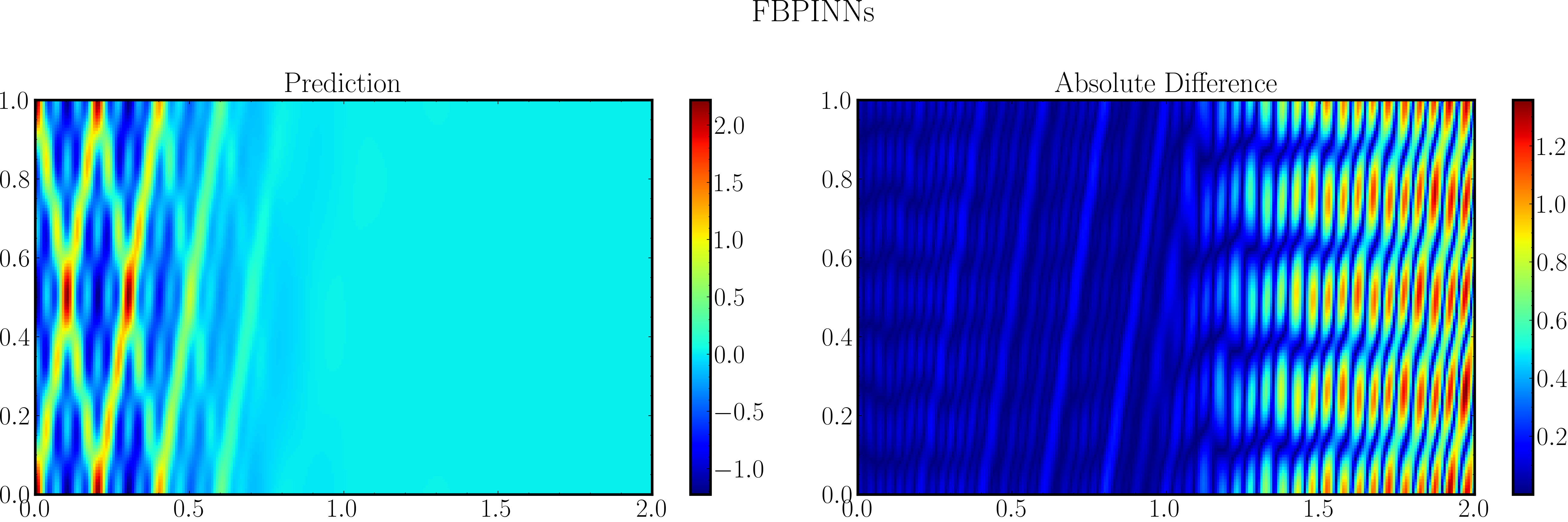}
    \caption{DW: FB-PINNs.}
  \end{subfigure}\hfill
  \begin{subfigure}[t]{0.48\textwidth}
    \centering
    \includegraphics[width=\linewidth]{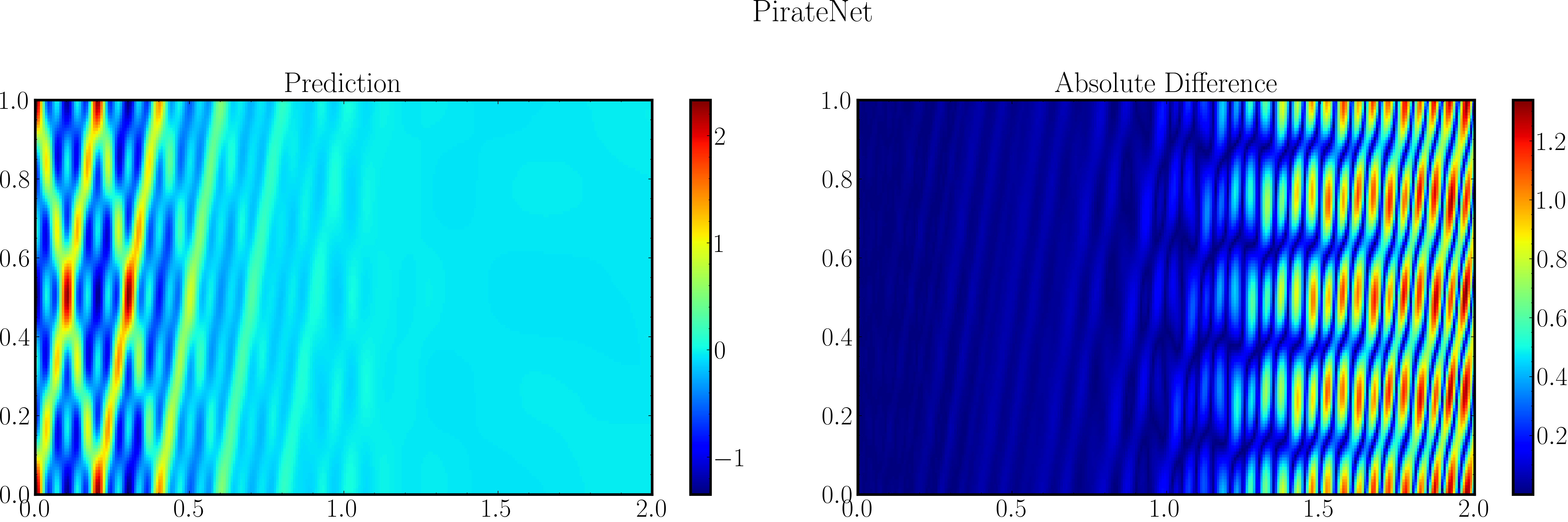}
    \caption{DW: PirateNet.}
  \end{subfigure}

  \vspace{0.6em}

  \begin{subfigure}[t]{0.48\textwidth}
    \centering
    \includegraphics[width=\linewidth]{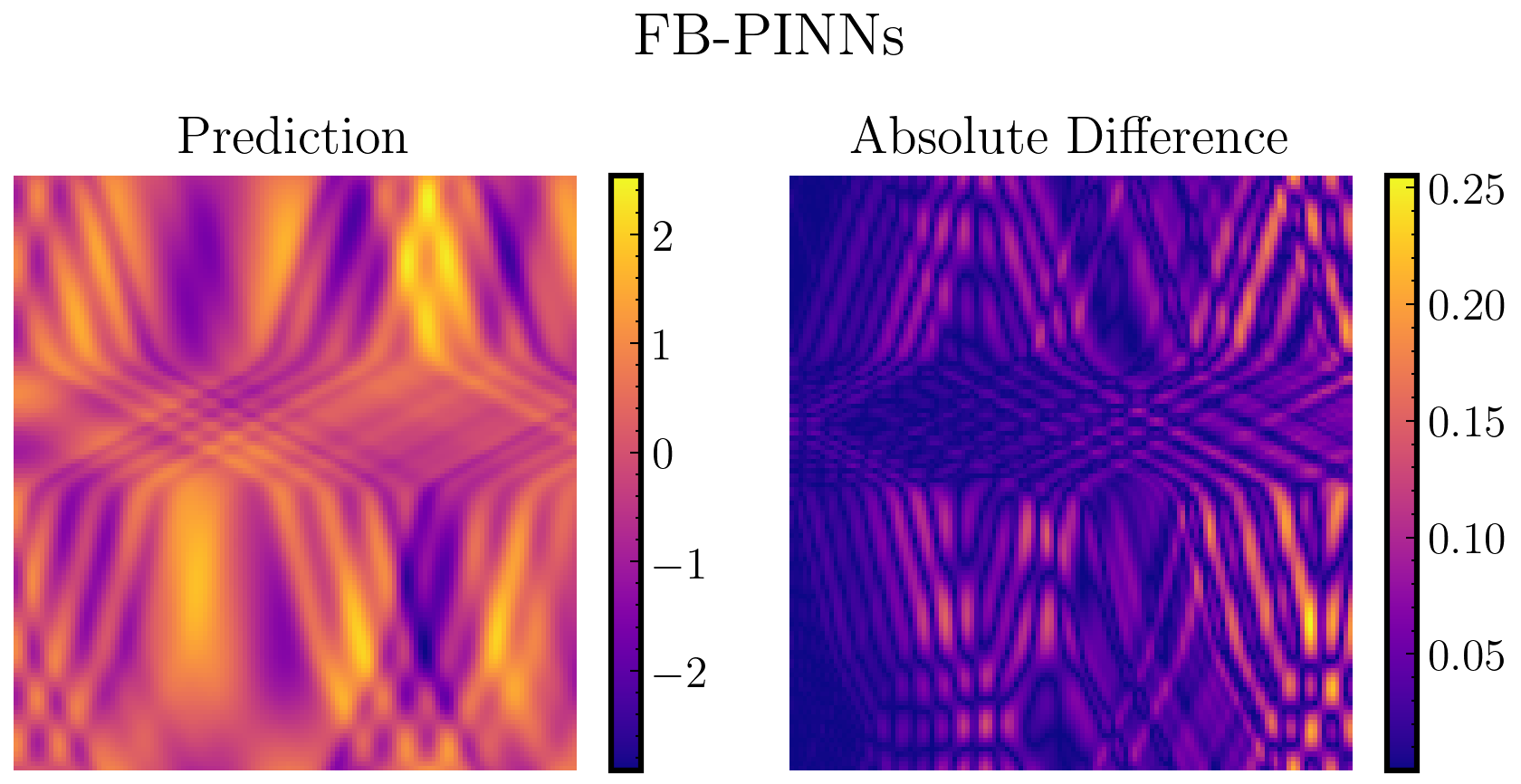}
    \caption{LW: FB-PINNs.}
  \end{subfigure}\hfill
  \begin{subfigure}[t]{0.48\textwidth}
    \centering
    \includegraphics[width=\linewidth]{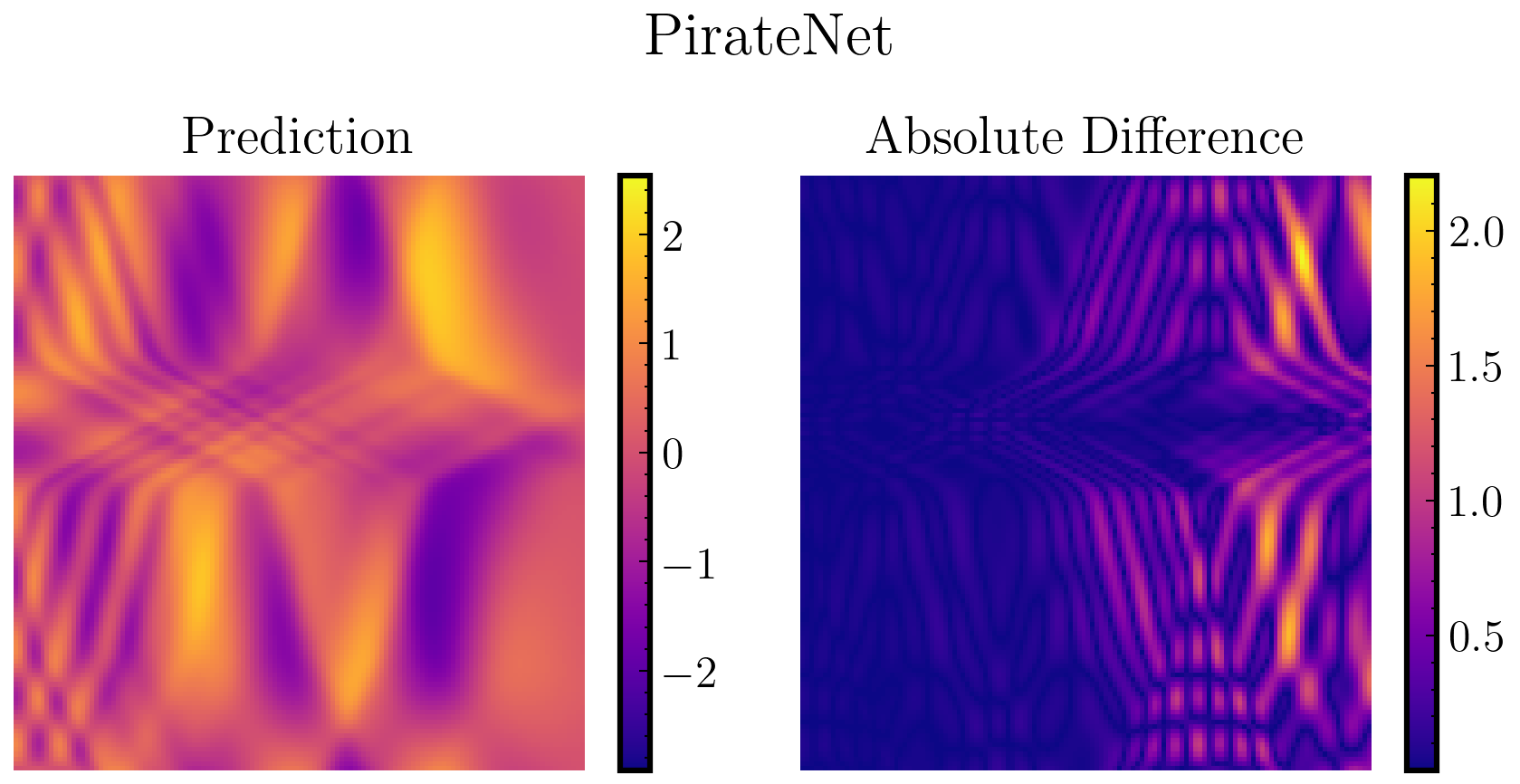}
    \caption{LW: PirateNet.}
  \end{subfigure}
  \caption{Median prediction-error visualizations for FB-PINNs and PirateNet arranged by benchmark: top row advection-diffusion (AD), middle row damped-wave (DW), bottom row layered-wave (LW).}
  \label{fig:appendix_baseline_error_visualizations}
\end{figure}
\subsection{Detailed model statistics}\label{appendix:benchmark_model_stats}
Table~\ref{tab:benchmark_model_stats} reports the relative $L_2$ error statistics for all models and benchmarks. The mean and standard deviation are calculated across 5 random seeds, and the median is also provided to mitigate the influence of outliers. All the visualizations are evaluated on the trained model with the median error within its class.
\begin{table}[htbp]
    \centering

    \renewcommand{\arraystretch}{1.22}
    \vspace{0pt}
      \centering
          \begin{tabular}{llcc}
            \hline
             & Model & Mean $\pm$ 2 std & Median \\
            \hline
            \multirow{4}{*}{AC} & FB-PINNs & 3.032e-05 $\pm$ 1.485e-05 & 2.735e-05 \\
            & L-MoE & 2.881e-05 $\pm$ 8.749e-06 & 2.739e-05 \\
            & PirateNet & 2.256e-05 $\pm$ 3.744e-06 & 2.197e-05 \\
            & ResNet & \textbf{2.145e-05 $\pm$ 1.536e-06} & \textbf{2.172e-05} \\
            \hline
            \multirow{4}{*}{KdV} & FB-PINNs & 8.610e-05 $\pm$ 5.639e-05 & 8.754e-05 \\
            & L-MoE & 1.315e-04 $\pm$ 5.004e-05 & 1.290e-04 \\
            & PirateNet & 1.032e-04 $\pm$ 6.231e-05 & 9.398e-05 \\
            & ResNet & \textbf{6.329e-05 $\pm$ 4.931e-05} & \textbf{4.957e-05} \\
            \hline
            \multirow{4}{*}{Burgers} & FB-PINNs & 4.130e-05 $\pm$ 1.698e-06 & 4.126e-05 \\
            & L-MoE & 4.117e-05 $\pm$ 3.196e-06 & 4.044e-05 \\
            & PirateNet & \textbf{4.026e-05 $\pm$ 6.032e-08} & 4.027e-05 \\
            & ResNet & 4.034e-05 $\pm$ 5.145e-07 & \textbf{4.023e-05} \\
            \hline
            \multirow{4}{*}{AD} & FB-PINNs & 3.448e-01 $\pm$ 1.723e-01 & 3.732e-01 \\
            & L-MoE & \textbf{1.423e-02 $\pm$ 2.562e-02} & \textbf{1.019e-02} \\
            & PirateNet & 6.030e-01 $\pm$ 7.287e-02 & 6.067e-01 \\
            & ResNet & 5.302e-01 $\pm$ 2.001e-01 & 5.532e-01 \\
            \hline
            \multirow{4}{*}{DW} & FB-PINNs & 7.431e-01 $\pm$ 1.084e-02 & 7.406e-01 \\
            & L-MoE & \textbf{1.058e-02 $\pm$ 9.643e-03} & \textbf{1.056e-02} \\
            & PirateNet & 7.394e-01 $\pm$ 6.675e-03 & 7.396e-01 \\
            & ResNet & 5.446e-01 $\pm$ 6.684e-01 & 7.345e-01 \\
            \hline
            \multirow{4}{*}{LW} & FB-PINNs & 6.932e-02 $\pm$ 1.454e-02 & 6.622e-02 \\
            & L-MoE & \textbf{4.264e-02 $\pm$ 2.084e-02} & \textbf{3.637e-02} \\
            & PirateNet & 5.067e-01 $\pm$ 2.275e-01 & 4.835e-01 \\
            & ResNet & 1.814e-01 $\pm$ 1.689e-01 & 1.471e-01 \\
            \hline
          \end{tabular}
      \vspace{6pt}
      \caption{Relative $L_2$ error statistics for Allen--Cahn (AC), KdV, Burgers, advection-diffusion (AD), damped wave (DW), and layered wave (LW). The mean and standard deviation are calculated across 5 random seeds. The median is also provided to mitigate the influence of outliers. The best performance for each benchmark is highlighted in bold.}
      \label{tab:benchmark_model_stats}
\end{table}

\clearpage
\section{Ablation study}\label{appendix:ablation}

We perform ablations on several key design choices in our method. The process is greedy. That is, we start from a manually chosen hyperparameter set, then ablate one hyperparameter at a time while keeping others fixed. The best choice from each round is carried to the next round. We run ablations on both advection-diffusion and damped-wave benchmarks. Because the ablation grid is extensive, we use a fixed random seed and do not run multiple seeds per setting to avoid conflating trends with initialization noise. The following tables report relative $L_2$ error for each design on Latent-MoE. AD denotes advection-diffusion and DW denotes damped wave.

\subsection{Encoder}\label{appendix:encoder_ablation}
Table~\ref{tab:ablation_encoder_latent_moe} shows encoder-design ablations. The regular encoder is a periodic embedding followed by random Fourier features. We find that the centered, weight-shared MoE encoder performs best on
both benchmarks. The centered-but-independent encoder (separate $g_i$
per subdomain) does substantially worse, as does the uncentered MoE
encoder. We conjecture that the encoder defines the base coordinate representation; while random Fourier features enrich it, independent random projections per expert can introduce representation misalignment across experts.
\paragraph{Regular} A regular encoder $g$ processes the input coordinate uniformly without per-subdomain adjustment.
\paragraph{Shared MoE, centered} A MoE encoder with only a set of parameter that is shifted to the centroid of each subdomain is described in the main manuscript (\S~\ref{subsec:encoder}).
\paragraph{MoE, centered} A truly MoE encoder with centering further uses different encoders $g_i$ for $i-$th subdomain, in addition to the per-subdomain shift. 
\begin{equation}
  z^{(0)} = \sum_{i=1}^{E} \phi(x^{\text{dd}})_i g_i(x^{\text{dd}} - c_i, \, x^{\text{ndd}}),
\end{equation}
where $x^{\text{dd}}$, $x^{\text{ndd}}$ denote the decomposed and
non-decomposed coordinates, and $g_i$ has the same encoder architecture but a distinctive set of parameters that dedicates to the $i-$th subdomain.
\paragraph{MoE, uncentered} Uncentered MoE encoder removes the per expert centering. That is, 
\begin{equation}
  z^{(0)} = \sum_{i=1}^{E} \phi(x^{\text{dd}})_i g_i(x^{\text{dd}}, \, x^{\text{ndd}}),
\end{equation}
\begin{table}[H]
  \centering
  \caption{Encoder ablation.}
  \label{tab:ablation_encoder_latent_moe}
  \begin{tabular}{lcc}
    \hline
    Encoder & AD & DW \\
    \hline
    Regular & 0.21985 & 0.69486 \\
    Shared MoE, centered & \textbf{0.00285} & \textbf{0.0049675} \\
    MoE, centered & 0.4816 & 0.73586 \\
    MoE, uncentered &0.28402 &0.73568\\
    \hline
  \end{tabular}
\end{table}

\subsection{Bump function plateau ratio}
Table \ref{tab:ablation_plateau_latent_moe} shows the ablation results for the bump function plateau ratio, which controls the extent of the overlap between neighboring experts. We find that the best plateau ratio is $1$, which means that the width of the plateau is the same as the width of the bridge. The plateau ratio here is the plateau length divided by the bridge length, so a ratio of $0$ gives pure bridges with maximal overlap and the partition becomes fully disjoint as the ratio goes to infinity.
\begin{table}[H]
  \centering
  \caption{Bump function plateau ratio ablation.}
  \label{tab:ablation_plateau_latent_moe}
  \begin{tabular}{ccc}
    \hline
    Plateau ratio & AD & DW \\
    \hline
    0 & 0.004108 & 0.005607 \\
    1 & \textbf{0.002391} & \textbf{0.004967} \\
    2 & 0.003808 & 0.018568 \\
    3 & 0.010355 & 0.052192 \\
    \hline
  \end{tabular}
\end{table}

\subsection{\texorpdfstring{Relative causality coefficient $\epsilon$}{Relative causality coefficient epsilon}}
Table \ref{tab:ablation_epsilon_latent_moe} shows the ablation results for the relative causality coefficient $\epsilon$. We find that the best $\epsilon$ is $4$ for the advection-diffusion benchmark and $2$ for the damped wave benchmark. However, we do note that the performance is not very sensitive to the choice of $\epsilon$ within a reasonable range.
\begin{table}[H]
  \centering
  \caption{Relative causality coefficient ablation.}
  \label{tab:ablation_epsilon_latent_moe}
  \begin{tabular}{lcc}
    \hline
    $\epsilon$ & AD & DW \\
    \hline
    0 & 0.002504 & 0.0048625 \\
    2 & 0.002807 & \textbf{0.004506} \\
    4 & \textbf{0.002482} & 0.0096405 \\
    8 & 0.003511 & 0.01152 \\
    \hline
  \end{tabular}
\end{table}

\section{Scaling experiments}\label{appendix:scaling}
In addition to hyperparameter ablations, we study the scalability of ResNet and Latent-MoE, whose parameter counts can be matched. Figure~\ref{fig:appendix_scaling_curves} shows that Latent-MoE benefits clearly from scaling the number of experts. By contrast, widening the corresponding feedforward layers in ResNet yields only marginal improvement. This suggests that Latent-MoE provides a more principled scaling path than naive widening.
\begin{figure}[H]
  \centering
  \begin{subfigure}[t]{0.42\linewidth}
    \centering
    \includegraphics[width=\linewidth]{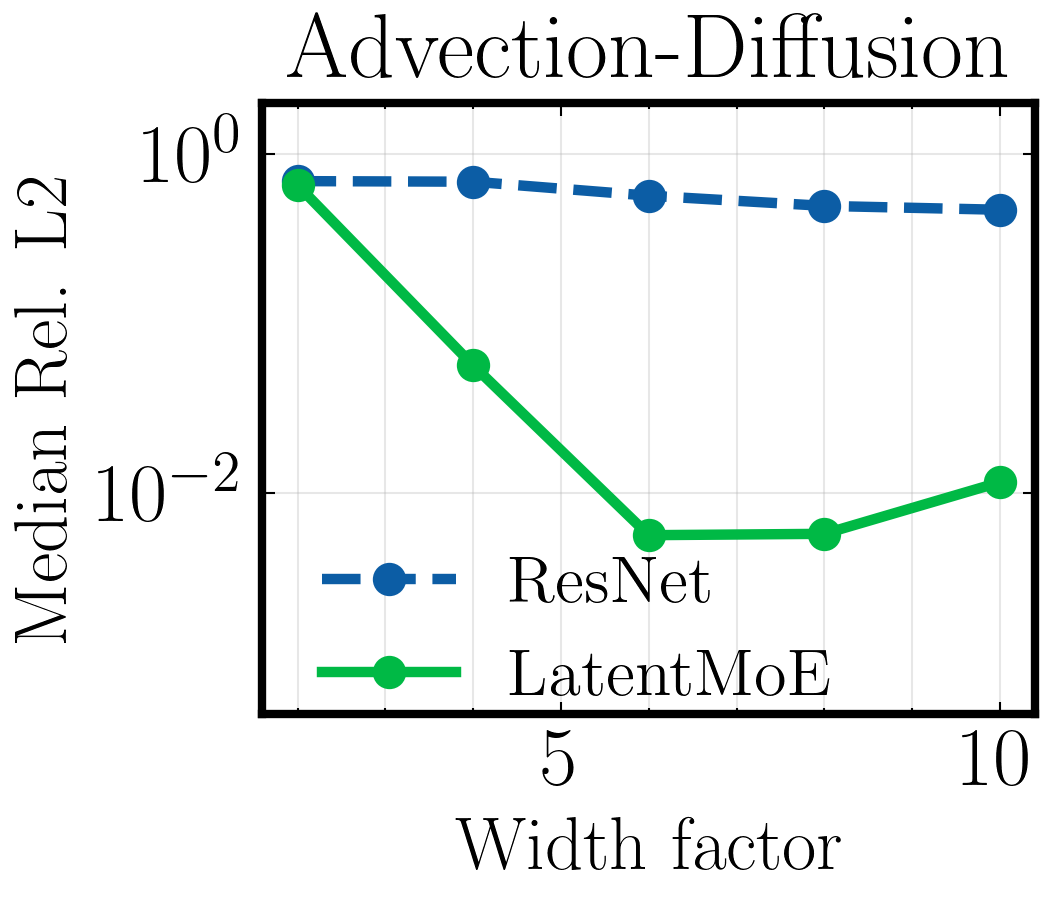}
    \caption{AD: Scaling.}
  \end{subfigure}
  \begin{subfigure}[t]{0.42\linewidth}
    \centering
    \includegraphics[width=\linewidth]{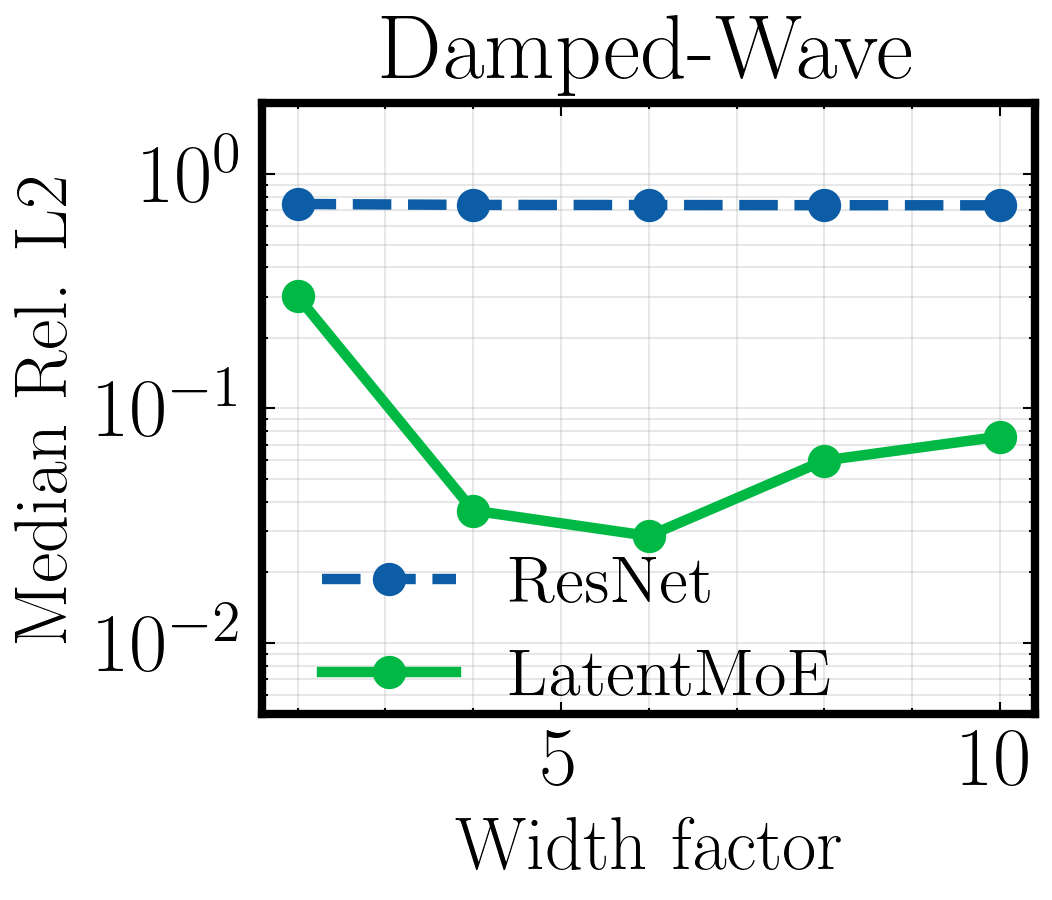}
    \caption{DW: Scaling.}
  \end{subfigure}
  \caption{Scaling diagnostics for advection-diffusion (AD) and damped wave (DW), reported with the same protocol used in Section~\ref{sec:experiments}. Each data point is the median among the 3 random seeds.}
  \label{fig:appendix_scaling_curves}
\end{figure}

\section{Variable stage transition experiments}\label{sec:variable_stage_transition}

In addition, Figure~\ref{fig:variable_transition_heatmaps_latest} shows windowed relative $L_2$-error heatmaps for both benchmarks across all four architectures (ResNet, PirateNet, FB-PINNs, and Latent-MoE), where each horizontal slice corresponds to one benchmark under different stage-transition times. That is,
\begin{align}
    \text{Error}(t) = \frac{||f_\text{pred}(\cdot, s) - f_\text{truth}(\cdot, s)||_2}{||f_\text{truth}(\cdot, s)||_2}, s\in[t-\Delta t/2, t+\Delta t / 2],
\end{align}
where $\Delta t$ is the window size ($\Delta t=1/16$). The heatmaps show that sharp stage transitions challenge all baselines; however, Latent-MoE remains markedly more robust, with errors that are more localized around transition boundaries. We also observe that subdomain boundaries do not necessarily induce elevated error, even when stage transitions are misaligned with the soft subdomain boundaries.
\begin{figure}[H]
  \centering
  \begin{subfigure}[t]{0.24\textwidth}
    \centering
    \includegraphics[width=\linewidth]{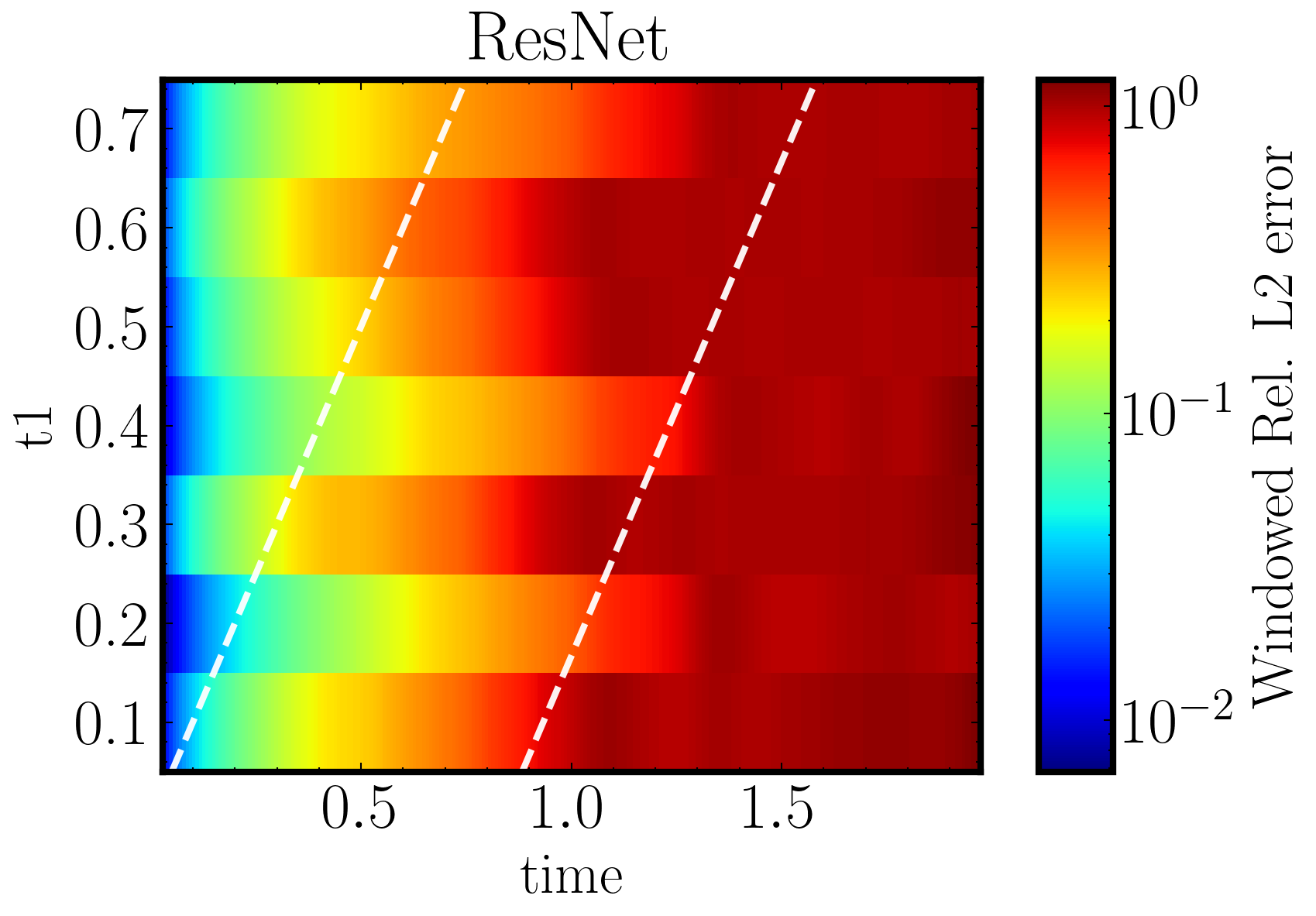}
    \caption{AD: ResNet.}
  \end{subfigure}
  \hfill
  \begin{subfigure}[t]{0.24\textwidth}
    \centering
    \includegraphics[width=\linewidth]{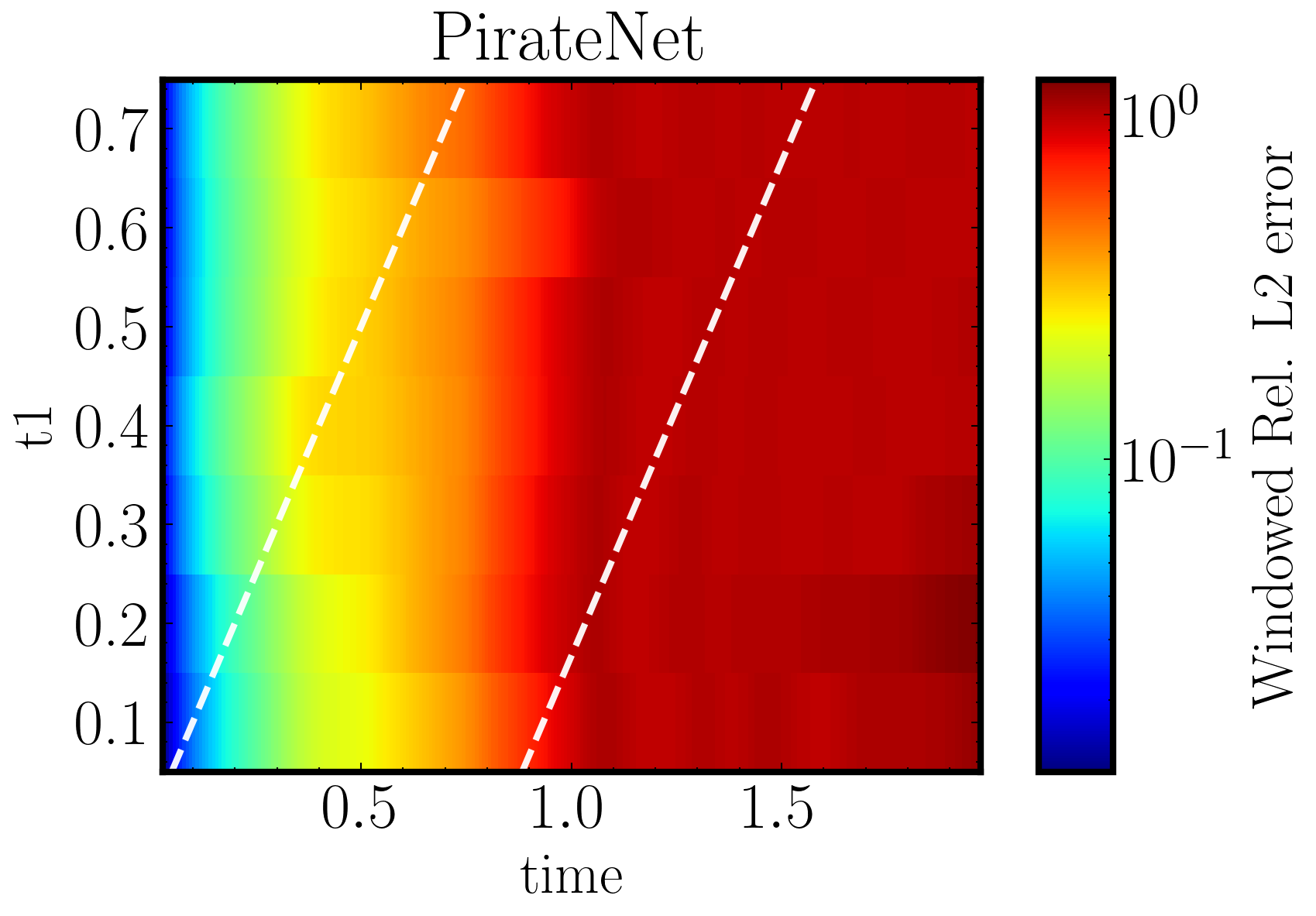}
    \caption{AD: PirateNet.}
  \end{subfigure}
  \hfill
  \begin{subfigure}[t]{0.24\textwidth}
    \centering
    \includegraphics[width=\linewidth]{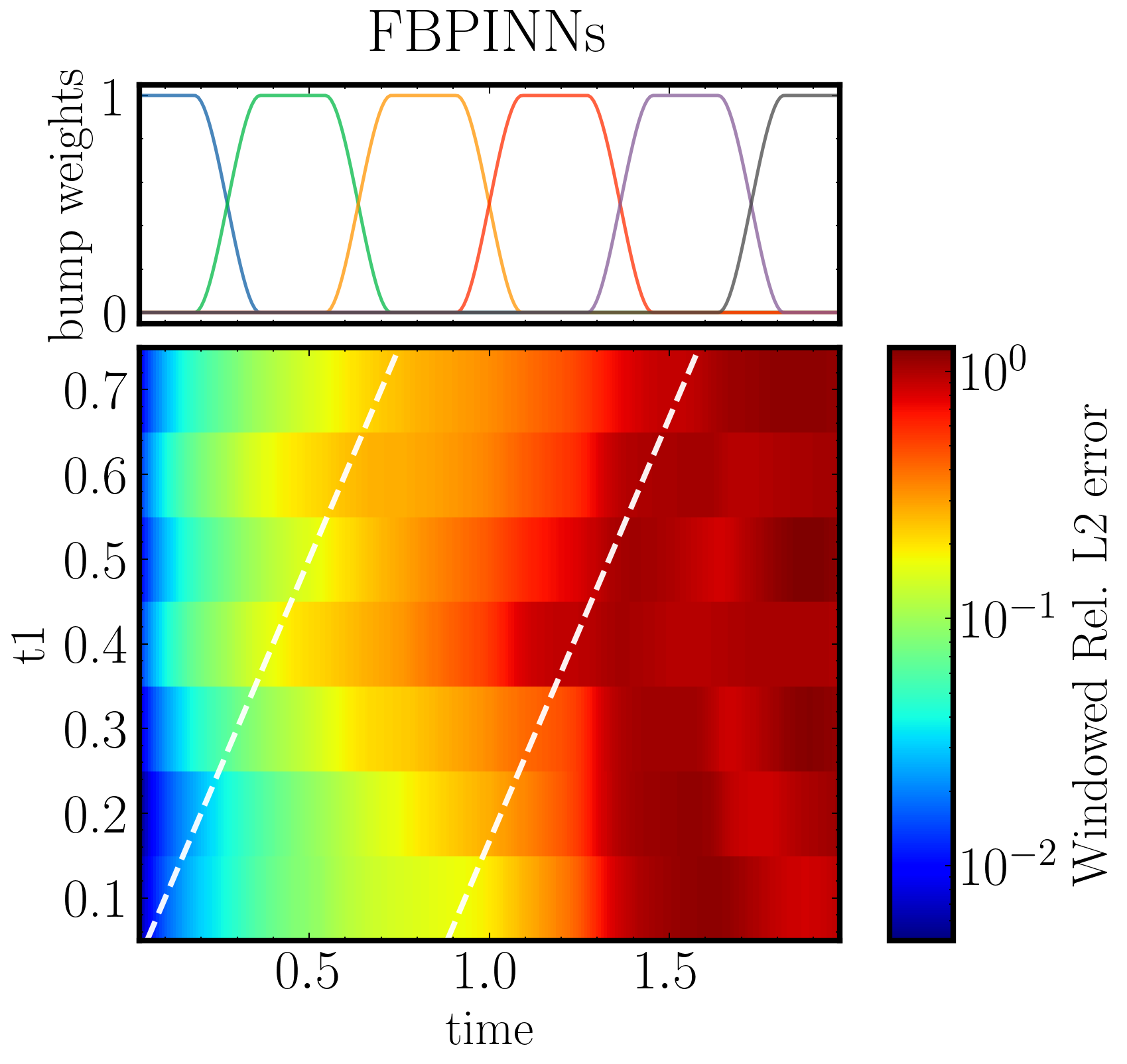}
    \caption{AD: FB-PINNs.}
  \end{subfigure}
  \hfill
  \begin{subfigure}[t]{0.24\textwidth}
    \centering
    \includegraphics[width=\linewidth]{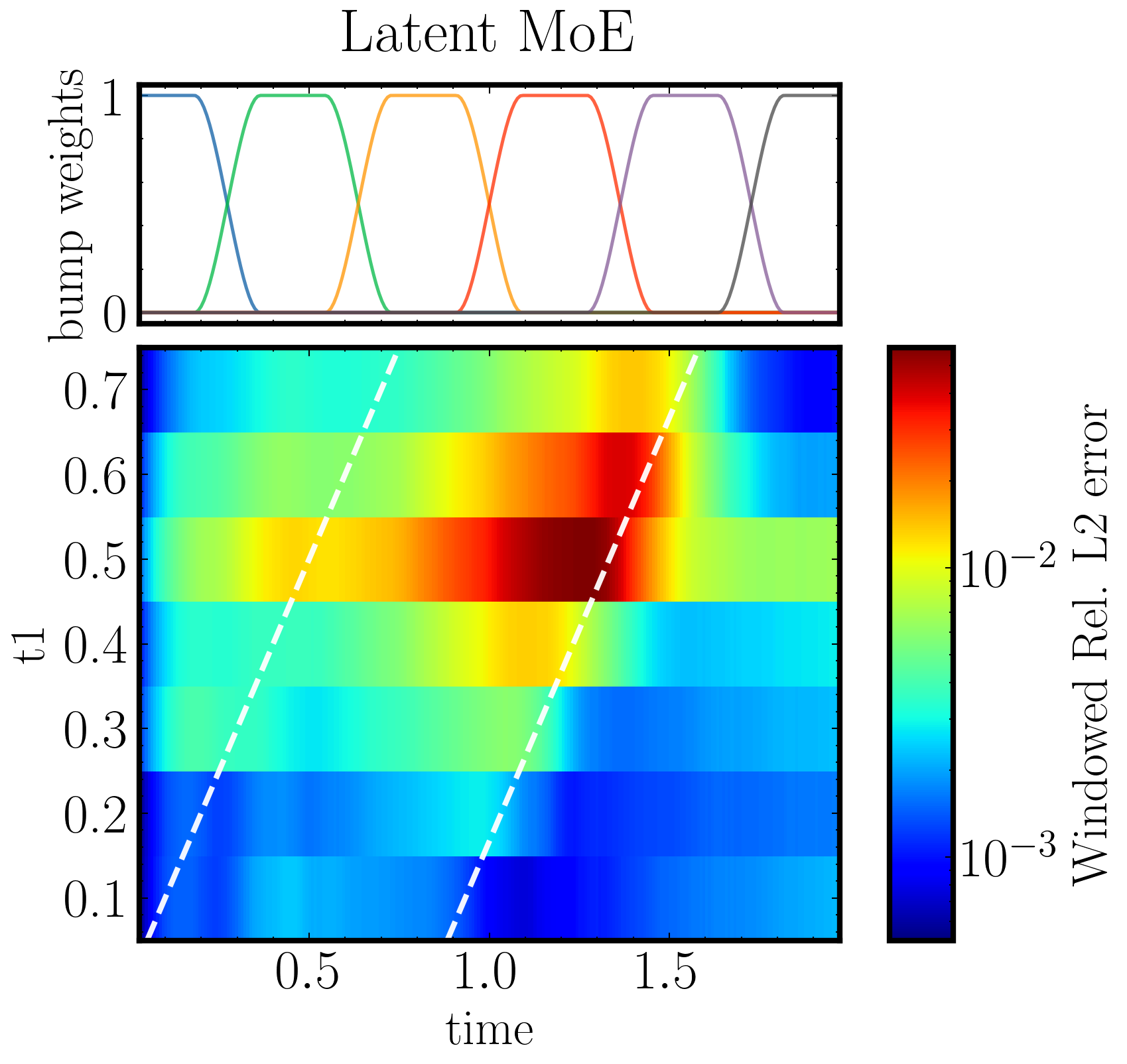}
    \caption{AD: Latent-MoE.}
  \end{subfigure}
    
  \vspace{0.5em}

  \begin{subfigure}[t]{0.24\textwidth}
    \centering
    \includegraphics[width=\linewidth]{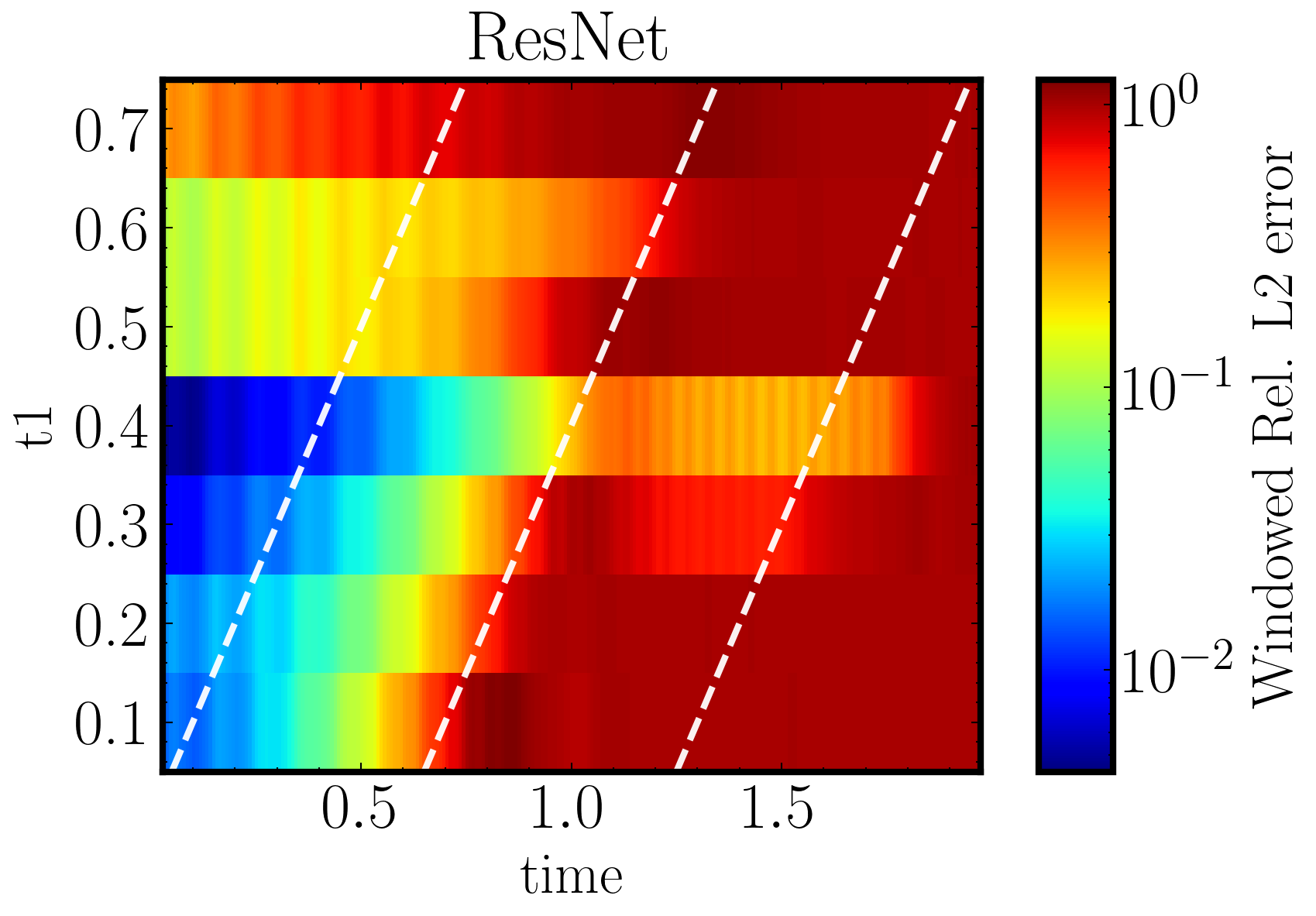}
    \caption{DW: ResNet.}
  \end{subfigure}
  \hfill
  \begin{subfigure}[t]{0.24\textwidth}
    \centering
    \includegraphics[width=\linewidth]{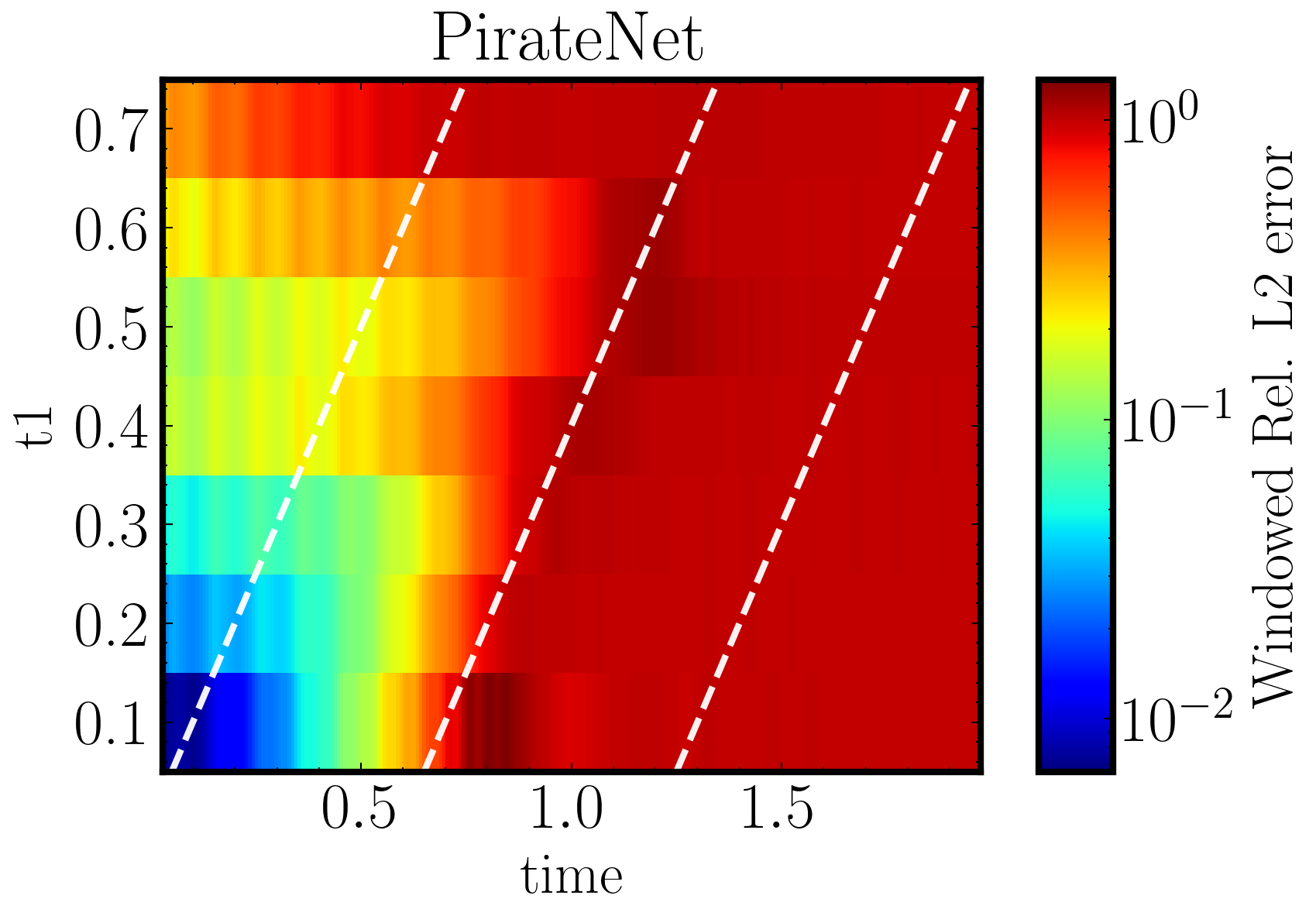}
    \caption{DW: PirateNet.}
  \end{subfigure}
  \hfill
  \begin{subfigure}[t]{0.24\textwidth}
    \centering
    \includegraphics[width=\linewidth]{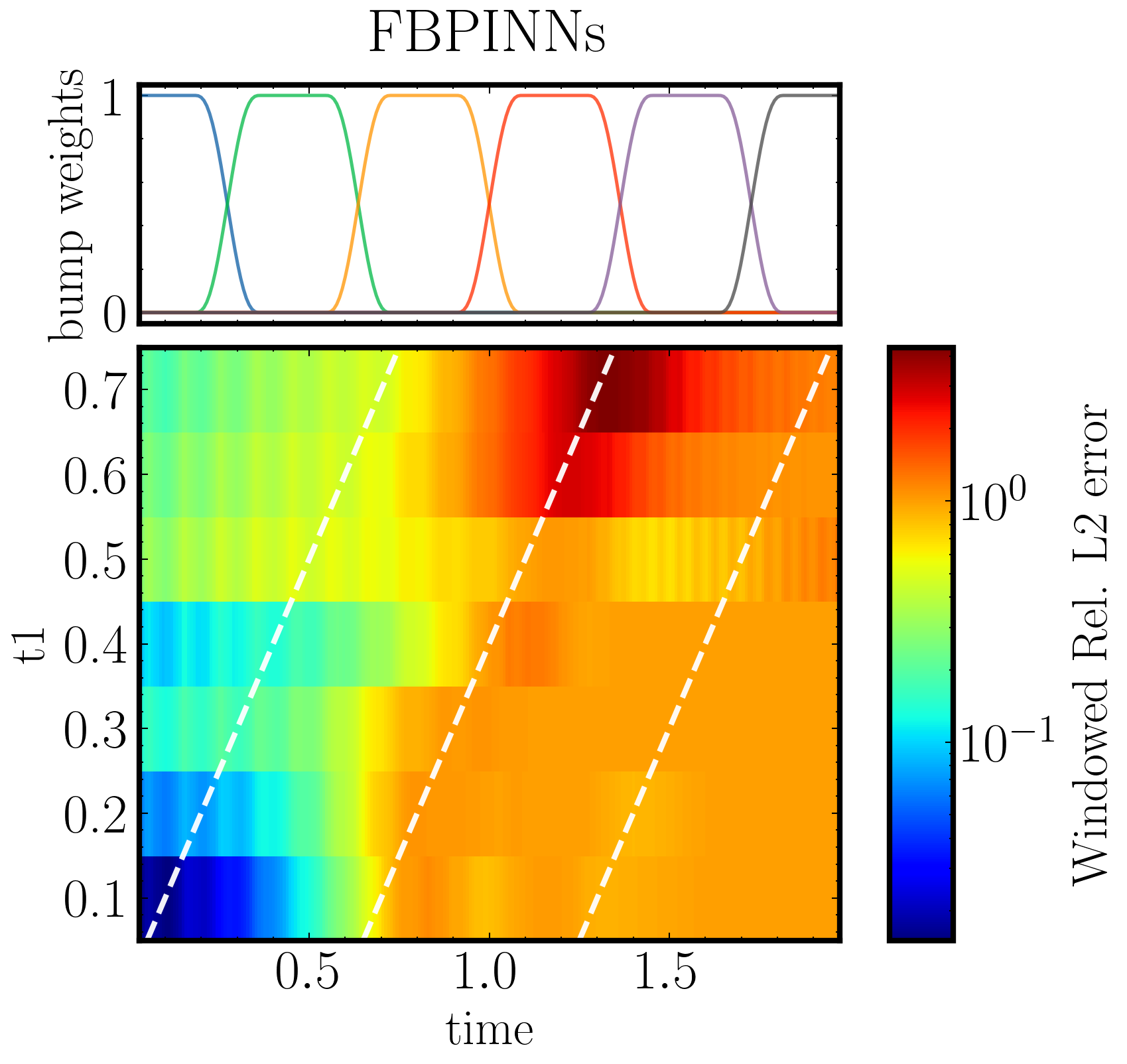}
    \caption{DW: FB-PINNs.}
  \end{subfigure}
  \hfill
  \begin{subfigure}[t]{0.24\textwidth}
    \centering
    \includegraphics[width=\linewidth]{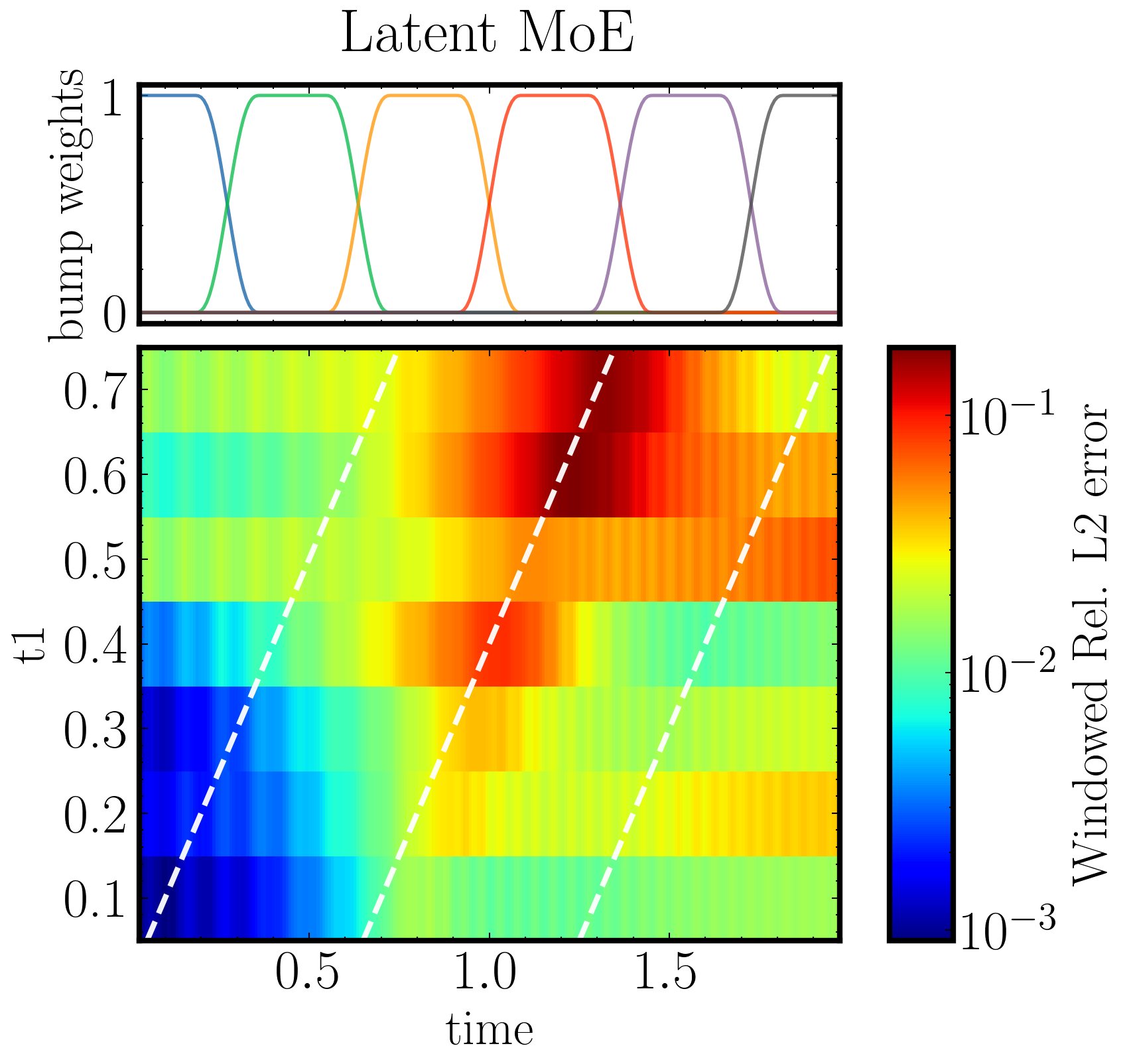}
    \caption{DW: Latent-MoE.}
  \end{subfigure}
  \caption{Windowed relative $L_2$ error for variable stage transitions. Top row: advection-diffusion (AD). Bottom row: damped wave (DW). In each row, we compare ResNet, PirateNet, FB-PINNs, and Latent-MoE under the same transition-shift settings.}
  \label{fig:variable_transition_heatmaps_latest}
\end{figure}

\begin{figure}[H]
  \centering
  \begin{subfigure}[t]{0.35\textwidth}
    \centering
    \includegraphics[width=\linewidth]{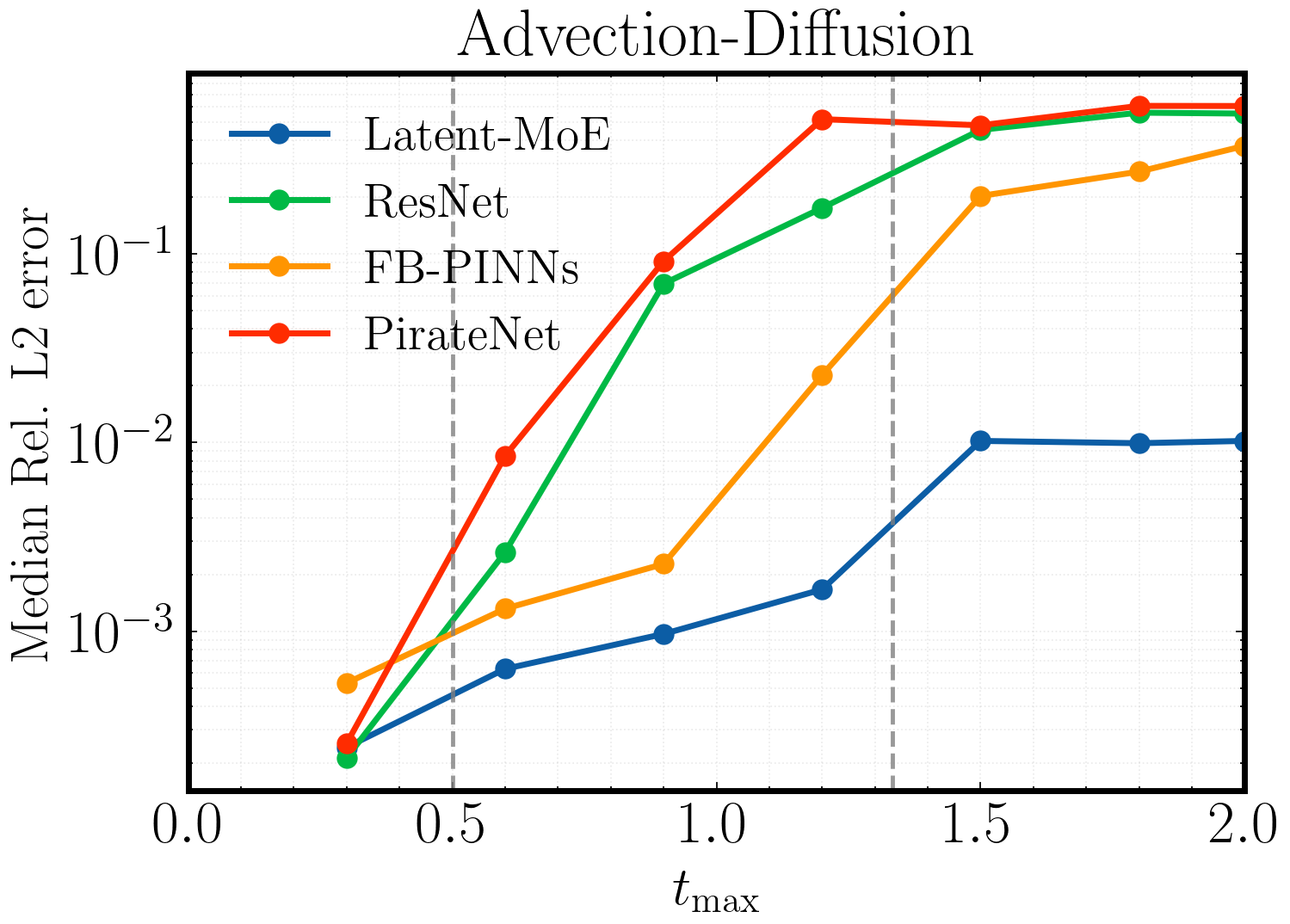}
    \caption{Advection-diffusion (AD).}
  \end{subfigure}
  \begin{subfigure}[t]{0.35\textwidth}
    \centering
    \includegraphics[width=\linewidth]{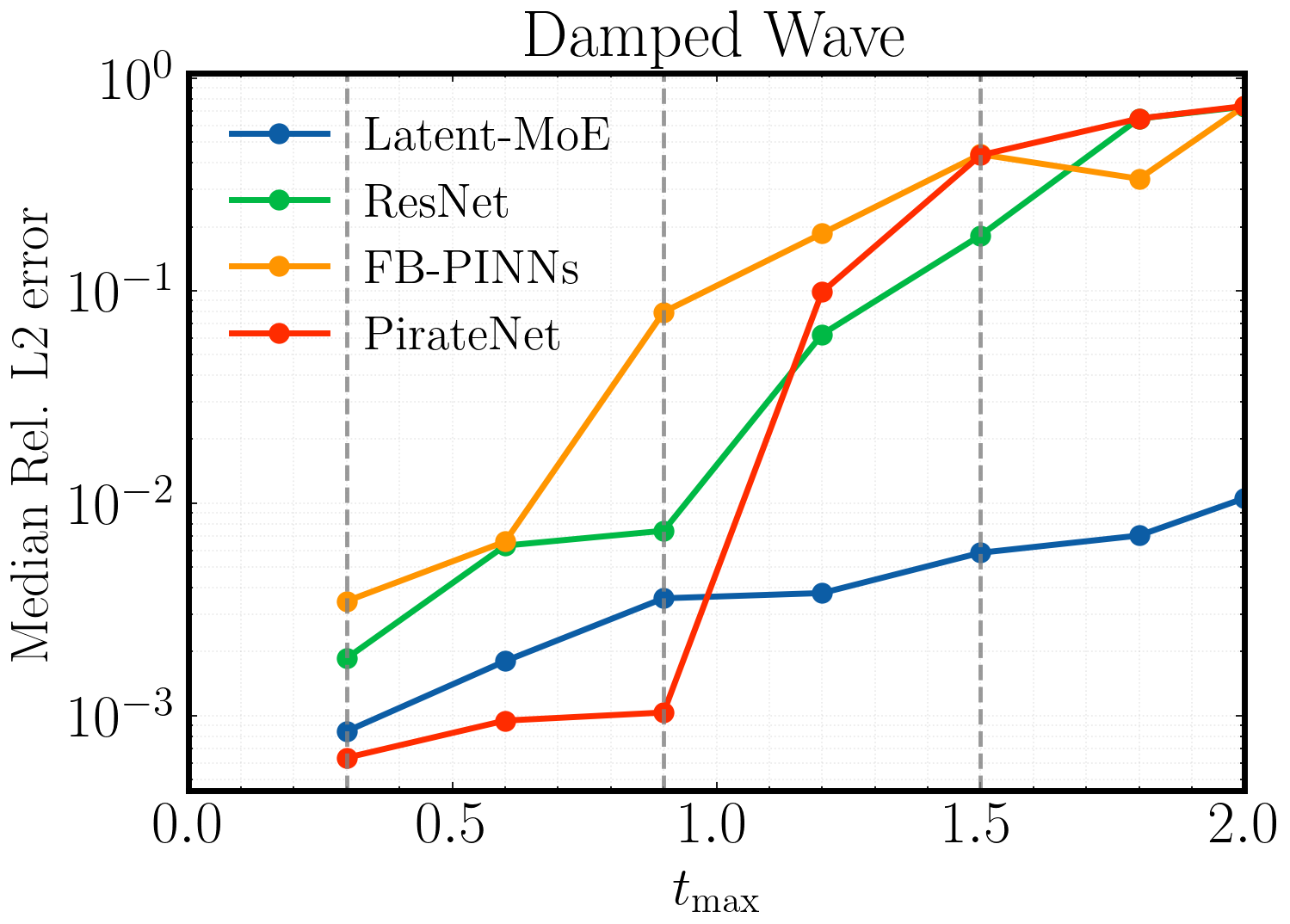}
    \caption{Damped wave (DW).}
  \end{subfigure}
  \caption{{\small Median relative $L_2$ error of the models trained for variable $t_\text{max}$ cutoffs for the advection-diffusion (left) and damped-wave (right) benchmarks. PirateNet and ResNet are competitive or even better at short horizons before the first stage transition, but their errors increase sharply as $t_{\max}$ extends into the multi-stage regime. Latent-MoE exhibits a markedly more consistent error across the full range of $t_{\max}$.}}
  \label{fig:variable_tmax}
\end{figure}

\section{Variable time horizon experiments}\label{sec:variable_tmax}

Figure~\ref{fig:variable_tmax} reports the median relative $L_2$ error as a function of the evaluation cutoff $t_{\max}$, sweeping from a short horizon—where the dynamics are dominated by pure wave propagation before the diffusion or damping regime activates—to the full benchmark duration.
At short time horizons, PirateNet and ResNet achieve competitive or even lower errors, benefiting from the relative simplicity and homogeneity of the early wave-propagation phase.
However, as $t_{\max}$ grows and the physics transitions into the multi-stage regime, both baselines exhibit a rapid and sustained rise in error, indicating that their single-network representations struggle to accommodate the change in dynamics.
Latent-MoE, by contrast, maintains a substantially more consistent error profile across the full range of $t_{\max}$, demonstrating that its latent domain decomposition effectively absorbs the increasing complexity introduced by later-stage dynamics.

\section{Sensitivity against random seeds}

Table~\ref{tab:benchmark_model_stats_main} and Figure~\ref{fig:boxplots_appendix} shows that sensitivity to initialization exists for all benchmarks, especially for ResNet. In this section, we visualize training instability through full training trajectories. Figure~\ref{fig:all_random_seed_trajectories} shows relative $L_2$ trajectories for all random seeds of each benchmark-architecture combination. Although Latent-MoE may appear to have larger variance on the log-scale $L_2$ axis, its convergence trend remains consistent. Larger relative $L_2$ values are often due to slower early convergence, and the trajectories still exhibit a steadily decreasing trend at the training cutoff.

\begin{figure}[H]
  \centering
  \begin{subfigure}[t]{0.3\linewidth}
    \centering
    \includegraphics[width=\linewidth]{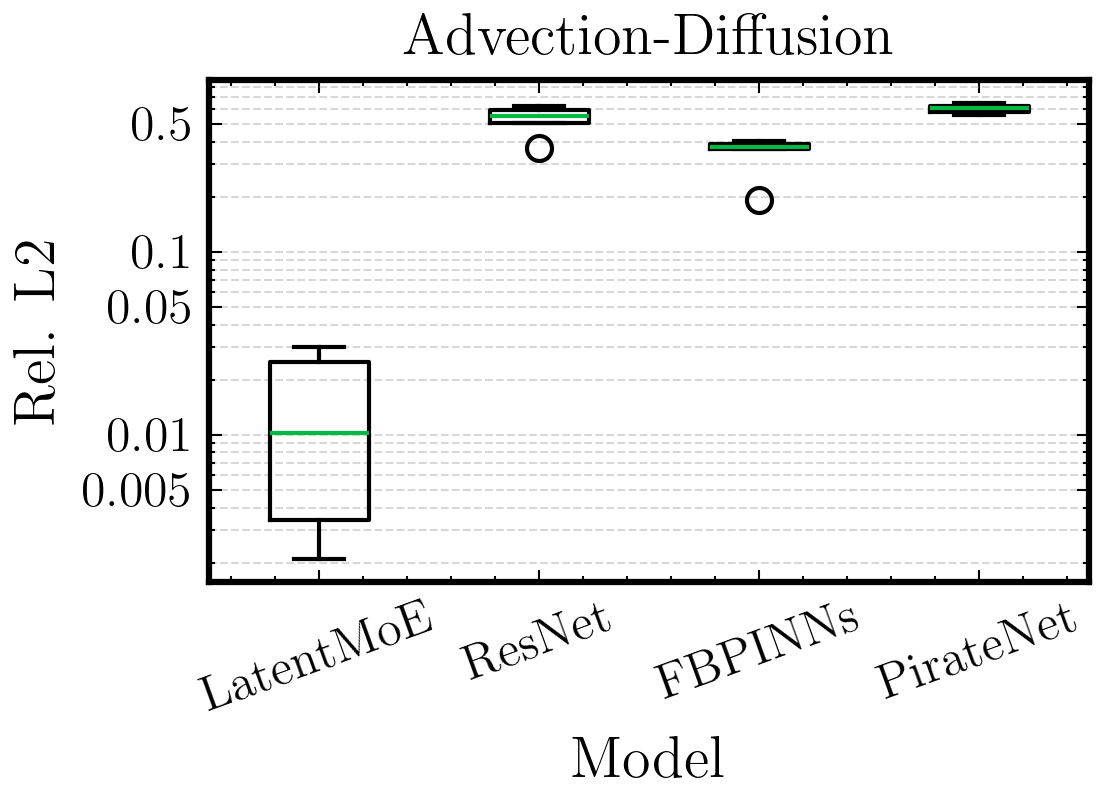}
    \caption{Advection-diffusion.}
  \end{subfigure}
  \begin{subfigure}[t]{0.3\linewidth}
    \centering
    \includegraphics[width=\linewidth]{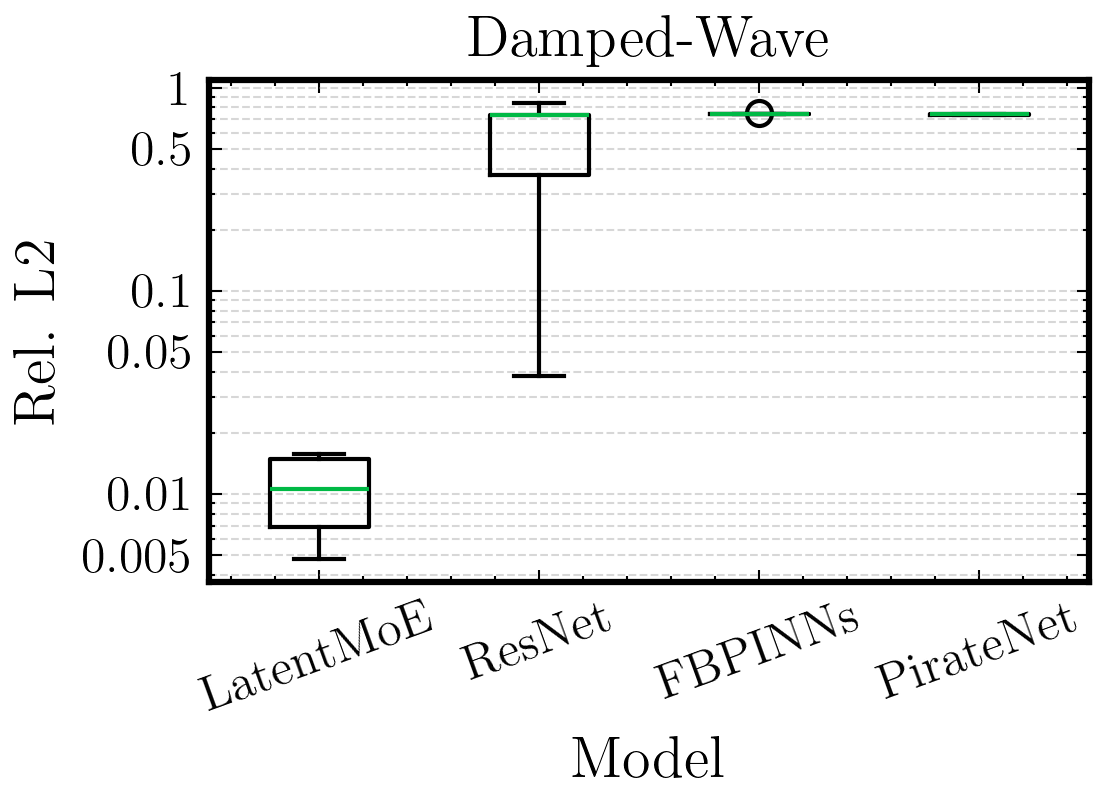}
    \caption{Damped wave.}
  \end{subfigure}
  \begin{subfigure}[t]{0.3\linewidth}
    \centering
    \includegraphics[width=\linewidth]{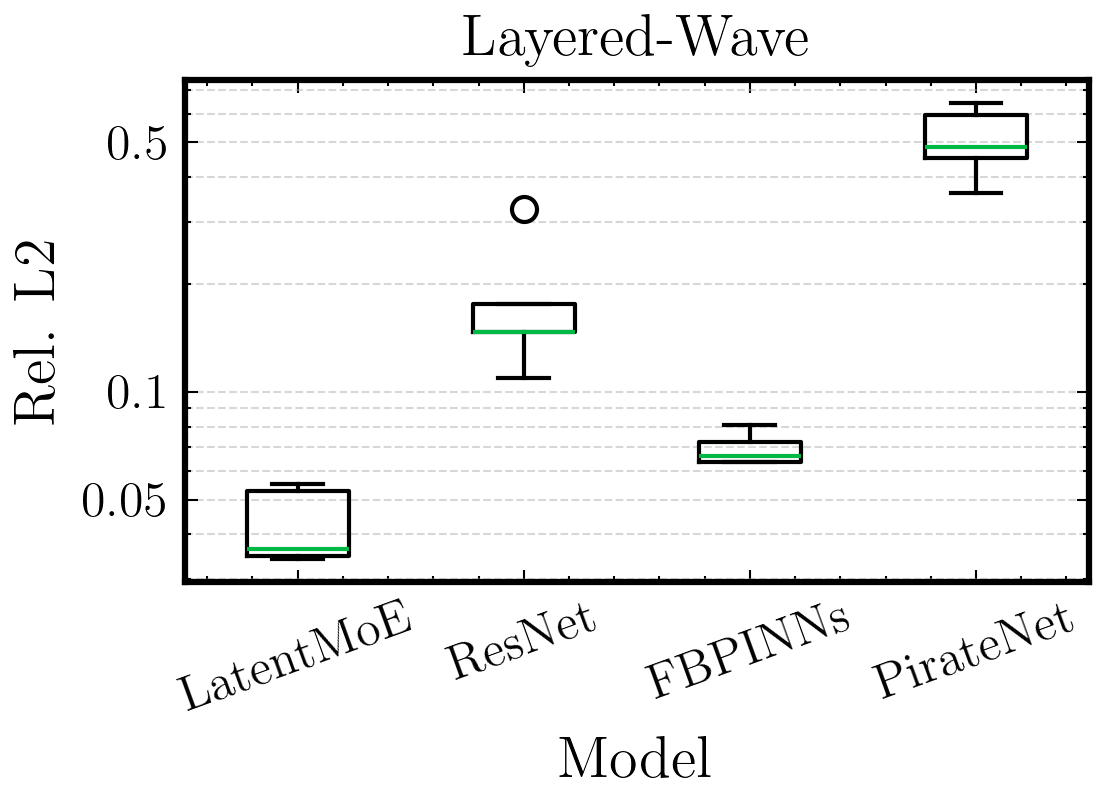}
    \caption{Layered wave.}
  \end{subfigure}
  \caption{Test-error boxplots over random seeds for the compared models on all benchmarks.}
  \label{fig:boxplots_appendix}
\end{figure}

\begin{figure}[H]
  \centering
  \begin{subfigure}[t]{0.24\textwidth}
    \centering
    \includegraphics[width=\linewidth]{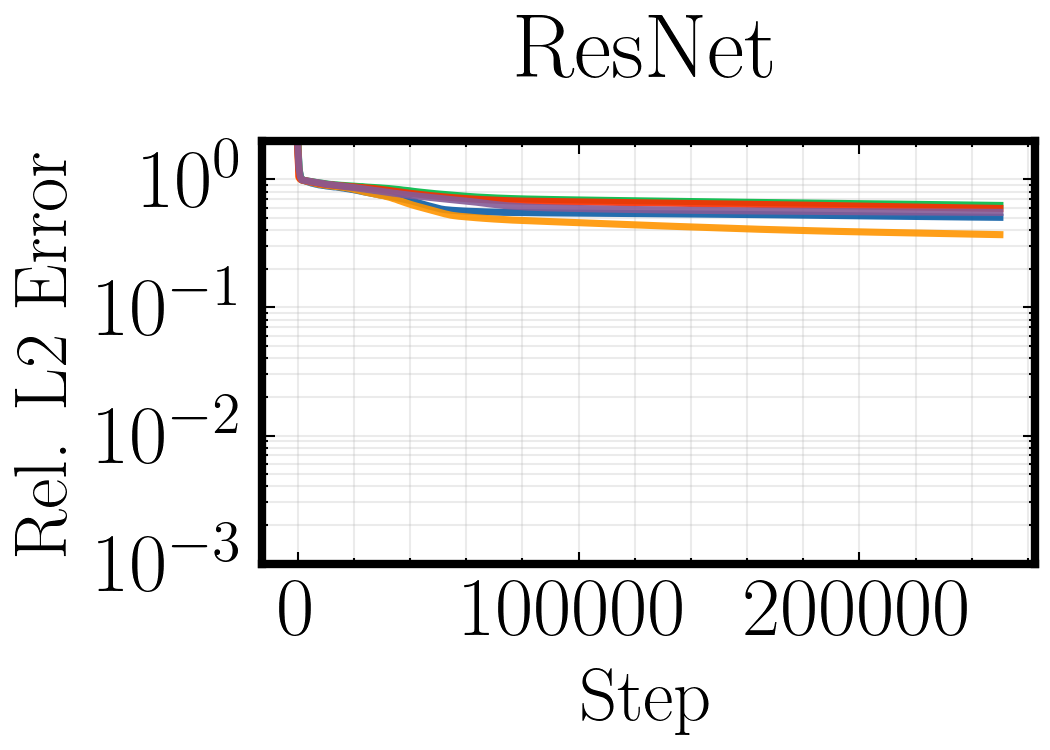}
    \caption{AD: ResNet.}
  \end{subfigure}
  \hfill
  \begin{subfigure}[t]{0.24\textwidth}
    \centering
    \includegraphics[width=\linewidth]{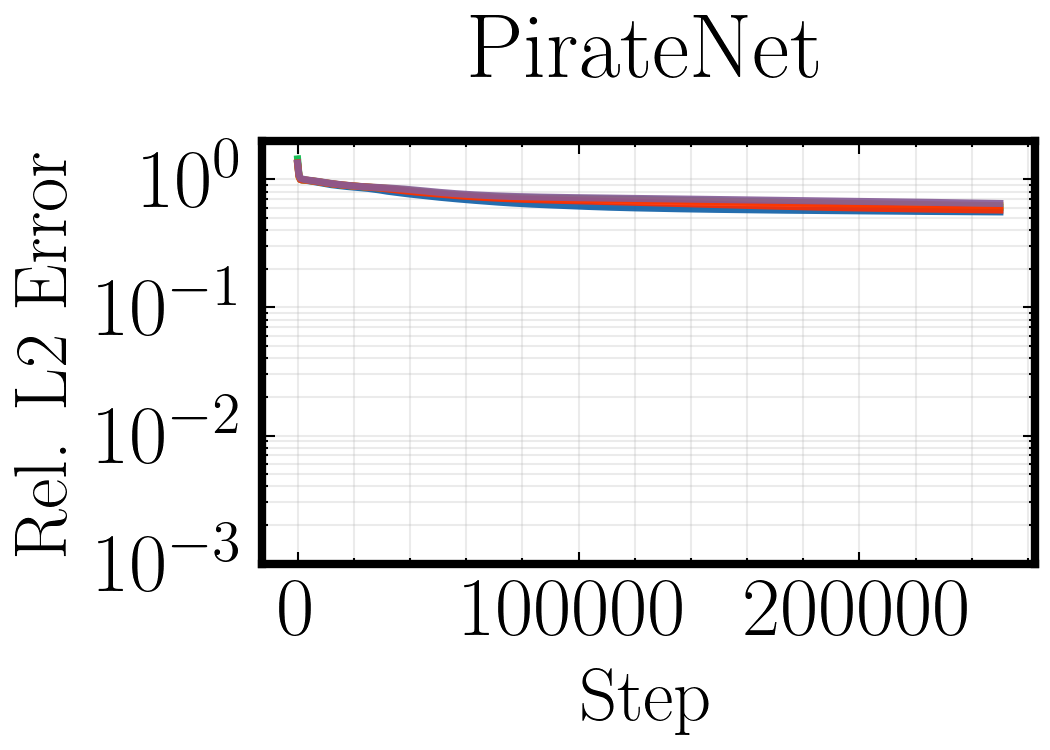}
    \caption{AD: PirateNet.}
  \end{subfigure}
  \hfill
  \begin{subfigure}[t]{0.24\textwidth}
    \centering
    \includegraphics[width=\linewidth]{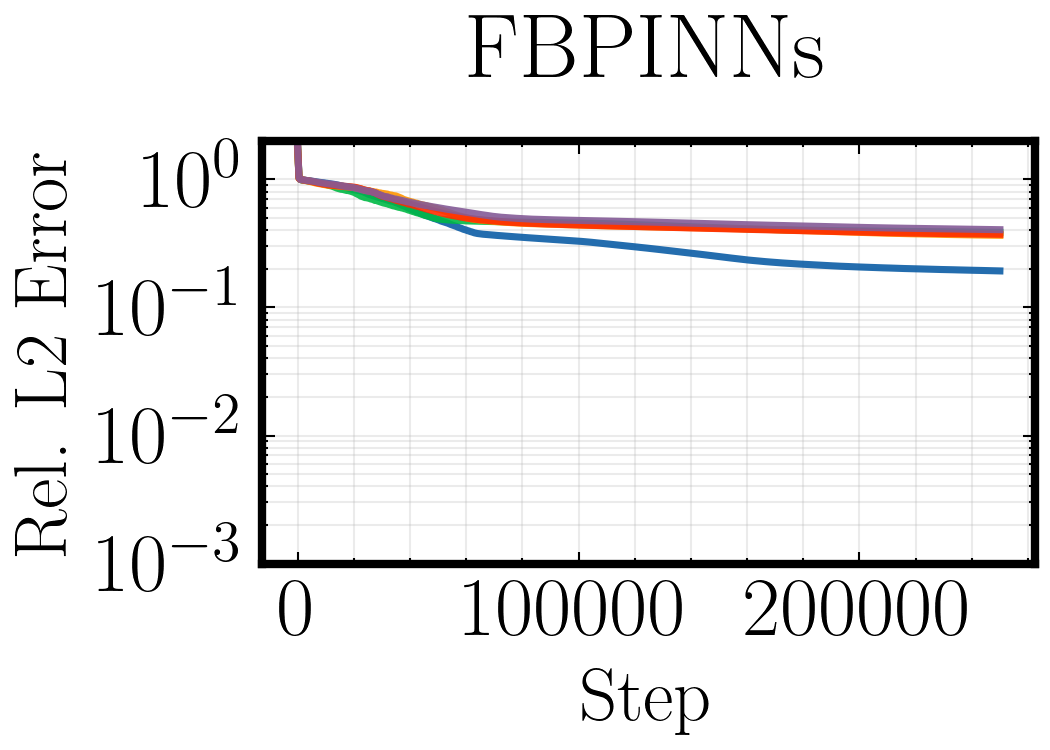}
    \caption{AD: FB-PINNs.}
  \end{subfigure}
  \hfill
  \begin{subfigure}[t]{0.24\textwidth}
    \centering
    \includegraphics[width=\linewidth]{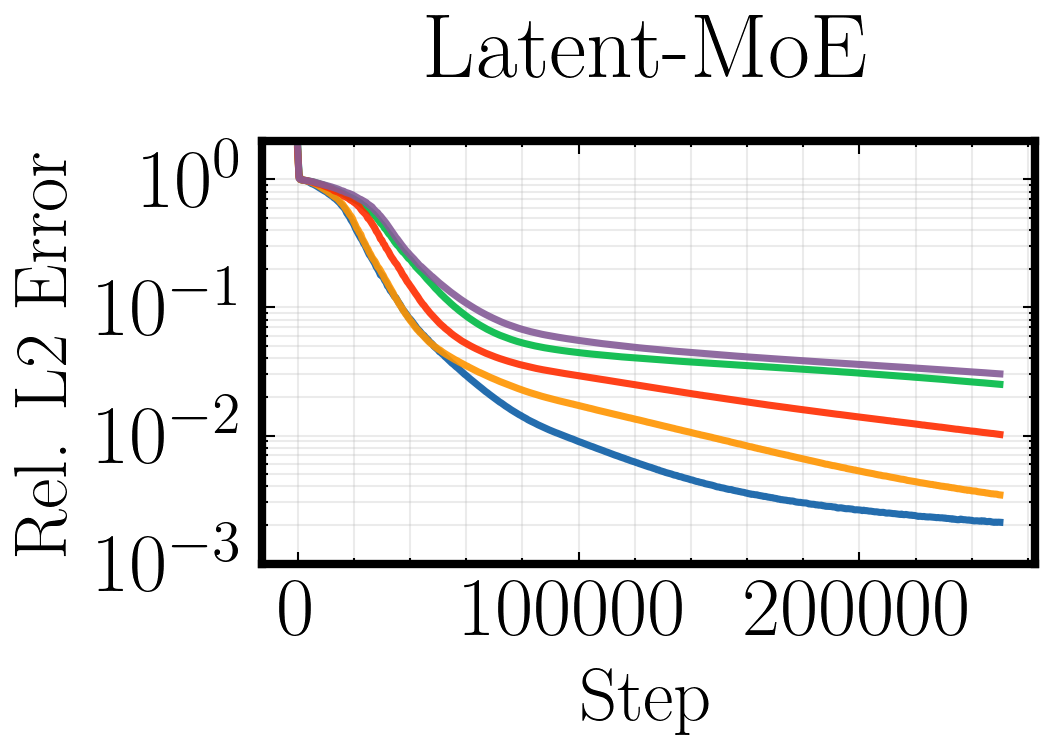}
    \caption{AD: Latent-MoE.}
  \end{subfigure}

  \vspace{0.5em}

  \begin{subfigure}[t]{0.24\textwidth}
    \centering
    \includegraphics[width=\linewidth]{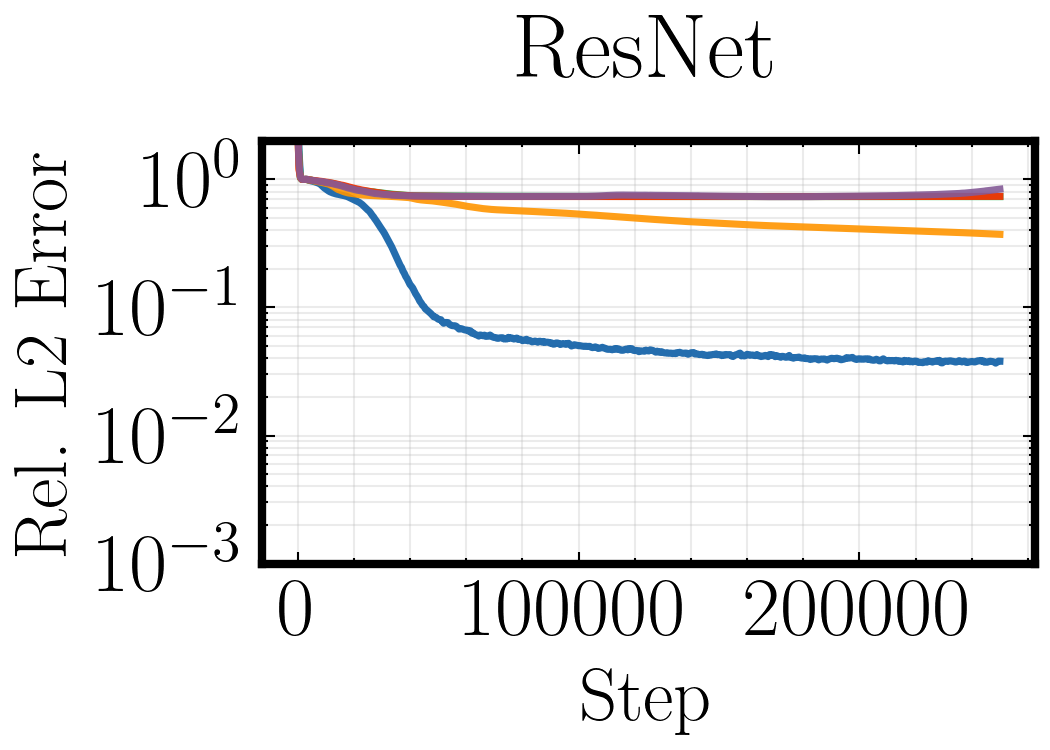}
    \caption{DW: ResNet.}
  \end{subfigure}
  \hfill
  \begin{subfigure}[t]{0.24\textwidth}
    \centering
    \includegraphics[width=\linewidth]{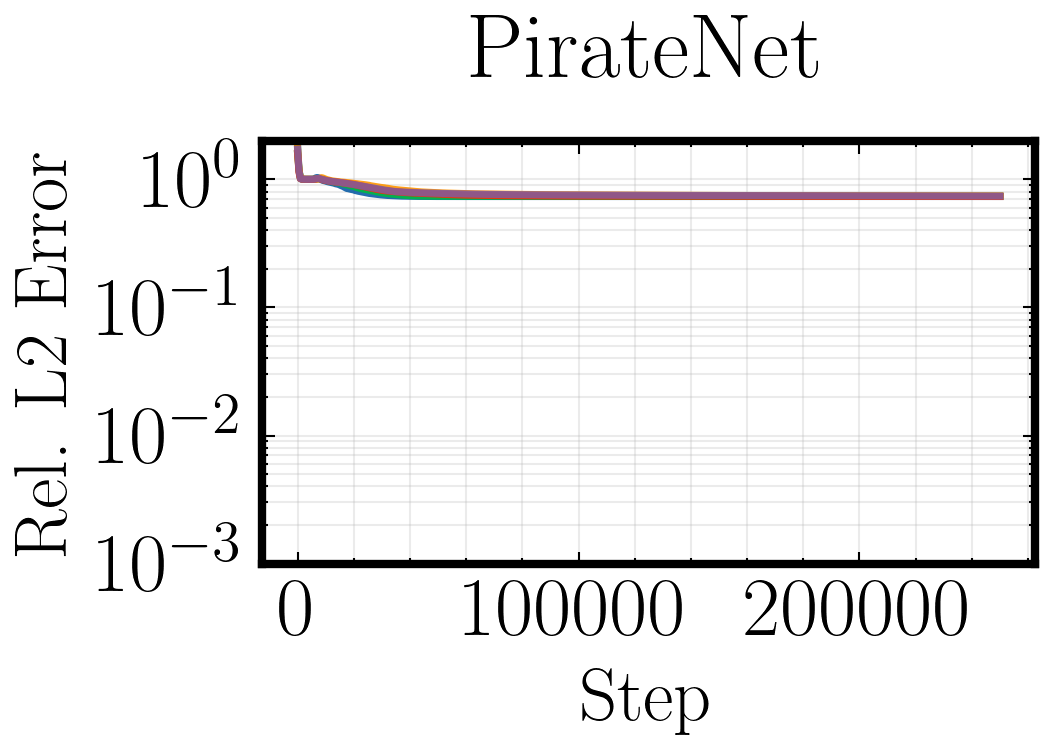}
    \caption{DW: PirateNet.}
  \end{subfigure}
  \hfill
  \begin{subfigure}[t]{0.24\textwidth}
    \centering
    \includegraphics[width=\linewidth]{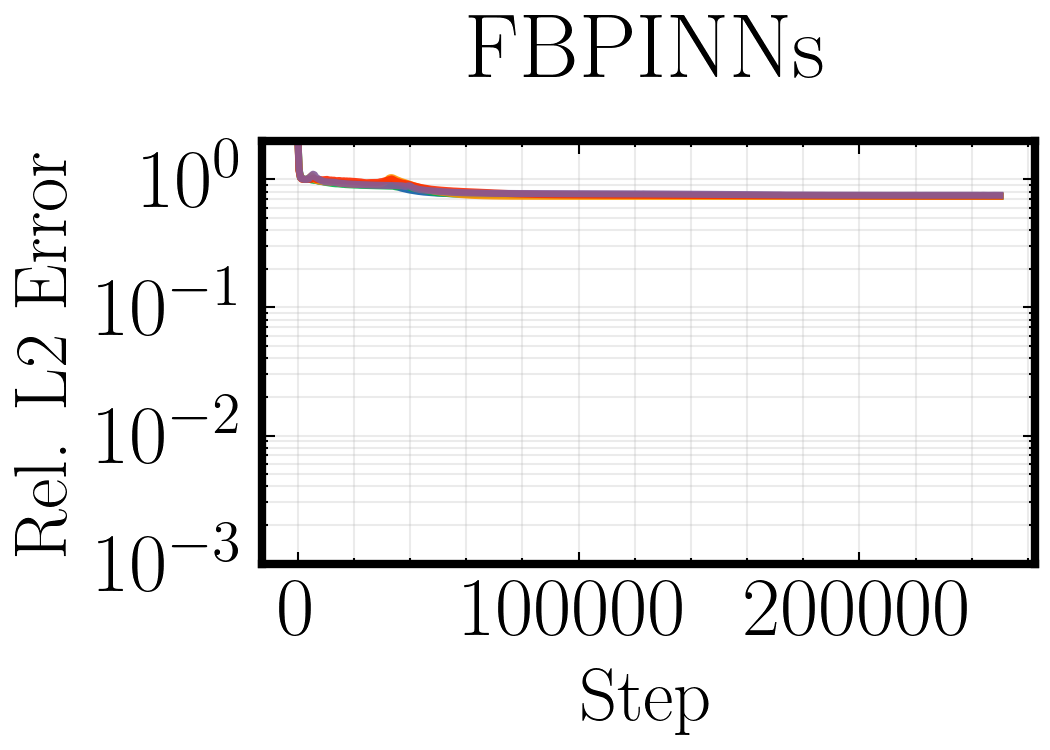}
    \caption{DW: FB-PINNs.}
  \end{subfigure}
  \hfill
  \begin{subfigure}[t]{0.24\textwidth}
    \centering
    \includegraphics[width=\linewidth]{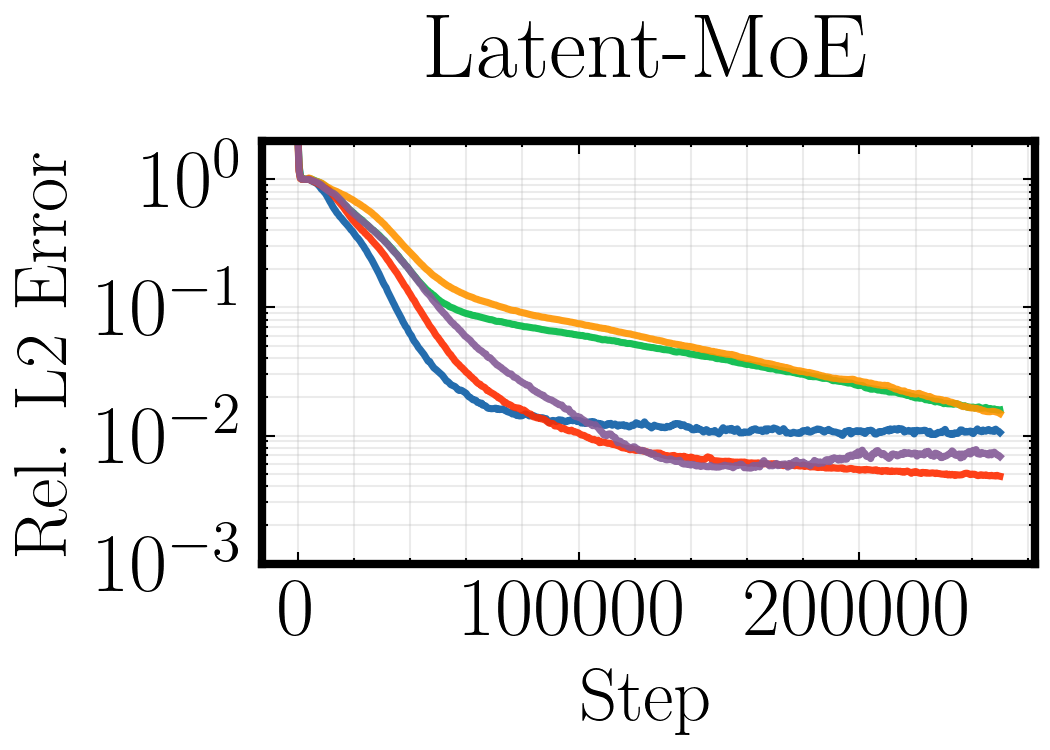}
    \caption{DW: Latent-MoE.}
  \end{subfigure}

  \begin{subfigure}[t]{0.24\textwidth}
    \centering
    \includegraphics[width=\linewidth]{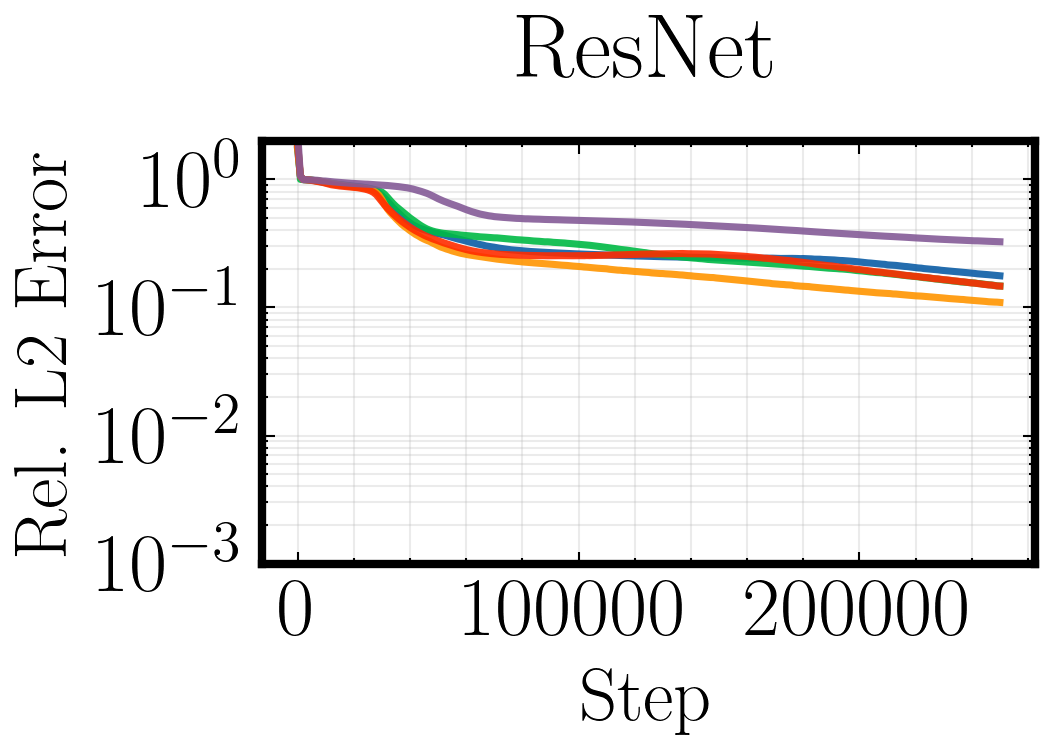}
    \caption{LW: ResNet.}
  \end{subfigure}
  \hfill
  \begin{subfigure}[t]{0.24\textwidth}
    \centering
    \includegraphics[width=\linewidth]{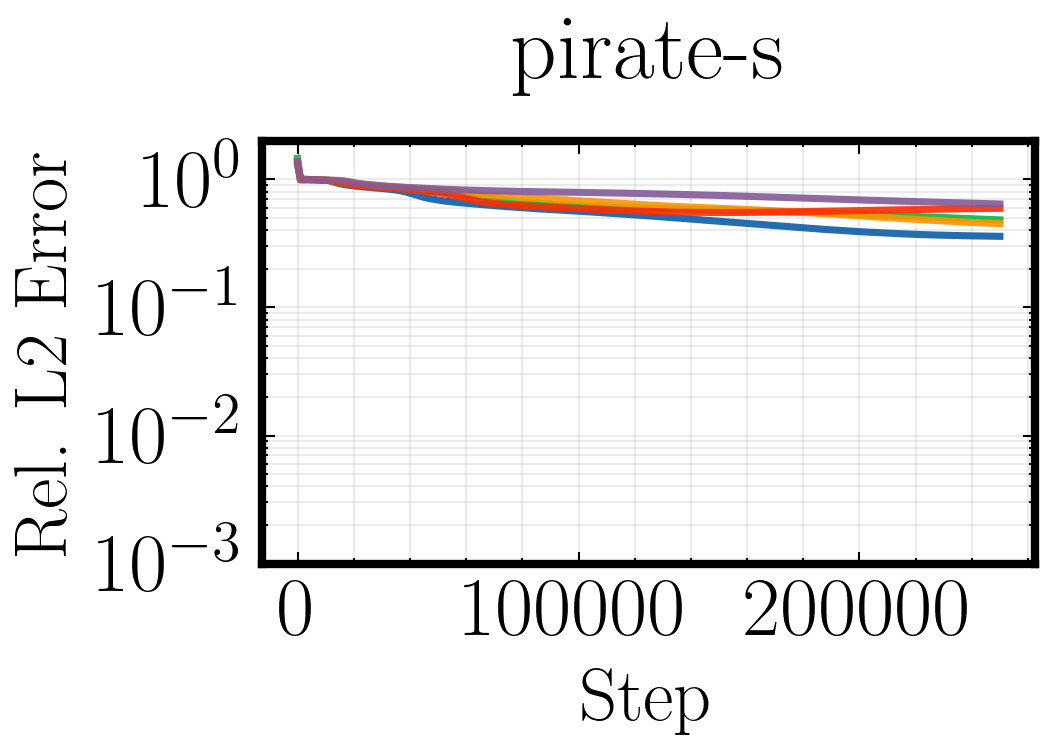}
    \caption{LW: PirateNet.}
  \end{subfigure}
  \hfill
  \begin{subfigure}[t]{0.24\textwidth}
    \centering
    \includegraphics[width=\linewidth]{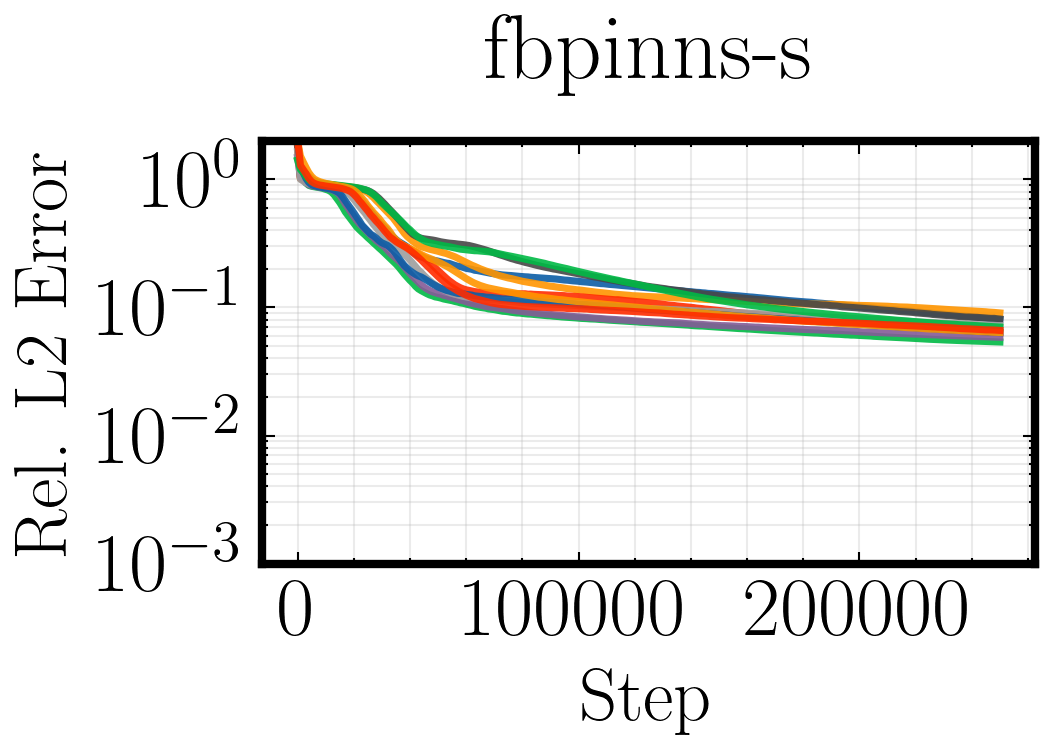}
    \caption{LW: FB-PINNs.}
  \end{subfigure}
  \hfill
  \begin{subfigure}[t]{0.24\textwidth}
    \centering
    \includegraphics[width=\linewidth]{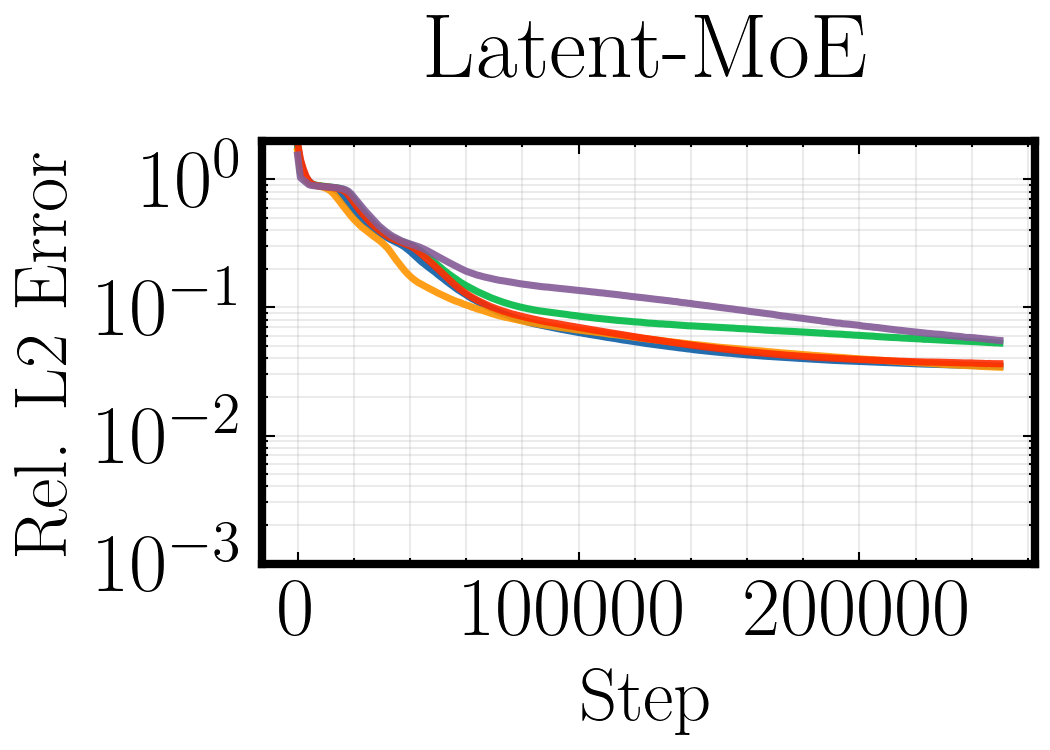}
    \caption{LW: Latent-MoE.}
  \end{subfigure}
  \caption{Relative $L_2$ trajectories of different random seeds for the four models on all benchmarks. Top row: advection-diffusion (AD). Middle row: damped wave (DW). Bottom row: layered wave (LW).}\label{fig:all_random_seed_trajectories}
\end{figure}

\section{Computational cost}\label{appendix:compute_cost}
All models are trained on a single NVIDIA RTX 6000 Ada Generation GPU. $250,000$ training steps on $500,000$ parameters models take approximately 2 hours per model, and $150,000$ training steps on smaller models ($350,000$ parameters) take approximately 1 hours per model. In this section, we report computational cost, including training cost and inference cost with all experts active of the larger models. Table~\ref{tab:compute_cost} shows training and inference time for each model at different widths (widening factor for ResNet, number of experts for Latent-MoE and FB-PINNs), evaluated on the damped-wave benchmark. For FB-PINNs and Latent-MoE, inference can also be run sparsely on non-overlapping subdomains, so at most two experts are active per evaluation point. Figure~\ref{fig:compute_cost_plots} illustrates that the Latent-MoE model is more computationally efficient than the FB-PINNs, thanks to the shared backbone.

\begin{table}[H]
\centering
\small
\setlength{\tabcolsep}{8pt}
\caption{Training and inference compute cost. The training cost is evaluated as an average of 1000 steps with the training batch size, and the inference time is evaluated as an average of 200 steps, each with 4096 points.}\label{tab:compute_cost}
\renewcommand{\arraystretch}{1.22}
\begin{tabular}{cccc}
\toprule
Architecture & Width & Train (ms) & Inference (ms) \\
\midrule
\multirow{4}{*}{ResNet} & 2 & 14.442 & 0.218 \\
 & 4 & 17.999 & 0.242 \\
 & 6 & 22.884 & 0.297 \\
 & 8 & 28.062 & 0.327 \\
\hline
\multirow{4}{*}{Latent MoE} & 2 & 20.308 & 0.208 \\
 & 4 & 23.517 & 0.249 \\
 & 6 & 27.070 & 0.298 \\
 & 8 & 30.423 & 0.372 \\
\hline
\multirow{4}{*}{FB-PINNs} & 2 & 22.712 & 0.228 \\
 & 4 & 27.021 & 0.267 \\
 & 6 & 31.832 & 0.361 \\
 & 8 & 36.453 & 0.411 \\
\hline
PirateNet & - & 29.395 & 0.376 \\
\bottomrule
\end{tabular}
\end{table}

Figure~\ref{fig:compute_cost_plots} provides the corresponding compute-cost visualizations for training and inference.
\begin{figure}[H]
  \centering
  \begin{subfigure}[t]{0.33\linewidth}
    \centering
    \includegraphics[width=\linewidth]{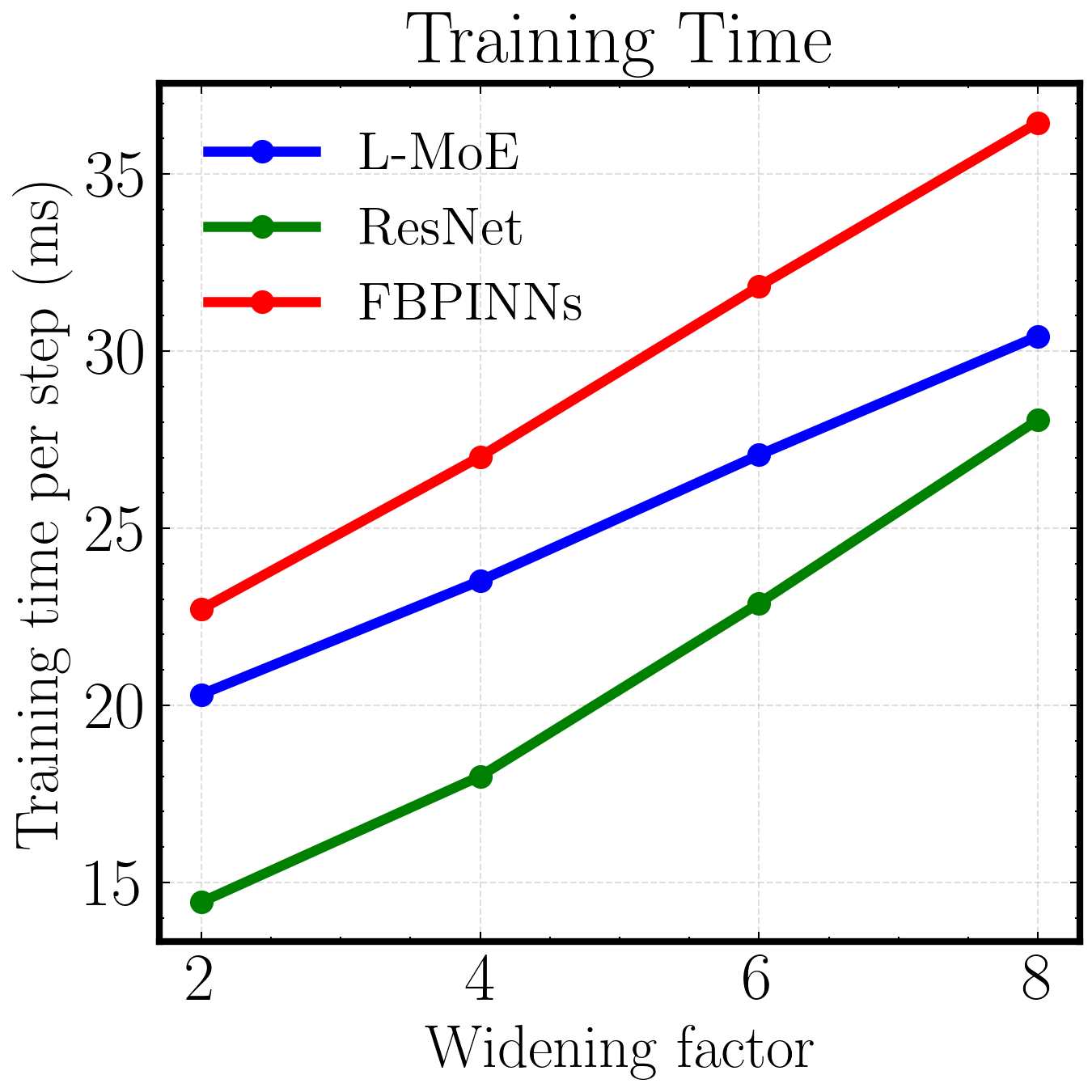}
    \caption{Training cost}.
  \end{subfigure}
  \begin{subfigure}[t]{0.33\linewidth}
    \centering
    \includegraphics[width=\linewidth]{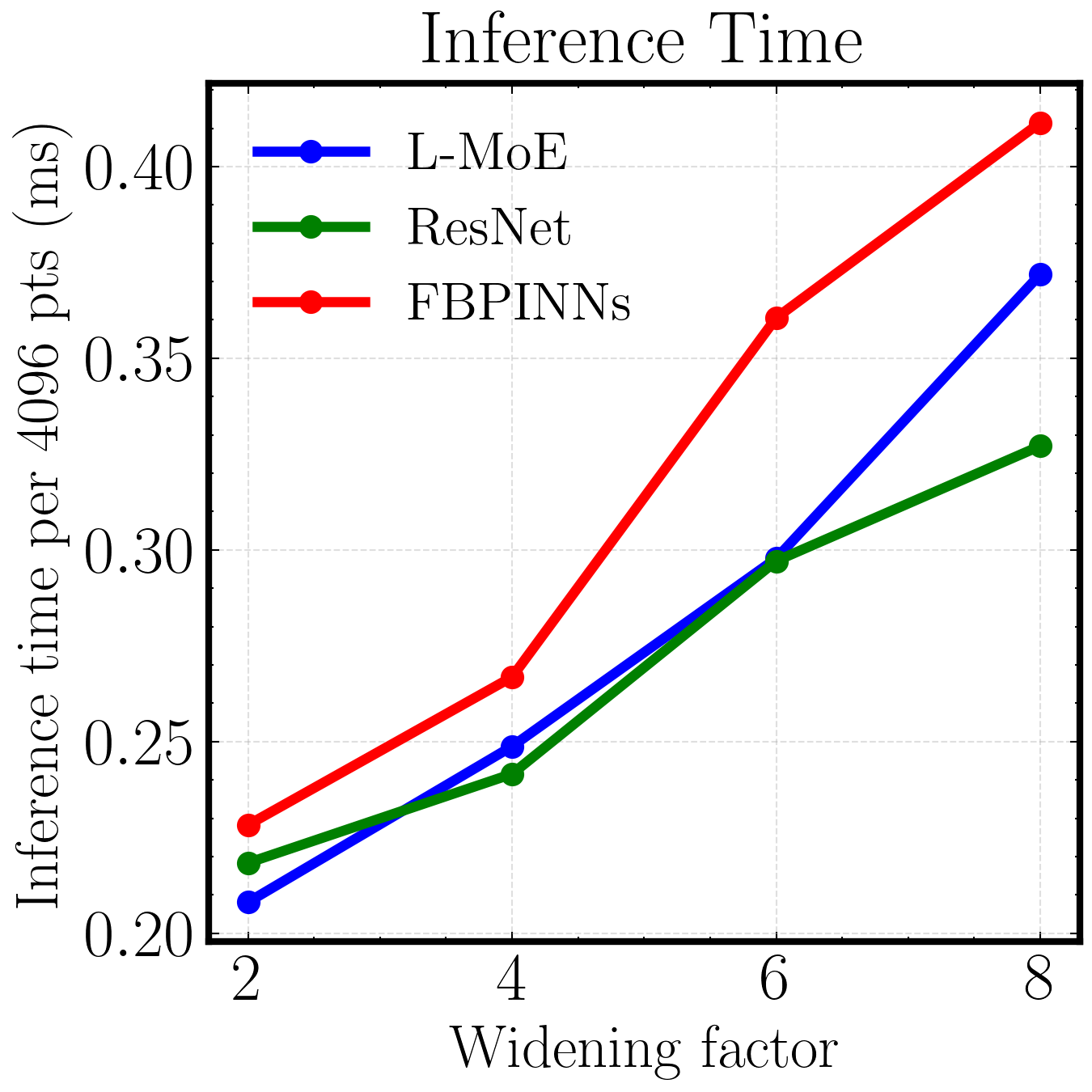}
    \caption{Inference cost.}
  \end{subfigure}
  \caption{Compute-cost curves across model families and widths for the damped-wave benchmark.}
  \label{fig:compute_cost_plots}
\end{figure}

\clearpage
\section{Reproducibility Statement}\label{appendix:reproducibility}
All experiments are seeded with \texttt{jax.random.PRNGKey} in JAX. However, exact reproducibility is still not guaranteed because GPU operations are not fully deterministic, and this effect is amplified by the stiff PINN loss landscape. As a result, outcomes may vary slightly across runs even with the same seed, especially across hardware accelerators, software versions, and operating systems.

\section{Declaration of LLM usage}\label{appendix:llm_usage}

We used publicly available LLMs for coding assistance, writing assistance, and proofreading. We carefully audited all LLM outputs and manually ran the experiments so that the LLMs had no access to experimental results. The tables are generated by a Python script that reads the raw data and formats it into LaTeX code.

\section{Broader Impact}\label{appendix:broader_impact}

Multi-regime physics -- wave propagation through layered media, reactive
flows with temperature-dependent kinetics, structures driven by staged
forcing -- is where many engineering and geoscience applications of PINNs
are heading, and where standard architectures fail silently: as we show,
even state-of-the-art baselines converge to plausible-looking spurious
solutions on staged dynamics. By characterizing this failure mode through
the NTK and providing an architectural remedy, this work expands the
regime in which PINNs are a defensible modeling tool. Extensions such as
physics-informed neural
operators~\cite{wang2021learningthe} make the same machinery applicable to
families of PDEs with varying coefficients and boundary conditions, broadening
practical reach to settings such as climate modeling, fluid dynamics, and
materials science. The same property is double-edged: improved apparent
fidelity does not substitute for validation against reference solvers or
experimental data, and PINN-based surrogates should not be deployed in
safety-critical settings without independent verification. We see no
specific dual-use risk associated with the methods presented here.